\documentclass[journal,onecolumn,12pt,twoside]{IEEEtran}

\newif\ifdoublecolumn
\newif\ifsinglecolumn
\IEEEoverridecommandlockouts   

\ifCLASSINFOpdf
   \usepackage[pdftex]{graphicx}
  \graphicspath{{figures/}}
  \usepackage{epstopdf}
  \DeclareGraphicsExtensions{.pdf,.jpeg,.png}
\else
   \usepackage[dvips]{graphicx}
   \graphicspath{{figures/}}
   \DeclareGraphicsExtensions{.eps}
\fi

\usepackage{pgfplots}
\pgfplotsset{compat=newest} 

\usetikzlibrary{plotmarks}
\usetikzlibrary{arrows.meta}
\usepgfplotslibrary{patchplots}
\usepackage{grffile}

\pgfplotsset{plot coordinates/math parser=false}
\newlength\figureheight
\newlength\figurewidth

\usepackage{pdfpages} 

\usepackage{stfloats}
\usepackage{nicefrac}
\usepackage{amsthm}
\usepackage{amssymb}
\usepackage{pifont}

\usepackage{latexsym}
\usepackage{times}

\usepackage{bm} 

\usepackage{array}
\usepackage[cmex10]{amsmath}

\usepackage{calrsfs}
\DeclareMathAlphabet{\pazocal}{OMS}{zplm}{m}{n}

\usepackage{multicol}
\usepackage{urwchancal}

\newcommand\copyrighttext{%
  \footnotesize \textcopyright 2022 IEEE. Personal use of this material is permitted.
  Permission from IEEE must be obtained for all other uses, in any current or future
  media, including reprinting/republishing this material for advertising or promotional
  purposes, creating new collective works, for resale or redistribution to servers or
  lists, or reuse of any copyrighted component of this work in other works.
  DOI: 10.1109/JSTSP.2022.3221681
  }
\newcommand\copyrightnotice{%
\begin{tikzpicture}[remember picture,overlay]
\node[anchor=south,yshift=10pt] at (current page.south) {\fbox{\parbox{\dimexpr\textwidth-\fboxsep-\fboxrule\relax}{\copyrighttext}}};
\end{tikzpicture}%
}

\usepackage{support-caption}
\usepackage[subrefformat=parens,labelformat=parens]{subcaption}
\usepackage{multirow}
\usepackage{xcolor}

\usepackage[nolist,printonlyused]{acronym}      

\usepackage{optidef} 
\usepackage{algpseudocode, algorithm}

\makeatletter
\algrenewcommand\ALG@beginalgorithmic{\footnotesize}
\makeatother
\algrenewcommand\algorithmiccomment[2][\normalsize]{{#1\hfill\(\triangleright\) #2}}
\usepackage{comment}
\usepackage{cite}

\usepackage{titlecaps}

\usepackage{setspace}

\usepackage{xfrac} 

\newcommand{\altfrac}[2]{\ifmmode\def\tmp{$}\else\def\tmp{}\fi\mbox{%
    {\raisebox{.24\ht\strutbox}{\tmp#1\tmp}}%
    \kern-2.2pt\scalebox{1.6}[1.5]{/}\kern-1.8pt%
    {\tmp#2\tmp}%
    }}

\usepackage{textcomp}

\usepackage[export]{adjustbox} 

\usepackage{hyperref}

\newtheorem{theorem}{Theorem} 
\newtheorem{corollary}{Corollary}
\newtheorem{lemma}{Lemma}

\newtheorem{definition}{Definition}
\newtheorem{prop}{Proposition}
\newtheorem{remark}{Remark}

\newtheorem{assumption}{Assumption}

\usepackage{color}

\usepackage{xpatch}
\iftrue 
\makeatletter
\xpatchcmd{\proof}{\@addpunct{.}}{\normalfont\,\@addpunct{:}}{}{}
\makeatother
\fi

\DeclareMathAlphabet\mathbfcal{OMS}{cmsy}{b}{n}
\begin{acronym}[LTE-Advanced]
  \acro{2G}{Second Generation}
  \acro{3G}{3$^\text{rd}$~Generation}
  \acro{3GPP}{3$^\text{rd}$~Generation Partnership Project}
  \acro{4G}{4$^\text{th}$~Generation}
  \acro{5G}{5$^\text{th}$~Generation}
  \acro{5GPPP}{5G Infrastructure Public Private Partnership}
  \acro{ACLR}{adjacent channel leakage ratio}
  \acro{ADMM}{alternating direction method of multipliers}
  \acro{ARQ}{automatic repeat request}
  \acro{BADMM}{Bregman \ac{ADMM}} 
  \acro{BER}{bit error rate}
  \acro{BLER}{block error rate}
  \acro{BPC}{binary power control}
  \acro{BPSK}{binary phase-shift keying}
  \acro{BRA}{balanced random allocation}
  \acro{BS}{base station}
  \acro{BW}{bandwidth}
  \acro{CADMM}{\ac{CSI}-aware \ac{ADMM}}
  \acro{CAP}{combinatorial allocation problem}
  \acro{CAPEX}{capital expenditure}
  \acro{CBF}{coordinated beamforming}
  \acro{CCDF}{complementary cumulative distribution function}
  \acro{CFR}{crest-factor reduction}
  \acro{CP}{cyclic prefix}
  \acro{CS}{coordinated scheduling}
  \acro{CPOCS}{\ac{CSI}-aware \ac{POCS}}
  \acro{CRUSADE}{convex suppression of unwanted emission and distortion elimination}
  \acro{CRUISE}{convex reduction of unwanted spectral emission}
  \acro{CRUISE-DAYS}{Convex RedUctIon of unwanted Spectral Emission with Davis And Yin Splitting}
  \acro{cruise}{convex reduction of unwanted spectral emission}
  \acro{CRUISE-SSP}{\ac{CRUISE}-\ac{SSP}}
  \acro{CSI}{channel state information}
  \acro{CSIT}{channel state information at the transmitter}
  \acro{D2D}{device-to-device}
  \acro{DCA}{dynamic channel allocation}
  \acro{DCI}{downlink control information}
  \acro{DE}{differential evolution}
  \acro{DFT}{discrete Fourier transform}
  \acro{DIST}{Distance}
  \acro{DL}{downlink}
  \acro{DMA}{double moving average}
  \acro{DMRS}{demodulation reference signal}
  \acro{D2DM}{D2D mode}
  \acro{DMS}{D2D mode selection}
  \acro{DROPE}{Douglas-Rachford-based \ac{OOBE} power reduction with \ac{EVM} constraint}
  \acro{DRS}{Douglas-Rachford splitting}
  \acro{DYS}{Davis-Yin splitting}
  \acro{DPC}{dirty paper coding}
  \acro{DR}{Douglas-Rachford}
  \acro{DRA}{dynamic resource assignment}
  \acro{DSA}{dynamic spectrum access}
  \acro{eMBB}{enhanced mobile broadband}
  \acro{eV2X}{enhanced vehicle-to-everything}
  \acro{EADMM}{\ac{EVM}-constrained \ac{ADMM}}
  \acro{EIRP}{equivalent isotropically radiated power}
  \acro{EPOCS}{\ac{EVM}-constrained \ac{POCS}}
  \acro{ETSI}{European telecommunications standards institute}
  \acro{EVM}{error vector magnitude}
  \acro{EMSP}{EVM-constrained \ac{MSP}}
  \acro{ENSP}{EVM-constrained \ac{NSP}}
  \acro{ESP}{EVM-constrained spectral precoding}
  \acro{ESSP}{EVM-constrained \ac{SSP}}
  \acro{FD}{fully digital}
  \acro{FDD}{frequency division duplexing}
  \acro{FFT}{(fast) discrete Fourier Transform}
  \acro{FL}{Federated learning}
  \acro{GFBS}{generalized forward-backward splitting}
  \acro{HARQ}{hybrid automatic repeat request}
  \acro{HB}{hybrid beamforming}
  \acro{KKT}{Karush-Kuhn-Tucker}
  \acro{KPI}{key performance indicator}
  \acro{ICF}{iterative clipping and filtering}
  \acro{IDFT}{Inverse \ac{DFT}}
  \acro{IFFT}{Inverse \ac{FFT}}
  \acro{i.i.d.}{independent and identically distributed}
  \acro{InC}{in-coverage}
  \acro{IPAPR}{instantaneous \ac{PAPR}}
  \acro{IoT}{internet of things}
  \acro{ITS}{intelligent transportation systems}
  \acro{L-BFGS}{limited-memory Broyden-Fletcher-Goldfarb-Shanno}
  \acro{LDPC}{low-density parity-check code}
  \acro{LMMSE}{linear minimum mean squared error}
  \acro{LOS}{line-of-sight}
  \acro{LS}{large-scale}
  \acro{LS-MIMO-OFDM}{\ac{LS}-\ac{MIMO}-\ac{OFDM}}
  \acro{LS-MSP}{large-scale MSP}
  \acro{LTE}{long-term evolution}
  \acro{MAC}{medium access control}
  \acro{MADMM}{modified consensus \ac{ADMM}}
  \acro{MCS}{modulation and coding scheme}
  \acro{METIS}{Mobile Enablers for the Twenty-Twenty Information Society}
  \acro{MIMO}{multiple-input multiple-output}
  \acro{MIMO-OFDM}{\ac{MIMO}-\ac{OFDM}}
  \acro{MISO}{multiple-input single-output}
  \acro{MRC}{maximum ratio combining}
  \acro{MS}{mobile station}
  \acro{MSE}{mean squared error}
  \acro{MSP}{mask-compliant spectral precoder}
  \acro{MMSE}{minimum mean squared error}
  \acro{MTC}{machine type communications}
  \acro{MU-MIMO}{multi-user \ac{MIMO}}
  \acro{MU-MIMO-OFDM}{\ac{MU-MIMO}-\ac{OFDM}}
  \acro{mMTC}{massive machine type communications}
  \acro{cMTC}{critical machine type communications}
  \acro{NR}{new radio}
  \acro{NSP}{notching spectral precoder}
  \acro{NSPS}{national security and public safety}
  \acro{NWC}{network coding}
  \acro{OFDM}{orthogonal frequency division multiplexing}
  \acro{OOB}{out-of-band}
  \acro{OOBE}{out-of-band emissions}
  \acro{OoC}{out-of-coverage}
  \acro{PA}{power amplifier}
  \acro{PAPR}{peak-to-average power ratio}
  \acro{PDCCH}{physical downlink control channel}
  \acro{PDSCH}{physical downlink shared channel}
  \acro{POCS}{projection on convex sets}
  \acro{PPG}{proximal-proximal-gradient}
  \acro{PRB}{physical resource block}
  \acro{PRBs}{physical resource block}
  \acro{PRG}{precoding resource block group}
  \acro{PSBCH}{physical sidelink broadcast channel}
  \acro{PSFCH}{physical sidelink feedback channel}
  \acro{PSCCH}{physical sidelink control channel}
  \acro{PSP}{proximal splitting based mask compliant spectral precoding}
  \acro{PSSCH}{physical sidelink shared channel}
  \acro{PHY}{physical}
  \acro{PLNC}{physical layer network coding}
  \acro{PRB}{physical resource block}
  \acro{PSD}{power spectral density}
  \acro{QAM}{quadrature amplitude modulation}
  \acro{QCQP}{quadratic constrained quadratic programming}
  \acro{QoS}{quality of service}
  \acro{QPSK}{quadrature-phase shift keying}
  \acro{PaC}{partial coverage}
  \acro{RAISES}{reallocation-based assignment for improved spectral efficiency and satisfaction}
  \acro{RAN}{radio access network}
  \acro{RA}{resource allocation}
  \acro{RAT}{radio access technology}
  \acro{RB}{resource block}
  \acro{RF}{radio frequency}
  \acro{RSRP}{reference signal received power}
  \acro{Rx}{receive}
  \acro{RX}{receive}
  \acro{RxEVM}{receive-EVM}
  \acro{Rx-TxEVM}{received \ac{TxEVM}}
  \acro{RZF}{regularized zero forcing}
  \acro{SC-FDM}{single carrier frequency division modulation}
  \acro{SEM}{spectral emission mask}
  \acro{SFBC}{space-frequency block coding}
  \acro{SCI}{sidelink control information}
  \acro{SIC}{successive interference cancellation}
  \acro{SINR}{signal-to-interference-plus-noise ratio}
  \acro{SISO}{single-input single-output}
  \acro{SL}{sidelink}
  \acro{SP}{Spectral precoding}
  \acro{SNR}{signal-to-noise ratio}
  \acro{SSP}{semi-analytical spectral precoding}
  \acro{STC}{space-time coding}
  \acro{SU}{single-user}
  \acro{SU-MIMO}{single-user \ac{MIMO}}
  \acro{SU-MIMO-OFDM}{\ac{SU-MIMO}-\ac{OFDM}}
  \acro{SVM}{support vector machine}
  \acro{TDD}{time division duplexing}
  \acro{TDL}{tapped delay line}
  \acro{TOP-ADMM}{three-operator \ac{ADMM}}
  \acro{Tx}{transmit}
  \acro{TX}{transmit}
  \acro{TxEVM}{transmit-EVM}
  \acro{UE}{user equipment}
  \acro{UL}{uplink}
  \acro{UP}{uniform power}
  \acro{URLLC}{ultra-reliable and low latency communications}
  \acro{VUE}{vehicular user equipment}
  \acro{V2N}{vehicular-to-network}
  \acro{V2X}{vehicle-to-everything}
  \acro{V2V}{vehicle-to-vehicle}
  \acro{V2P}{vehicle-to-pedestrian}
  \acro{WMMSE}{weighted minimum mean squared error}
  \acro{wMMSE}{weighted minimum mean squared error}
  \acro{ZF}{zero forcing}
  \acro{ZMCSCG}{zero mean circularly symmetric complex Gaussian}
\end{acronym}

\usepackage{etoolbox}

\usepackage{setspace}

\usepackage{centernot}

\usepackage{mfirstuc} 

\usepackage{enumerate}

\usepackage{footmisc}

\doublecolumntrue

\allowdisplaybreaks

\makeatletter
\newcommand\makebig[2]{%
  \@xp\newcommand\@xp*\csname#1\endcsname{\bBigg@{#2}}%
  \@xp\newcommand\@xp*\csname#1l\endcsname{\@xp\mathopen\csname#1\endcsname}%
  \@xp\newcommand\@xp*\csname#1r\endcsname{\@xp\mathclose\csname#1\endcsname}%
}
\makeatother

\makebig{biggg} {3.0}
\makebig{Biggg} {3.5}
\makebig{bigggg}{4.0}
\makebig{Bigggg}{4.5}
\makebig{biggggg}{5.0}
\makebig{Biggggg}{5.5}
\makebig{bigggggg}{6.0}
\makebig{Bigggggg}{6.5}

\def\FedAvg{\texttt{FedAvg}}
\def\FedProx{\texttt{FedProx}}

\def\FedADMM{\texttt{FedADMM}}
\def\FedTOPADMM{\texttt{FedTOP-ADMM}}
\def\FedTOPADMMI{\texttt{FedTOP-ADMM~I}}
\def\FedTOPADMMII{\texttt{FedTOP-ADMM~II}}

\def\bg{{\boldsymbol  g}}

\begin{document}

\title{ \color{black}{Federated Learning Using Three-Operator ADMM}}
\iftrue

\author{
\iftrue
Shashi~Kant,~\IEEEmembership{Student Member,~IEEE,}
Jos\'{e} Mairton B. da Silva Jr.,~\IEEEmembership{Member,~IEEE,}
Gabor~Fodor,~\IEEEmembership{Senior~Member,~IEEE,}
Bo~G\"oransson,~\IEEEmembership{Member,~IEEE,}
Mats~Bengtsson,~\IEEEmembership{Senior Member,~IEEE,}
and 
Carlo~Fischione,~\IEEEmembership{Senior Member,~IEEE}
\else
LORD
\vspace{15mm}

\fi

\vspace{-3mm}

\iftrue

\thanks{S.\ Kant, B.\ G\"oransson, and G.\ Fodor are with Ericsson AB and KTH Royal Institute of Technology, Stockholm, Sweden (e-mail: \{shashi.v.kant, bo.goransson, gabor.fodor\}@ericsson.com). The work of S.\ Kant was supported in part by the Swedish Foundation for Strategic Research under grant ID17-0114.}
\thanks{Jos\'{e} Mairton B. da Silva Jr. is with Princeton University, Princeton, NJ, USA, and also with the KTH Royal Institute of Technology, Stockholm, Sweden (e-mail: jmbdsj@kth.se). His work was supported by the European Union's Horizon Europe research and innovation programme through the Marie Sk\l{}odowska-Curie project FLASH under Grant 101067652.}

\thanks{M.\ Bengtsson and C.\ Fischione are with KTH Royal Institute of Technology, Stockholm, Sweden (e-mail: \{matben, carlofi\}@kth.se)}
%

\else

\fi
%
}
\fi

\maketitle
\copyrightnotice

\renewcommand\qedsymbol{$\blacksquare$}

\newcounter{subeqsave}
\newcommand{\savesubeqnumber}{\setcounter{subeqsave}{\value{equation}}%
\typeout{AAA\theequation.\theparentequation}}
\newcommand{\recallsubeqnumber}{%
  \setcounter{equation}{\value{subeqsave}}\stepcounter{equation}}

\iftrue
\renewcommand{\vec}[1]{\ensuremath{\boldsymbol{#1}}}
\newcommand{\mat}[1]{\ensuremath{\boldsymbol{#1}}}
\newcommand{\herm}{{\rm H}}
\newcommand{\tran}{{\rm T}}
\newcommand{\trans}{{\rm T}}
\newcommand{\trace}{{\rm Tr}}
\newcommand{\diag}{{\rm diag}}
\newcommand{\Diag}{{\rm Diag}}
\newcommand{\sign}{{\rm sign}}
\newcommand{\rank}{{\rm rank}}
\newcommand{\EVM}{{\rm EVM}}
\newcommand{\SNR}{{\rm SNR}}
\newcommand{\SINR}{{\rm SINR}}
\newcommand{\expect}{\mathbb{E}}
\newcommand{\Cm}{\mathbb{C}}
\newcommand{\Rm}{\mathbb{R}}
\newcommand{\CN}{\pazocal{CN}}
\newcommand{\be}{\begin{equation}}
\newcommand{\ee}{\end{equation}}
\newcommand{\pdf}{\pazocal{P}}
\newcommand{\prox}{\ensuremath{\boldsymbol{\rm{prox}}}}
\newcommand{\proj}{\ensuremath{\boldsymbol{\rm{proj}}}}
\newcommand{\dom}{\rm{dom}}
\newcommand{\epi}{\rm{epi}}
\newcommand{\ie}{\textit{i.e.}}
\newcommand{\eg}{\textit{e.g.}}
\newcommand{\cf}{\textit{cf.}}
\newcommand{\etc}{etc.} 
\newcommand{\avgTxEVM}{\ensuremath{\epsilon_{\rm avg}}}
\newcommand{\TxEVM}{\ensuremath{\bm{\epsilon}}}
\newcommand{\Ind}{{\delta}} 

\newcommand{\vecOp}{\rm{vec}}
\newcommand{\unvecOp}{\rm{unvec}}

\newcommand{\NR}{\ensuremath{N_{\rm R}}}
\newcommand{\NT}{\ensuremath{N_{\rm T}}}
\newcommand{\NL}{\ensuremath{N_{\rm L}}}
\newcommand{\NU}{\ensuremath{N_{\rm U}}}
\newcommand{\Ncp}{\ensuremath{N_{\rm CP}}}
\newcommand{\Nsc}{\ensuremath{N_{\rm SC}}}

\newcommand{\Kcp}{\ensuremath{K_{\rm CP}}}

\newcommand{\RxEVM}{\ensuremath{\bm{\varsigma}}}

\newcommand{\A}{\ensuremath{\boldsymbol{A}}}
\renewcommand{\a}{\ensuremath{\boldsymbol{a}}}
\newcommand{\B}{\ensuremath{\boldsymbol{B}}}
\renewcommand{\b}{\ensuremath{\boldsymbol{b}}}
\newcommand{\C}{\ensuremath{\boldsymbol{C}}}

\newcommand{\D}{\ensuremath{\boldsymbol{D}}}
\renewcommand{\d}{\ensuremath{\boldsymbol{d}}}
\newcommand{\E}{\ensuremath{\boldsymbol{E}}}
\newcommand{\F}{\ensuremath{\boldsymbol{F}}}

\newcommand{\K}{\ensuremath{\boldsymbol{K}}}

\newcommand{\M}{{\mathbf{M}}}
\newcommand{\I}{\ensuremath{\boldsymbol{I}}}

\newcommand{\R}{\ensuremath{\boldsymbol{R}}}

\renewcommand{\S}{\ensuremath{\boldsymbol{S}}}

\newcommand{\T}{\ensuremath{\boldsymbol{T}}}
\renewcommand{\t}{\ensuremath{\boldsymbol{t}}}

\newcommand{\U}{\ensuremath{\boldsymbol{U}}}

\newcommand{\V}{\ensuremath{\boldsymbol{V}}}

\newcommand{\X}{\ensuremath{\boldsymbol{X}}}

\newcommand{\Y}{\ensuremath{\boldsymbol{Y}}}

\newcommand{\0}{\ensuremath{\boldsymbol{0}}}
\fi

\iffalse
\colorlet{blueorange}{blue!15!orange!85!}
\definecolor{ao(english)}{rgb}{0.0, 0.5, 0.0}
\def\baselinestretch{1}
\def\sk#1{\textcolor{black}{#1}}
\def\skblue#1{\textcolor{blue}{#1}}
\def\skred#1{\textcolor{red}{#1}}
\def\skgreen#1{\textcolor{ao(english)}{#1}}
\def\skteal#1{\textcolor{teal}{#1}}
\def\skblack#1{\textcolor{black}{#1}}
\def\gf#1{\textcolor{black}{#1}}
\def\bg#1{\textcolor{green}{#1}}
\def\mb#1{\textcolor{magenta}{#1}}
\def\cf#1{\textcolor{teal}{#1}}
\else 
\colorlet{blueorange}{blue!15!orange!85!}
\def\baselinestretch{1}
\def\sk#1{\textcolor{black}{#1}}
\def\skblue#1{\textcolor{black}{#1}}
\def\skred#1{\textcolor{black}{#1}}
\def\skgreen#1{\textcolor{black}{#1}}
\def\skteal#1{\textcolor{black}{#1}}
\def\skblack#1{\textcolor{black}{#1}}
\def\bg#1{\textcolor{black}{#1}}
\def\mb#1{\textcolor{black}{#1}}
\def\cf#1{\textcolor{black}{#1}}
\fi

\newcommand{\mx}[1]{\mathbf{#1}}

\DeclareRobustCommand{\BigOh}{%
  \text{\usefont{OMS}{cmsy}{m}{n}O}%
}

\let\oldemptyset\emptyset
\let\emptyset\varnothing





\algblock{ParFor}{EndParFor}
\algnewcommand\algorithmicparfor{\textbf{parfor}}
\algnewcommand\algorithmicpardo{\textbf{do}}
\algnewcommand\algorithmicendparfor{\textbf{end\ parfor}}
\algrenewtext{ParFor}[1]{\algorithmicparfor\ #1\ \algorithmicpardo}
\algrenewtext{EndParFor}{\algorithmicendparfor}

\vspace{-10.5mm}
\begin{abstract}
Federated learning (FL) has emerged as an instance of distributed machine learning paradigm that avoids the transmission of data generated on the users' side. Although data are not transmitted, edge devices have to deal with limited communication bandwidths, data heterogeneity, and straggler effects due to the limited computational resources of users' devices. A prominent approach to overcome such difficulties is FedADMM, which is based on the classical two-operator consensus alternating direction method of multipliers (ADMM). The common assumption of FL algorithms, including FedADMM, is that they learn a global model using data only on the users' side and not on the edge server. However, in edge learning, the server is expected to be near the base station and \skblack{have} direct access to rich datasets. In this paper, we argue that leveraging the rich data on the edge server is much more beneficial than utilizing only user datasets. Specifically, we show that the mere application of FL with an additional virtual user node representing the data on the edge server is inefficient. We propose FedTOP-ADMM, which generalizes FedADMM and is based on a three-operator ADMM-type technique that exploits a smooth cost function on the edge server to learn a global model parallel to the edge devices. Our numerical experiments indicate that FedTOP-ADMM has substantial gain up to 33\% in communication efficiency to reach a desired test accuracy with respect to FedADMM, including a virtual user on the edge server.
\end{abstract}

\vspace{-6.5mm}
\section{Introduction}\label{sec_chD:intro}
\vspace{-1.5mm}

Centralized training of machine learning models becomes prohibitive for a large number of users, particularly if the users --- also known as clients or agents or workers --- have to share a large dataset with the central server. Furthermore, sharing a dataset with the central server may not be feasible for some users due to privacy concerns. Therefore, training algorithms using distributed and decentralized approaches are preferred. This has led to the concept of federated learning (\acs{FL}\acused{FL}), which results from the synergy between large-scale distributed optimization techniques and machine learning.  Consequently, FL has received considerable attention in the last few years since its introduction in~\cite{konevcny2016federated, McMahan__FedAvg__2017}. 

In the \ac{FL} framework --- illustrated in Figure~\ref{fig_chD:fl_arch_illustration} --- a distributed optimization problem, such as problem~\eqref{eqn_chD:two_operator_general_consensus_admm}, is essentially solved considering a central server and other devices by exchanging the 
parameters/weights of a considered model rather than sharing private data among themselves. The devices desire to achieve a learning model using data from all the other devices for the training. Instead of sending data from the devices to the edge server that computes such a model, these devices execute some local computations and periodically share only their parameters. Specifically, the FL technique intends to minimize a finite sum of (usually assumed) differentiable functions that depend on the data distributions on the various devices. The common solution to such a minimization involves an iterative procedure, wherein at each global communication iteration, the clients transmit their updated local parameters to a central edge server --- illustrated as Step~\ding{172} in Figure~\ref{fig_chD:fl_arch_illustration}. The edge server then updates the global model parameters --- shown as Step~\ding{173} in Figure~\ref{fig_chD:fl_arch_illustration}. However, the processing Steps~2a and~2b are specific to our proposed approach, and leverage the possibility of using a dataset on the edge server to improve the learning model. 
Thereafter, the server broadcasts the updated parameters or weights to all or the selected nodes---see Step~\ding{174} in Figure~\ref{fig_chD:fl_arch_illustration}. Lastly, the clients update their local parameters using their local (private) dataset and the received global model parameters from the central server---Step~\ding{175} in Figure~\ref{fig_chD:fl_arch_illustration}---to proceed to the next iteration. 

\begin{figure}[tp!]
    \centering
    \scalebox{0.99}{\includegraphics[trim=5mm 0mm 15mm 5mm,clip]{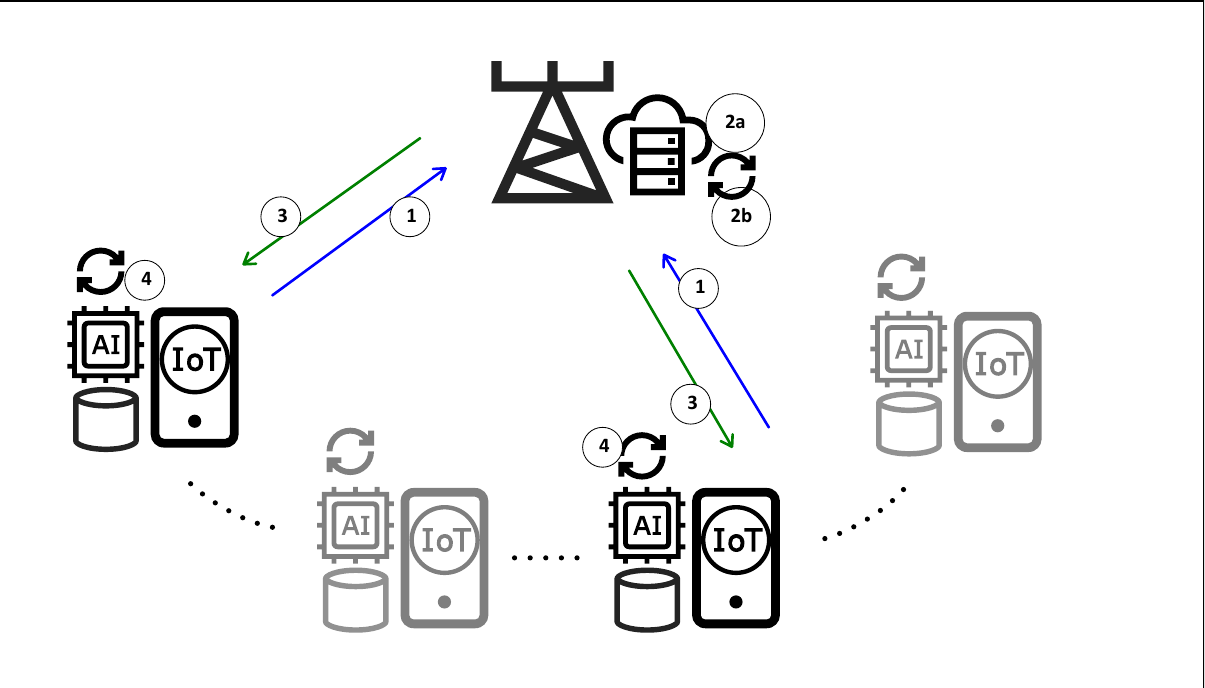}}
    \caption{ Illustration of {FL} architecture, with the new scenario investigated in this paper of a dataset available on the edge server. }
    \label{fig_chD:fl_arch_illustration}
\end{figure}

Many state-of-the-art \ac{FL} techniques, such as {\FedAvg}~\cite{McMahan__FedAvg__2017} and {\FedProx}~\cite{Tian_Li_etal__FedProx__2020} can be seen as an instance of one-operator proximal splitting techniques~\cite{Saber_etal__FL_from_splitting_algo__2021}, \ie, described in~\eqref{eqn_chD:two_operator_general_consensus_admm} with function $g\!=\!0$, which considers learning on the users' side mainly. 
{Differently from these techniques, we propose to go beyond and use a three-operator{\footnote{\skblack{Notice that ``operator" terminology is used in the dual of ADMM-type, such as Douglas-Rachford~\cite{Bauschke:2011}. Therefore, we borrow this terminology in this work.}}} proximal splitting technique, such as our recently proposed three-operator alternating direction method of multipliers (\acused{TOP-ADMM}\ac{TOP-ADMM}) method~\cite{Kant_journal_msp_top_admm:2020}, to leverage the possibility of learning on the edge server together with the traditional learning on the edge devices.}

\subsection{Motivation for Learning Model on the Edge Server}
In future (6G) cellular networks, machine learning services will be used both to design the networks and as a service provided by the networks~\cite{Hellstrom2022}. 
We expect to leverage the raw data generated not only on the users' side, but also on the edge server's side collocated at the base station or/and radio access networks. More concretely, we foresee scenarios, in which the raw data is available at the base station \skblack{because data is generated continuously at the physical layer and radio in 5G NR and beyond. Thus, data at the server side} together with data on the users' side, enable many distributed signal processing applications such as joint communication and sensing, or edge learning with the Internet of Things (including edge devices). 
The optimization problem for \ac{FL} considering learning on the edge server and on the edge devices is naturally described by a sum of three functions or three operators---see~\eqref{eqn_chD:three_operator__FL_optimization_problem}. A three-operator problem for \ac{FL} can be solved via traditional block-wise two-operator splitting techniques or treating the learning model on the edge server as a virtual user, \eg, in an existing two-operator-based{\footnote{\skblack{Although {\FedADMM} is built on two-operator \acused{ADMM}\ac{ADMM}, it actually solves one-operator problem---see discussion in Section~\ref{sec_chD:cmp_admm_fedadmm_topadmm_fedtopadmm} and Table~\ref{table_chD:distributed_opt_formulations}.}}} {\FedADMM}~\cite{zhou_and_li_fedadmm2022}. However, we show numerically, see Figure~\ref{fig_chD:cmp_perf_fedtopadmm1_and2__fedadmm_and_fedadmmvc}, that such an approach is not necessarily more communication efficient than tackling such a problem fundamentally from the three-operator proximal splitting perspective.

To the best of our knowledge, \ac{FL} using three-operator techniques have not been addressed in the literature, which tackles the learning on both the server and users independently. 
{This proposal is highly novel and benefits from the richly available datasets on the edge server and edge devices from current fifth generation (5G) and future 6G cellular networks.
Hence, these are key motivating reasons to consider a three-operator problem being tackled in an \ac{FL} fashion.}

\subsection{Contributions}

We present a new communication-efficient and computationally efficient \ac{FL} framework, referred to as {\FedTOPADMM}, using a three-operator \ac{ADMM} method. 
Specifically, our key contributions are:
\begin{itemize}	
    \item We demonstrate the viability of a practical edge learning scenario in which private datasets are available on the devices, and another private dataset is available on the edge server/base station.   We model this edge learning scenario using a novel three-operator splitting method that benefits from the private datasets on both edge server and edge devices.
    \item We propose the {\FedTOPADMM} method by applying and extending our recently proposed \ac{TOP-ADMM} method~\cite{Kant_journal_msp_top_admm:2020} to tackle a composite optimization problem~\eqref{eqn_chD:three_operator__FL_optimization_problem} comprising a sum of three functions (or {three operators}). More specifically, we propose two variants of {\FedTOPADMM}, termed \linebreak{\FedTOPADMMI} and {\FedTOPADMMII}, where {\FedTOPADMMII} does not learn on the server side when aggregating the parameters from the users and the server itself to generate a common model parameter. However, {\FedTOPADMMI} learns a model before aggregation of the parameters in addition to learning in parallel with the users. Thus, {\FedTOPADMMI} has a slightly better performance compared to {\FedTOPADMMII}.
    \item We extend the results of \cite{Kant_journal_msp_top_admm:2020} by establishing a new theoretical convergence proof of \ac{TOP-ADMM} under general convex settings. \skblack{Additionally, extending the convergence results of \ac{TOP-ADMM}, we prove the optimality conditions of {\FedTOPADMM}.}
    \item {\FedTOPADMM} capitalizes on the possible data available on the edge server collocated at the base station in addition to the data available on the users' side. Consequently, our numerical experiments show noticeable communication efficiency gain over the existing state-of-the-art \ac{FL} schemes using real-world data.
	\item Our proposed {\FedTOPADMM} is built on the existing framework of communication- and computationally efficient {\FedADMM}~\cite{zhou_and_li_fedadmm2022}. Therefore, {\FedTOPADMM} inherits all the merits of {\FedADMM}. Furthermore, we show that our proposed {\FedTOPADMM} is up to 33\% more efficient in terms of communication rounds {to achieve the same target test accuracy} as the one achieved by {\FedADMM}, where the extra dataset on the edge server is modelled as an additional virtual client collocated at the base station. 
\end{itemize}

\section{State of the art}
In this section, we briefly describe the related works on \ac{FL}, including two or three operators in proximal splitting techniques useful for \ac{FL}.

\subsection{Related Works on Federated Learning}

Starting from the seminal work~\cite{konevcny2016federated}, several extensions and applications of \ac{FL} have been proposed, \eg, \FedAvg~\cite{McMahan__FedAvg__2017}, \FedProx~\cite{Tian_Li_etal__FedProx__2020}, {\FedADMM}~\cite{zhou_and_li_fedadmm2022}, and other variations,  \cite{jordan2018communication,tandon2017gradient,chen2018lag}. A relatively recent overview can be found in~\cite{li2020federated}.

\ac{FL} may suffer from two drawbacks: privacy leakage and communication inefficiency. Although \ac{FL} keeps the data local on the clients and thus inherently has privacy properties, it does not guarantee complete privacy because significant information may still leak through observing the gradients~\cite{bagdasaryan2019differential}. Moreover, the \ac{FL} algorithms may have an unsustainable communication cost: the local parameters must be communicated via uplink from the devices to the edge server, and via downlink from the server to the local devices. The local parameters can be vectors of huge sizes whose frequent transmissions and reception may deplete the battery of the devices and consume precious communication resources. 

A number of works have addressed the problem of communication efficiency in \ac{FL}~\cite{stich2018sparsified,  di2018efficient,wangni2018gradient, kairouz2021advances, chen2018lag, sun2019communication, yu2019parallel}. We can roughly divide them into two classes: 1) data compression in terms of quantization and sparsification of the local parameters before every transmission~\cite{stich2018sparsified, wangni2018gradient, di2018efficient, kairouz2021advances}, and 2) reduction of the communication iterations~\cite{kairouz2021advances, chen2018lag, sun2019communication, yu2019parallel}. The works in the second class attempt to reduce some communication rounds between the devices and the edge server, as for example proposed in lazily aggregated gradient
(LAG) approach~\cite{chen2018lag,sun2019communication}. In LAG, each device transmits its local parameter only if the variation from the last transmission is large enough. \sk{However, both classes of approaches assume an underlying iterative algorithm whose iterations are sometimes eliminated or whose carried information (in bits) is processed to consume fewer communication resources. The process of making the underlying algorithm more communication efficient, regardless of the improvements that can be done on top of it, has been less investigated.} The state of the art can be found in {\FedADMM}~\cite{zhou_and_li_fedadmm2022}, which \sk{not only allows the averaging of users' parameters periodically to reduce the communication rounds but also} aims to improve the communication efficiency by using one-operator proximal splitting techniques in \ac{FL} methods. Our paper focuses on this line of research and extends it to three-operator proximal splitting while proposing learning jointly on the edge server and edge devices side. 

\subsection{Related Works on Operator/Proximal Splitting} \label{sec_chD:related_works_on_proximal_splitting}

In recent decades, a plethora of proximal or operator splitting techniques, see, \eg,~\cite{Combettes2011, Boyd2011, Parikh2013, Komodakis_and_Pesquet:2015, Glowinski2016, Beck2017, Davis2017, Yan_PD3O:2018, Ryu_Yin_ls_book_2022_draft, Kant_journal_msp_top_admm:2020}, have been proposed in the literature. {Although the convergence of many proximal/operator splitting algorithms is proven only for convex settings, these techniques can still be employed to solve many nonconvex problems prevalent in machine learning problems (without convergence or performance guarantees).} In the past decade, \ac{ADMM}-like methods, see, \eg,~\cite{Boyd2011, Ryu_Yin_ls_book_2022_draft}, have enjoyed a renaissance because of their wide applicability in large-scale distributed machine learning problems by breaking down a large-scale problem into easy-to-solve smaller problems. However, these proximal splitting techniques are typically first-order methods, which can be very slow to converge to solutions having high accuracy. Nonetheless, modest accuracy can be sufficient for many practical \ac{FL} applications. 

Operator splitting for two operators, or loosely speaking, optimization problems, in which the objective function is given by the sum of two functions, have recently been employed for \ac{FL}~\cite{FedSplit__2020, FedDR__2021, Saber_etal__FL_from_splitting_algo__2021, zhou_and_li_fedadmm2022, zhou_iceadmm_draft_2022}. However, some \ac{FL} optimization problems --- see problem~\eqref{eqn_chD:three_operator__FL_optimization_problem} ---
can be cast as a composite sum of three functions comprising smooth and non-smooth functions. Unfortunately, operator splitting with more than two composite terms in the objective function are either not straightforward or converges slowly~\cite{Davis2017, Yan_PD3O:2018}. Recently, some authors have extended the primal-type or dual-type, including \ac{ADMM}-type algorithms, as well as the primal-dual classes of splitting algorithms from two operators to three operators \cite{Wang_Banerjee__Bregman_ADMM__2014, Davis2017, Yan_PD3O:2018, Banert2020, Kant_journal_msp_top_admm:2020, Combettes_data_science__2021, Ryu_Yin_ls_book_2022_draft, Condat__prox_splitting:2021, Hong_and_Luo__2017}.

\subsection{Notation and Paper Organization}
Let the set of complex  and real numbers be denoted by $\mathbb{C}$ and $\mathbb{R}$, respectively. $\Re\{x\}$ denotes the real part of a complex number $x \!\in\!\Cm$. 
The $i$-th element of a vector $\vec{a} \! \in \! \mathbb{C}^{m \times 1}$ and $j$-th column vector of a matrix $\mat{A} \! \in\! \mathbb{C}^{m \times n}$ are denoted by ${a}[i] \!\coloneqq \left(\vec{a}\right)_i \!\! \in \! \mathbb{C}$ 
and  $\mat{A}\left[:, j\right] \! \in \! \mathbb{C}^{m \times 1}$, respectively. 
We form a matrix by stacking the set of vectors $\left\{ \vec{a}[n]\! \in \! \mathbb{C}^{M \times 1} \right\}_{n=1}^N$ and $\left\{ \vec{b}[m] \! \in \! \mathbb{C}^{1 \times N} \right\}_{m=1}^M$ column-wise and row-wise as {$\mat{A} \!= \! \left[ \vec{a}[1],\ldots,\vec{a}[N] \right] \! \in \! \mathbb{C}^{M \times N}$ and $\mat{B} \! = \! \left[ \vec{b}[1];\ldots;\vec{b}[M] \right] \! \in \! \mathbb{C}^{M \times N} $}, respectively. 
The transpose and conjugate transpose of a vector or matrix are denoted by $\left(\cdot\right)^{\rm T}$ and $\left(\cdot\right)^{\herm}$, respectively. The complex conjugate is represented by $\left(\cdot\right)^*$. The $K  \times  K$ identity matrix is written as $\vec{I}_K$. 
An $i$-th iterative update reads $(\cdot)^{(i)}$. 


The remainder of the paper is organized as follows. In the next section, we introduce the \ac{TOP-ADMM} technique. In Section~\ref{sec_chD:fl_topadmm_algorithm}, we establish our proposed {\FedTOPADMM} algorithm. In Section~\ref{sec_chD:simulation_results}, we present the numerical results, and in Section~\ref{sec_chD:conclusion_future_work} we conclude with a summary and future work. Appendix~\ref{appendix:use_lemmas_def} contains some useful definitions and lemmas. Appendix~\ref{appendix:convergence_analysis_topadmm} presents the completely novel convergence proof of our recently proposed \ac{TOP-ADMM}~\cite{Kant_journal_msp_top_admm:2020} algorithm.

\section{Introduction to the TOP-ADMM Algorithm} \label{sec_chD:topadmm_algorithm}
In this section, we firstly introduce the classical two-operator consensus \ac{ADMM}. Subsequently, we present our recently proposed \ac{TOP-ADMM} algorithm~\cite{Kant_journal_msp_top_admm:2020, Kant_etal__CFR_TOPADMM_2022}. 

\subsection{Classical Consensus ADMM}
The two-operator consensus \ac{ADMM} \cite{Boyd2011, Parikh2013} is a popular {method in the optimization and machine learning communities} to solve problems of the form 
\begin{align} 
\begin{aligned} \label{eqn_chD:two_operator_general_consensus_admm}
&\underset{ \left\{{\vec{x}}_m \in \Cm^n \right\}\!,  {\vec{z}} \in \Cm^n}{\text{minimize}}  
 \quad \sum_{m=1}^M \!f_m\!\left( {{\vec{x}}_m} \right) \! + \! g\!\left( {\vec{z}} \right)  \\
&\text{subject to}  \qquad 
   {{\vec{x}}_m} \!-\! {\vec{z}} \!=\! {\vec{0}}, \quad \forall m=1,\ldots,M,
\end{aligned}
\end{align}
where $\left\{f_m \right\}$ and $g(\cdot)$ are closed, convex, and proper functions. The classical \ac{ADMM} algorithm that solves problem \eqref{eqn_chD:two_operator_general_consensus_admm} can be summarized, following \cite[Chapter~7]{Boyd2011}, as follows
\iffalse
\begin{align*}
    {\vec{x}}_m^{\left(i+1\right)} \! 
    &\coloneqq \! \arg\min_{{\vec{x}}_m}  f_m\!\left(\! {\vec{x}}_m \!\right) \! + \! 2\Re\!\left\{{\left(\vec{y}_m^{\left(i\right)}\right)}^{\herm} \! \left( {{\vec{x}}_m} \! - \! {\vec{z}}^{\left(i\right)} \right) \right\} \! + \! \rho \! \left\| \! {{\vec{x}}_m} \! - \! \vec{z}^{\left(i\right)} \! \right\|_2^2 \ \forall m=1,\ldots,M,  \\
    &\equiv \arg\min_{{\vec{x}}_m}  f_m\left( {\vec{x}}_m \right) \! + \! \rho \left\| {{\vec{x}}_m} \! - \! {\vec{z}^{\left(i\right)}} + \frac{{\vec{y}}_m^{\left(i\right)}}{\rho} \right\|_2^2, \ \forall m,  \\ 
    {\vec{z}}^{\left(i+1\right)} \! 
    &\coloneqq \! \arg\min_{{\vec{z}}}  g\!\left(\! {\vec{z}} \!\right) \! + \!  \sum_{m=1}^M \! \left(\! 2\Re\!\left\{{\left(\vec{y}_m^{\left(i\right)}\right)}^{\herm} \left( {{\vec{x}}_m^{\left(i+1\right)}} \! - \! {\vec{z}} \right)\right\} \! + \! \rho \left\| \! {{\vec{x}}_m^{\left(i+1\right)}} \! - \! {\vec{z}} \! \right\|_2^2 \right) \\
    &\equiv \! \arg\min_{{{\vec{z}}}}  g\!\left( {\vec{z}} \right)  +   \sum_{m=1}^M  \rho \left\|  {{\vec{x}}_m^{\left(i+1\right)}} \! - \! {\vec{z}} \! + \!  \frac{{\vec{y}}_m^{\left(i\right)}}{\rho} \right\|_2^2,  \\
    {\vec{y}}_m^{\left(i+1\right)} &\coloneqq {\vec{y}}_m^{\left(i\right)} + \rho \left( {{\vec{x}}_m^{\left(i+1\right)}} - {\vec{z}^{\left(i+1\right)}} \right), \quad \forall m=1,\ldots,M.
\end{align*}
\else
\begin{subequations} \label{eqn_chD:classical_consensus_admm_algorithm}
\begin{alignat}{3} \label{eqn_chD:update_x_m__classical_consensus_admm_algorithm}
    {\vec{x}}_m^{\left(i+1\right)} \! 
    &\coloneqq \!  \arg\min_{{\vec{x}}_m}  f_m\!\left( {\vec{x}}_m \right)  +  \rho \left\| {{\vec{x}}_m} \! - \! {\vec{z}^{\left(i\right)}} \!+\! \frac{{\vec{y}}_m^{\left(i\right)}}{\rho} \right\|_2^2, \ \forall m, \\ 
    \label{eqn_chD:update_z__classical_consensus_admm_algorithm}
    {\vec{z}}^{\left(i+1\right)} \! 
    &\coloneqq \! \arg\min_{{{\vec{z}}}}  g\!\left( {\vec{z}} \right)  +   \sum_{m=1}^M  \rho \left\|  {{\vec{x}}_m^{\left(i+1\right)}} \! - \! {\vec{z}} \! + \!  \frac{{\vec{y}}_m^{\left(i\right)}}{\rho} \right\|_2^2,  \\
    \label{eqn_chD:dual_update_y_m__classical_consensus_admm_algorithm}
    {\vec{y}}_m^{\left(i+1\right)} &\coloneqq {\vec{y}}_m^{\left(i\right)} \!+\! \rho \left( {{\vec{x}}_m^{\left(i+1\right)}} \!-\! {\vec{z}^{\left(i+1\right)}} \right), \  \forall m=1,\ldots,M,
\end{alignat}
\end{subequations}
where ${\vec{y}}_m^{\left(i\right)} \!\in\!\Cm^n$ is the Lagrange multiplier and $\rho \! \in \! \Rm_{>0}$ is a penalty parameter. 

If the functions $f_m$ and $g$ contain a computationally inconvenient quadratic term, then one could employ proximal linearized \ac{ADMM}; see, \eg,~\cite{Ryu_Yin_ls_book_2022_draft} and references therein to cancel out such a quadratic term. More specifically, in linearized \ac{ADMM}, one adds a so-called proximal term $\left\| \vec{x}_m \!-\! \vec{x}_m^{\left(i\right)}\right\|_{\mat{Q}_{x}}$ and $\left\| \vec{z} \!-\! \vec{z}^{\left(i\right)}\right\|_{\mat{Q}_{z}}$ with positive definite matrices $\mat{Q}_{x}$ and $\mat{Q}_{z}$  to~\eqref{eqn_chD:update_x_m__classical_consensus_admm_algorithm} and~\eqref{eqn_chD:update_z__classical_consensus_admm_algorithm}, respectively.

\fi

\subsection{Consensus TOP-ADMM}

{
The recently proposed \skblack{three-operator-based\footnote{\skblack{As briefly mentioned in Section~\ref{sec_chD:related_works_on_proximal_splitting}, operator splitting with more than two composite terms in the objective has been an open research problem without resorting to problem reformulations or product space reformulations~\cite{Pierra1984, Stathopoulos2016}. The reason is that it has been found empirically that such product space formulations and problem reformulations may not be straightforward and generally may be slow to converge or not be feasible for some problems~\cite{Davis2017, Yan_PD3O:2018, Kant_etal__intro_to_TOPADMM_2022}.}}} \ac{TOP-ADMM}~\cite{Kant_journal_msp_top_admm:2020} is one of the generalized algorithms for the classical consensus \ac{ADMM}. 
The \ac{TOP-ADMM} can be employed to solve many centralized and distributed optimization problems in signal processing\skblack{, which} are difficult to solve \skblack{by} traditional ADMM methods.
Specifically, the \ac{TOP-ADMM} algorithm has been used to solve spectrum shaping via spectral precoding and peak-to-average power ratio reduction in multiple carriers-based wireless communication systems, such as 5G cellular networks~\cite{Kant_journal_msp_top_admm:2020, Kant_etal__CFR_TOPADMM_2022}, which were difficult to solve with the classical two-operator \ac{ADMM} methods. \skblack{More precisely, we can refer to~\cite[Figure~4 and Figure~5]{Kant_journal_msp_top_admm:2020} as a concrete and realistic example within wireless communications that shows numerically that two-operator ADMM-like algorithm solving three-operator problem exhibit slow convergence compared to TOP-ADMM.}
}

{
First, let us understand the shortcomings of classical (linearized) \ac{ADMM} and subsequently present \ac{TOP-ADMM} to deal with the demerits of \ac{ADMM}. To this end, 
let us add {a} convex \mbox{$L$-smooth} ($L\!>\!0$) function $h$ to the cost function of}~\eqref{eqn_chD:two_operator_general_consensus_admm} with some scaling $\beta \!\in\!\Rm_{>0}$ such that the problem formulation~becomes 
\iffalse
\begin{subequations}
\begin{align}
\label{eqn_chD:general_consensus_top_admm__generic_form__nonconsensus}
&\sum_{m=1}^M f_m \left( {{\vec{x}}_m} \right) \! + \! g\left( {{\vec{z}}} \right) \! + \! \beta h\left( {{\vec{z}}} \right) \\
&\sum_{m=1}^M \left[ f_m \left( {{\vec{x}}_m} \right) \! + \! \frac{1}{M} \! h\!\left( {{\vec{x}}_m} \right)\right]  \! + \! g\left( {{\vec{z}}} \right)  
\end{align}
\end{subequations}
\else
\begin{align} 
\begin{aligned} \label{eqn_chD:general_consensus_top_admm__generic_form}
&\underset{ \left\{{\vec{x}}_m \in \Cm^n \right\}\!,  {\vec{z}} \in \Cm^n}{\text{minimize}}  
 \quad \sum_{m=1}^M \!f_m\!\left( {{\vec{x}}_m} \right) \! + \! g\!\left( {\vec{z}} \right) \! + \! \beta  h\left( {{\vec{z}}} \right)  \\
&\text{subject to}  \qquad 
   {{\vec{x}}_m} \!-\! {\vec{z}} \!=\! {\vec{0}}, \quad \forall m=1,\ldots,M.
\end{aligned}
\end{align}
\fi


There are at least two possibilities to tackle problem~\eqref{eqn_chD:general_consensus_top_admm__generic_form} using the classic two-operator \ac{ADMM}: 1) product space or problem reformulation \cite{Pierra1984, Stathopoulos2016}, and 2) consider $\overline{g}\left( {{\vec{z}}} \right) \!\coloneqq \! g\left( {{\vec{z}}} \right) \! + \beta h\left( {{\vec{z}}} \right)$ or $\overline{f}_m\left( {{\vec{x}_m}} \right) \!\coloneqq\! f_m\left( {{\vec{x}_m}} \right) \! + ({\beta}/{M}) h\left( {{\vec{x}_m}} \right)$. The first approach is either not straightforward or {the} resulting algorithm {converges slowly} \cite{Davis2017, Yan_PD3O:2018}. The second approach yields a subproblem, \ie, ${\vec{z}}^{\left( i+1 \right)}$ or ${\vec{x}}_m^{\left(i+1\right)}$ update in \eqref{eqn_chD:classical_consensus_admm_algorithm}, which may not have {a computationally efficient} solution in general. Furthermore, {the classical} two-operator \ac{ADMM} does not 
{exploit the smoothness of $h$.} Recall that if $h$ is a quadratic function, then clubbing it with either update~\eqref{eqn_chD:update_x_m__classical_consensus_admm_algorithm} or~\eqref{eqn_chD:update_z__classical_consensus_admm_algorithm} and employing linearized \ac{ADMM} can solve the respective subproblem. However, this approach does not necessarily yield a better performance than the \ac{TOP-ADMM}. 

In essence, the \ac{TOP-ADMM} algorithm can also be classified as a divide-and-conquer method, which decomposes a large optimization problem --- difficult to solve in a composite form~\eqref{eqn_chD:general_consensus_top_admm__generic_form} --- into smaller subproblems that are easy to solve. 

{
In the following, we lay out the general definition of our \ac{TOP-ADMM} method, which was introduced in~\cite{Kant_journal_msp_top_admm:2020}. For completeness, we present the convergence proof of \ac{TOP-ADMM}, {which is novel and was not present} in~\cite{Kant_journal_msp_top_admm:2020}. 
}
\begin{theorem}[TOP-ADMM] \label{thm_ChD:definition_of_topadmm_algorith__feasible_problem}
{Consider a problem given in~\eqref{eqn_chD:general_consensus_top_admm__generic_form} with at least one solution and a suitable step-size $\tau \! \in \! \Rm_{\geq 0}$.  Assume subproblems~\eqref{eqn_chD:update_xm__step1_parallel__general_top_admm__prox__for_convergence} and \eqref{eqn_chD:update_z__step2__general_top_admm__prox__for_convergence} have solutions, and consider a relaxation/penalty parameter $\rho \! \in \! \Rm_{> 0}$ and some arbitrary initial values $ \left(\left\{\vec{x}_m^{(0)}\right\}, {\vec{z}^{(0)}}, \left\{\vec{y}_m^{(0)}\right\}\right)$. Then, } the generated sequences $ \left(\left\{\vec{x}_m^{(i)}\right\}, {\vec{z}^{(i)}}, \left\{\vec{y}_m^{(i)}\right\}\right)$ by the following iterative scheme
\begin{subequations} \label{eqn_chD:generalized_top_admm_algorithm_iterates__for_convergence}
\begin{align}
    \label{eqn_chD:update_xm__step1_parallel__general_top_admm__prox__for_convergence}
    {\vec{x}}_m^{\left(i+1\right)} \! 
    =& \arg\min_{{\vec{x}}_m}  f_m\!\left( {\vec{x}}_m \right) \!+ \! \rho \! \left\|\!  {{\vec{x}}_m} \! - \! {\vec{z}^{\left(i\right)}} \! +\! \frac{{\vec{y}}_m^{\left(i\right)}}{\rho} \!\right\|_2^2, \forall m\!, \\ 
    \label{eqn_chD:update_z__step2__general_top_admm__prox__for_convergence}
    {\vec{z}}^{\left(i+1\right)} \! 
    =& \arg\min_{{{\vec{z}}}}  g\!\left( {\vec{z}} \right)  \nonumber \\
    &+\!   \sum_{m=1}^M  \rho \left\|  {{\vec{x}}_m^{\left(i+1\right)}} \! - \! {\vec{z}} \! {- \tau \nabla h \! \left( \! {\vec{z}}^{\left( i \right)} \right)} \! + \!  \frac{{\vec{y}}_m^{\left(i\right)}}{\rho} \right\|_2^2, \\
    \label{eqn_chD:update_dual_ym__step3_parallel__general_top_admm}
    {\vec{y}}_m^{\left(i+1\right)} =& {\vec{y}}_m^{\left(i\right)} + \rho \left( {{\vec{x}}_m^{\left(i+1\right)}} - {\vec{z}^{\left(i+1\right)}} \right), \ \forall m=1,\ldots,M,
\end{align}
\end{subequations}
at any limit point, converges to a \ac{KKT} stationary point of~\eqref{eqn_chD:general_consensus_top_admm__generic_form}.
\begin{proof}
See Appendix~\ref{appendix:convergence_analysis_topadmm}.
\end{proof}
\end{theorem}

{
Note that the classical consensus \ac{ADMM} algorithm is a special case of our proposed \ac{TOP-ADMM} algorithm when $h\!=\!0$ in~\eqref{eqn_chD:general_consensus_top_admm__generic_form} or $\nabla h\!=\!0$ in~\eqref{eqn_chD:generalized_top_admm_algorithm_iterates__for_convergence}. Although classical consensus \ac{ADMM} can solve many problems in machine learning, it does not necessarily yield an implementation-friendly algorithm, particularly, if the proximal operator of the $L$-smooth function $h$ is inefficient to compute.  
}


\section{Federated Learning using TOP-ADMM} \label{sec_chD:fl_topadmm_algorithm}

\ac{FL} can used to solve the distributed consensus problem~\eqref{eqn_chD:general_consensus_top_admm__generic_form}. Furthermore, we envision learning on the edge server, say with a loss function $h$ in~\eqref{eqn_chD:general_consensus_top_admm__generic_form} with some constraint or regularizer expressed by $g$ in~\eqref{eqn_chD:general_consensus_top_admm__generic_form}. Therefore, we pose the generic distributed problem~\eqref{eqn_chD:general_consensus_top_admm__generic_form} for \ac{FL} using new variables as given below: 
\begin{align}
\underset{ \left\{{\vec{w}}_m \in \Rm^n \right\},  {\vec{w}} \in \Rm^n}{\text{minimize}}& \!  \sum_{m=1}^M \! \alpha_m f_m\left( \vec{w}_m; \pazocal{D}_m \right) + g\left( \vec{w}\right) + \beta h\left( \vec{w}; \pazocal{D} \right) \nonumber \\ 
\label{eqn_chD:three_operator__FL_optimization_problem}
\text{subject to}& 
  \ \  \vec{w}_m - {\vec{w}} = {\vec{0}}, \quad \forall m=1,\ldots,M,
\end{align}
where the global weight vector of the considered learning model is given by $\vec{w} \!\in\! \Rm^n$ and the weight vector of user $m$ corresponds to $\vec{w}_m \!\in\! \Rm^n$. Furthermore, the loss function at user $m$ and the server are denoted by $f_m$ and $h$ with training dataset $\pazocal{D}_m$ and $\pazocal{D}$, respectively. The loss function $f_m$ of user $m$ is weighted by $\alpha_m \!\geq\!0$ satisfying $\sum_m \alpha_m \!=\! 1$, and the server's loss function $h$ is weighted by some nonnegative $\beta \!\geq\!0$. The edge server is expected to be collocated at the base station such that the server learns the global model using the data $\pazocal{D}$ generated or stored on the server side together with the learning from the data generated/available on the users' side. Observe that in many existing federated learning frameworks, such as {\FedADMM}~\cite{zhou_and_li_fedadmm2022}, $g\!=\!0$ and $h\!=\!0$ in~\eqref{eqn_chD:three_operator__FL_optimization_problem}. 

We are now ready to apply the \ac{TOP-ADMM} algorithm to~\eqref{eqn_chD:three_operator__FL_optimization_problem} for the \ac{FL} purpose with some modifications. Firstly, we swap the update order of \ac{TOP-ADMM}, \ie, \eqref{eqn_chD:update_xm__step1_parallel__general_top_admm__prox__for_convergence} and \eqref{eqn_chD:update_z__step2__general_top_admm__prox__for_convergence}. As suggested in~\cite[Chapter~5]{Glowinski2016}, the sequence updates of \ac{ADMM} for ${\vec{z}}^{\left(\skblack{i}+1\right)}$ in \eqref{eqn_chD:update_z__classical_consensus_admm_algorithm} and ${\vec{x}}_m^{\left(\skblack{i}+1\right)}$ in \eqref{eqn_chD:update_x_m__classical_consensus_admm_algorithm} are performed in Gauss-Seidel fashion. Specifically, one can interchange these updates without penalizing the convergence guarantee but the generated sequences over iterations may be different in general. Since \ac{TOP-ADMM} generalizes \ac{ADMM}, we can also interchange the updates. Secondly, we add a proximal term scaled by the parameter $\zeta^{\left(i\right)}$ to the subproblem in~\eqref{eqn_chD:update_z__classical_consensus_admm_algorithm}. Lastly, we employ a so-called Glowinski's relaxation factor $\gamma \!\in\!\left(0,2 \right)$ to the dual update, see, \eg, \cite{Ryu_Yin_ls_book_2022_draft}. Hence, the \ac{TOP-ADMM} algorithm tackling \eqref{eqn_chD:three_operator__FL_optimization_problem} can be summarized as follows:
\begin{subequations} \label{eqn_chD:top_admm_algorithm_for_FL}
\begin{align}
    \label{eqn_chD:step1__FedTOPADMM_algorithm}
    {\vec{w}}^{\left(i+1\right)} \! 
    =& \arg\min_{\vec{w}}  g\!\left( {\vec{w}} \right)  \nonumber \\
    &+\!   \sum_{m=1}^M  \!\frac{\rho_m}{2} \left\|  {{\vec{w}}_m^{\left(i\right)}} \! - \! {\vec{w}} \! {- \tau^{\left(i\right)} \nabla h \! \left( \! {\vec{w}}^{\left( i \right)} \right)} \! + \!  \frac{\bm{\lambda}_m^{\left(i\right)}}{\rho_m} \!\right\|_2^2  \\ 
    &+\! \frac{\zeta^{\left(i\right)}}{2} \left\| {\vec{w}} \!-\! \vec{w}^{\left(i\right)}   \right\|_2^2, \nonumber \\
    \label{eqn_chD:step2__FedTOPADMM_algorithm}
    {\vec{w}}_m^{\left(i+1\right)} \! 
    =& \arg\min_{\vec{w}_m}  \alpha_m f_m\!\left(\! \vec{w}_m \right) \!+ \! \frac{\rho_m}{2} \! \left\|\!  \vec{w}_m \! - \! \vec{w}^{\left(i\right)} \! +\! \frac{{\bm{\lambda}}_m^{\left(i\right)}}{\rho_m} \! \right\|_2^2\!,  \forall m\!, \\
    \label{eqn_chD:step3__FedTOPADMM_algorithm}
    {\bm{\lambda}}_m^{\left(i+1\right)} =& {\bm{\lambda}}_m^{\left(i\right)} + \gamma \rho_m \left( {\vec{w}_m^{\left(i+1\right)}} - {\vec{w}^{\left(i+1\right)}} \right), \ \forall m=1,\ldots,M,
\end{align}
\end{subequations}
where \skblack{step size $\tau^{(i)} \!\in\! \Rm_{\geq 0}$ and proximity parameter $\zeta^{(i)} \!\in\! \Rm_{\geq 0}$ are adaptive over iterations, \ie, more concretely, $\tau^{(i+1)} \!\leq\! \tau^{(i)}$, and $\zeta^{(i+1)} \!\leq\! \zeta^{(i)}$ for all $i\!=\!0\!,1\!,2\!,\ldots$. } 

Unfortunately, a direct application of the enhanced \ac{TOP-ADMM} algorithm~\eqref{eqn_chD:top_admm_algorithm_for_FL} in the \ac{FL} context would incur high communication costs due to the exchange of parameters among the server and the selected clients in every global iteration. Hence, \ac{TOP-ADMM} may negatively affect the communication efficiency. Additionally, at a given global iteration in the 
\ac{TOP-ADMM} algorithm~\eqref{eqn_chD:top_admm_algorithm_for_FL}, each client is expected to send and receive the updates to the server synchronously. Unfortunately, not all of the involved clients can transmit and receive the updated parameters in \ac{FL} due to the limited communication bandwidth and their computational capabilities.  
Considering the aforementioned limitations of the direct application of \ac{TOP-ADMM} for the \ac{FL}, we extend the \ac{TOP-ADMM} algorithm using the \ac{FL} framework introduced in~\cite{zhou_and_li_fedadmm2022}. 
Therefore, we establish a novel algorithm named {\FedTOPADMM} described in the next section catering to the \ac{FL} purpose. 

We would like to accentuate that it is unclear how to extend the existing {\FedADMM}~\cite{zhou_and_li_fedadmm2022} framework to support learning on the server, \ie, described by a loss function $h$ in~\eqref{eqn_chD:three_operator__FL_optimization_problem} while supporting a nonsmooth regularizer/function $g$ for the distributed learning. Ignoring $g$ for the moment in~\eqref{eqn_chD:three_operator__FL_optimization_problem}, one could argue to artificially add yet another parallel client on the server in the existing {\FedADMM}~framework, \ie, $h\!\equiv\!f_{M+1}$. However, the additional virtual client does not necessarily yield better convergence performance. Therefore, we evolve the {\FedADMM}~\cite{zhou_and_li_fedadmm2022} framework using our \ac{TOP-ADMM} algorithm. Subsequently, we show in the numerical Section~\ref{sec_chD:simulation_results} that {\FedTOPADMM}~renders superior performance over {\FedADMM}~with additional virtual client---see, \eg, Figure~\ref{fig_chD:cmp_perf_fedtopadmm1_and2__fedadmm_and_fedadmmvc}. Recall that the framework proposed in~\cite{zhou_and_li_fedadmm2022} is based on the classical two-operator consensus \ac{ADMM}---cf.~\eqref{eqn_chD:two_operator_general_consensus_admm} with $g\!=\!0$. Therefore, the proposed \ac{FL} using \ac{TOP-ADMM}, \ie, {\FedTOPADMM}, not only inherits all the properties of classical two-operator \ac{ADMM}, \ie, \FedADMM~\cite{zhou_and_li_fedadmm2022}, but also additionally exploits the $L$-smooth function on the server. In other words, {\FedADMM}~is a special case of {\FedTOPADMM}.  

In the sequel, we establish the {\FedTOPADMM} algorithm that is built on the \linebreak{\FedADMM}~\cite{zhou_and_li_fedadmm2022} framework utilizing the \ac{TOP-ADMM} algorithm.

\setlength{\textfloatsep}{5pt}
\begin{algorithm}[!thp]
\caption{\texttt{FedTOP-ADMM I/II}}\label{alg:fed_top_admm}
\begin{algorithmic}[1] 
\State \textbf{Input and Initialization}:  Choose $\tau^{\left(0 \right)} \!\geq\!0$, $\zeta^{\left(0 \right)} \!\geq\!0$, $\left\{ \tau_{m}^{\left(0 \right)}\right\}$, $\left\{\zeta_{m}^{\left(0 \right)}\right\}$; $\pazocal{U}^{\left( 0\right)} \!\coloneqq\! [M]$; $\vec{w}^{\left( 0\right)}$; $\left\{\vec{w}_m^{\left( 0\right)}\right\}$, $\left\{\bm{\lambda}_m^{\left( 0\right)}\right\}$; $\left\{ \mat{Q}_m \right\}$; $\left\{ \rho_m \right\}$; $\left\{ \alpha_m \right\}$; and $\vec{u}_m^{\left( 0\right)} \!\coloneqq\! \rho_m \vec{w}_m^{\left( 0\right)} \!+\! \bm{\lambda}_m^{\left( 0\right)}$ $\forall m$; $\nu^{(0)} \!\coloneqq\! \nicefrac{1}{\left( \sum_{m \in [M] } \rho_m + \zeta^{(0)} \right)}$; $\pazocal{D}, \left\{\pazocal{D}_m \right\}$; $\gamma \!\in\!(0,2)$
\For{$i=0,\ldots, I-I$}
    \State [\textbf{Server weight updates}]
    \If {\FedTOPADMMI~OR~(\FedTOPADMMII~AND~{$i \! \notin \! \pazocal{P}$) }}
    \begin{align*}
        \vec{y}^{\left( i+1\right)} 
            &\coloneqq - \tau^{\left( i\right)} \nabla h\!\left( \vec{w}^{\left( i\right)}; \pazocal{D} \right) + \zeta^{(i)} \vec{w}^{\left( i\right)}
    \end{align*}
    \Else \hspace{8mm} $\vec{y}^{\left( i+1\right)} \leftarrow \vec{y}^{\left( i\right)}$
    \EndIf
    \Statex \hspace{3.5mm} \underline{\textrm{Uplink communications and global parameter updates}}:
    \If {$i \! \in \! \pazocal{P}$ } \textit{ \% Uplink communications with users' selection}
    \State [\textbf{Server receives updated weights}] $\left\{ \vec{v}_m^{\left(i+1\right)} \!\leftarrow \!\vec{u}_m^{\left(i\right)} \right\}$
    \State [\textbf{Client selection}] $\pazocal{U}^{\left( i+1\right)} \!\subseteq \! [M]$ for downlink communications
    \Else
    \State [\textbf{Server utilizes previous weights}] $\left\{ \vec{v}_m^{\left(i+1\right)} \!\leftarrow \!\vec{v}_m^{\left(i\right)} \right\}$
    \EndIf
    \State [\textbf{Global parameter aggregation/update}] 
    \begin{align*} 
        \vec{w}^{\left( i+1\right)} \! 
    =& {\prox}_{\nu^{\left( i\right)} g}\!\left( \!\frac{1}{\nu^{\left( i\right)}} \! \left[\sum_{m=1}^M \!\vec{v}_m^{\left( i\right)} + \vec{y}^{\left( i\right)} \right] \! \right)
    \end{align*}
    \Statex \hspace{3.5mm} \underline{\textrm{Downlink communications and local parameter updates}}: \%\textit{Parallel}
    \State [\textbf{Server multicasts}]  
    \For{every $m \in  \pazocal{U}^{\left(i\right)}$} \textit{ \% update sequences}
    \If {$i \! \in \! \pazocal{P}$ } \%\textit{Selected users receive the updated weights} 
    \begin{align*}
        \hspace{-24mm} \vec{v}^{\left(i+1\right)} \coloneqq \vec{w}^{\left(i+1\right)}
    \end{align*}
    \Else \hspace{3.5mm} $\vec{v}^{\left(i+1\right)} \leftarrow \vec{v}^{\left(i\right)}$
    \EndIf
    \State [\textbf{User weight updates}] 
    \begin{align*}
        \Delta \vec{z}_m \coloneqq& \rho_m \left(\vec{w}_m^{\left(i\right)} \!-\!  \vec{v}^{\left(i+1\right)} \right) \!+\! \alpha_m \nabla f_m \! \left( \vec{w}_m^{\left(i\right)}; \pazocal{D}_m \right) \!+\! \bm{\lambda}_m^{\left(i\right)}  \\
        \vec{w}_m^{\left(i+1\right)} 
            \approx& \vec{w}_m^{\left(i\right)} \!-\! \left(\alpha_m \mat{Q}_m \!+\! \rho_m \I\right)^{-1}  \Delta \vec{z}_m \\
        \bm{\lambda}_m^{\left(i+1\right)} 
            =& \bm{\lambda}_m^{\left(i\right)} \!+\!  \gamma \rho_m \left( \vec{w}_m^{\left(i+1\right)} \!-\! \vec{v}^{\left(i+1\right)} \right) \\
        \vec{u}_m^{\left( i+1\right)} 
            \coloneqq& \rho_m \vec{w}_m^{\left( i+1\right)} \!+\! \bm{\lambda}_m^{\left( i+1\right)}
    \end{align*}
    \State 
    \EndFor
    \For{every $m \notin  \pazocal{U}^{\left(i\right)}$}  
    \State $\vec{u}_m^{\left( i+1\right)} = \vec{u}_m^{\left( i\right)}$
    \EndFor
\EndFor
\label{algorithm}
\end{algorithmic}
\end{algorithm}

\subsection{FedTOP-ADMM: Communication-Efficient Algorithm}
We present in Algorithm~\ref{alg:fed_top_admm} our novel {\FedTOPADMM} using the \ac{TOP-ADMM}~\eqref{eqn_chD:top_admm_algorithm_for_FL} and {\FedADMM}~\cite{zhou_and_li_fedadmm2022} framework for communication-efficient \ac{FL}. Specifically, we propose two variants of {\FedTOPADMM} algorithms, which are referred to as {\FedTOPADMMI} and {\FedTOPADMMII}. In {\FedTOPADMMI}, we learn the model on the server side continuously in every global iteration. Conversely, in {\FedTOPADMMII} we learn the considered global model when the server is not communicating and aggregating the parameters of the model, \ie, the server learns in parallel to the selected users. Notice that we refer to {\FedTOPADMMI} and \texttt{II} as common {\FedTOPADMM}, where the performance difference between \texttt{I} and \texttt{II} is non-noticeable. Hence, {\FedTOPADMM} corresponds to both {\FedTOPADMMI} and \linebreak{\FedTOPADMMII} depending on the context. 

As our proposed {\FedTOPADMM} generalizes the {\FedADMM} algorithm, it inherits all the properties of the {\FedADMM} algorithm, including its communication efficiency {when using only data available on the edge devices}. 
In the {\FedTOPADMM}, there are three additional hyperparameters compared to {\FedADMM}, namely $\tau^{\left(i \right)}$, $\zeta^{\left(i \right)}$, and $\gamma$. Meanwhile, note that {\FedADMM} requires tuning of $\left\{ \mat{Q}_m \right\}$ and $\left\{ \rho_m \right\}$ parameters---see~\cite{zhou_and_li_fedadmm2022} and also discussion in Section~\ref{sec_chD:simulation_results}. 

\begin{remark}
Let \skblack{$g\!=\!0$, $\gamma\!=\!1$, and} $\tau^{\left(i\right)}\!=\!0$ \skblack{or equivalently $h\!=\!0$} and  $\zeta^{\left(i\right)}\!=\!0$ for all \skblack{$i\!=\!0,1,2,\ldots$}. Then, {\FedTOPADMMI} boils down to {\FedADMM} and all the convergence results of {\FedADMM}~\cite[Section~4.3]{zhou_and_li_fedadmm2022} hold for any~\skblack{$J\!\geq\! 1$}. 
\end{remark}

\begin{remark}
Suppose $J\!=\!1$, $\tau^{\left(i\right)}\!=\!\tau \!\in\!\Rm_{\geq 0}$ and  $\zeta^{\left(i\right)}\!=\!0$ for all $i$, and $\rho_m \!=\! \rho$ for all $m$. Then, {\FedTOPADMMII} becomes \ac{TOP-ADMM}; and consequently, the global convergence Theorem~\ref{thm_ChD:definition_of_topadmm_algorith__feasible_problem} holds for {\FedTOPADMMII}. 
\end{remark}

\iffalse
\skblack{
To establish the convergence of {\FedTOPADMM} Algorithm~\ref{alg:fed_top_admm}, we first present the following proposition that extends the vanishing property of residual errors proved in Theorem~\ref{thm_ChD:definition_of_topadmm_algorith__feasible_problem}. 
Subsequently, using this proposition, we establish the global convergence of  
{\FedTOPADMM} 
in Theorem~\ref{theorem:global_convergence_fedtopadmm}.
}
\else
\skblack{
To establish the convergence of {\FedTOPADMM} Algorithm~\ref{alg:fed_top_admm}, we extend the vanishing property of residual errors proved in Theorem~\ref{thm_ChD:definition_of_topadmm_algorith__feasible_problem}. Using this result, we establish the global convergence of {\FedTOPADMM} 
in the following theorem. 
}
\fi

\iffalse
\skblack{
\begin{prop}\label{prop_chD:extension_of_fedtop}
Consider a problem given in~\eqref{eqn_chD:three_operator__FL_optimization_problem} under general convex settings with at least one solution. Assume all the subproblems of \eqref{eqn_chD:top_admm_algorithm_for_FL} or Algorithm~\ref{alg:fed_top_admm} have solutions. Also, in Algorithm~\ref{alg:fed_top_admm}, for each given period $T$, all users should be active at least once within the users active sets $\left\{\pazocal{U}^{\left(i+1\right)}, \pazocal{U}^{\left(i+2\right)},\ldots,\pazocal{U}^{\left(i+T\right)} \right\}$ for all $i\!\geq\!0$. Consider suitable non-increasing step-size $\left\{\tau^{(i)} \! \in \! \Rm_{\geq 0}\right\}$ and $\left\{\zeta^{(i)} \! \in \! \Rm_{\geq 0} \right\}$ for $i\!\geq0\!$, 
and relaxation/penalty parameter $\left\{\rho_m \! \in \! \Rm_{> 0}\right\}$ with some arbitrary initial $\{{\vec{w}}^{(0)}, {\vec{w}}_m^{(0)}, {\bm{\lambda}}_m^{(0)}\}$. Then, for any $J\!\geq\!1$, the so-called dual residual 
\begin{align*}
    \lim_{i \rightarrow +\infty} \left( {\vec{w}}^{\left( i+1 \right)} \! - \! {\vec{w}}^{\left( i \right)} \right) \! = \! 0
\end{align*}
and primal residual 
\begin{align*}
\lim_{i \rightarrow +\infty} \left( {\vec{w}}_m^{\left( i+1 \right)} \! - \! {\vec{w}}^{\left( i+1 \right)} \right) \! = \! 0, \ \forall m \! = \! 1,\ldots,M  .  
\end{align*}
\end{prop}
}

\skblack{
\begin{theorem}[Global convergence of {\FedTOPADMM} algorithm] \label{theorem:global_convergence_fedtopadmm}
Let $\left\{{\vec{w}}^{\left(i\right)}\right\}$ be a sequence generated by {\FedTOPADMM} iterative scheme in Algorithm~\ref{alg:fed_top_admm}. Then, it globally converges to a \ac{KKT} point of~\eqref{eqn_chD:three_operator__FL_optimization_problem} at any limit point.
\begin{proof}
{See Appendix \ref{sec_chD:proof_of_fedtopadmm_convergence}}.
\end{proof}
\end{theorem}
}

\else

\skblack{
\begin{theorem}[Global convergence of {\FedTOPADMM} algorithm] \label{theorem:global_convergence_fedtopadmm}
Consider a problem given in~\eqref{eqn_chD:three_operator__FL_optimization_problem} under general convex settings with at least one solution.  Let $\left\{{\vec{w}}^{\left(i\right)}\right\}$ be a sequence generated by {\FedTOPADMM} iterative scheme in Algorithm~\ref{alg:fed_top_admm}. Assume all the subproblems of \eqref{eqn_chD:top_admm_algorithm_for_FL} or Algorithm~\ref{alg:fed_top_admm} have solutions. Also, in Algorithm~\ref{alg:fed_top_admm}, for each given period $T$, assume all users should be active at least once within the users active sets $\left\{\pazocal{U}^{\left(i+1\right)}, \pazocal{U}^{\left(i+2\right)},\ldots,\pazocal{U}^{\left(i+T\right)} \right\}$ for all $i\!\geq\!0$. 
Consider suitable non-increasing step-size $\left\{\tau^{(i)} \! \in \! \Rm_{\geq 0}\right\}$ and $\left\{\zeta^{(i)} \! \in \! \Rm_{\geq 0} \right\}$ for $i\!\geq0\!$, 
and relaxation/penalty parameter $\left\{\rho_m \! \in \! \Rm_{> 0}\right\}$ with some arbitrary initial $\{{\vec{w}}^{(0)}, {\vec{w}}_m^{(0)}, {\bm{\lambda}}_m^{(0)}\}$. Assume that for any $J\!\geq\!1$, the so-called dual residual $\lim_{i \rightarrow +\infty} \left( {\vec{w}}^{\left( i+1 \right)} \! - \! {\vec{w}}^{\left( i \right)} \right) \! = \! 0$ and primal residual $\lim_{i \rightarrow +\infty} \left( {\vec{w}}_m^{\left( i+1 \right)} \! - \! {\vec{w}}^{\left( i+1 \right)} \right) \! = \! 0$, $\forall m \! = \! 1,\ldots,M$.
Then, the generated sequence $\left\{{\vec{w}}^{\left(i\right)}\right\}$ globally converges to a \ac{KKT} point of~\eqref{eqn_chD:three_operator__FL_optimization_problem} at any limit point.
\begin{proof}
{See Appendix \ref{sec_chD:proof_of_fedtopadmm_convergence}}.
\end{proof}
\end{theorem}
}

\fi



We describe the necessary processing steps of our proposed Algorithm~\ref{alg:fed_top_admm} as follows. 
In Algorithm~\ref{alg:fed_top_admm}, the total number of users participating in the \ac{FL} process is denoted by $M$. We specify the maximum number of global iterations as $I$. However, note that many heuristics-based early stopping techniques can potentially be employed on the server, \eg, when reaching the required test accuracy, but are not considered herein. 

In Step-1 of Algorithm~\ref{alg:fed_top_admm}, we provide the required inputs to the algorithm with appropriate initialization of vectors and (iterative) parameters, including the set $\pazocal{U}^{\left(0\right)}$ of selected users for the communication with the server. We denote communication events with users by $\pazocal{P} \!\coloneqq\! \left\{0, J, 2J, \ldots \right\}$,  which shows periodic events{\footnote{It is straightforward to consider aperiodic communication events $\pazocal{P}$ in the same framework, which would resemble an asynchronous setup. However, study of aperiodic communications is deferred to the future work. }} at every $J$ iterations. Therefore, $J$ represents the number of local iterations on the users' side. Consequently, the communication rounds in Algorithm~\ref{alg:fed_top_admm} is given by~\cite{zhou_and_li_fedadmm2022}
\begin{equation}
    \texttt{Communication rounds} \coloneqq \left\lfloor \nicefrac{i}{J} \right\rfloor,
\end{equation}
where $\left\lfloor \cdot \right\rfloor$ denotes flooring to the nearest integer. 


\begin{table*} 
\begin{center}
\skblack{
\caption{\skblack{Distributed optimization problem formulations}}\label{table_chD:distributed_opt_formulations}
\scalebox{1.0}{
\resizebox{\textwidth}{!}{%
	\begin{tabular}{|p{3.5cm}|p{5cm}|p{8cm}|} \hline
	\textbf{One-operator-based {FedADMM}~\cite{zhou_and_li_fedadmm2022}}      & \textbf{Two-operator ADMM}~\eqref{eqn_chD:two_operator_general_consensus_admm}  & \textbf{Three-operator-based TOP-ADMM~\eqref{eqn_chD:generalized_top_admm_algorithm_iterates__for_convergence}/\eqref{eqn_chD:top_admm_algorithm_for_FL} or {FedTOP-ADMM}~Algorithm~\ref{alg:fed_top_admm}}       \\ \hline
	$\operatorname{min}\, \sum_{m=1}^M \!f_m\!\left( {{\vec{w}}_m} \right) $ 
	& $\operatorname{min}\, \sum_{m=1}^M \!f_m\!\left( {{\vec{w}}_m} \right) \! + \! g\!\left( {\vec{w}} \right)$    
	& $\operatorname{min}\, \sum_{m=1}^M \!f_m\!\left( {{\vec{w}}_m} \right) \! + \! g\!\left( {\vec{w}} \right) \!+\! \beta h\!\left( {\vec{w}} \right)$  
	\\ 
	$\operatorname{s.t.}\, {{\vec{w}}_m} \!=\! {\vec{w}} \ \forall m $   
	& $\operatorname{s.t.}\, {{\vec{w}}_m} \!=\! {\vec{w}} \ \forall m $   
	& $\operatorname{s.t.}\, {{\vec{w}}_m} \!=\! {\vec{w}} \ \forall m $ 
	\\ \hline
	\end{tabular}
	}
} }
\end{center}
\end{table*}

The iterative {\FedTOPADMM} algorithm starts at Step-2. Note that {\FedTOPADMM} stops when the total number of global iterations $I$ is exhausted, or one employs an early stopping criteria within the loop, \eg, using test accuracy criteria as the server is expected to have some test dataset. 

If {\FedTOPADMMI} is employed, then at Step-3 and Step-4 for any global iteration $i$, the server performs some intermediate processing reminiscent of gradient descent-like step. Conversely, if \linebreak{\FedTOPADMMII} is employed, then when there is no communication event between the server and any users for a given iteration $i$, \ie, $i \!\notin\! \pazocal{P}$, the server performs the same processing in Step-3 and Step-4 as {\FedTOPADMMI}.  This intermediate step processing at server, \ie, Step-4, is part of the global weight vector $\vec{w}$ update---cf.~\eqref{eqn_chD:fed_top_admm__global_weight_update} and Step-4 of Algorithm~\ref{alg:fed_top_admm}. More specifically, assuming $g$ is closed, convex, and proper function (possibly nonsmooth), the solution corresponding to the subproblem~\eqref{eqn_chD:step1__FedTOPADMM_algorithm}, \ie, the $\vec{w}$ update, reads 
\begingroup\setlength{\jot}{-1.0ex} 
\begin{align} \label{eqn_chD:fed_top_admm__global_weight_update}
&\hspace{-2.5mm}\vec{w}^{\left( i+1\right)} \nonumber \\ 
&\hspace{-2mm}= \! {\prox}_{\nu^{\left( i\right)} g}\!\left( \!\frac{1}{\nu^{\left( i\right)}} \! \left[\!\sum_{m=1}^M \!\vec{u}_m^{\left( i\right)} \!-\! \tau^{\left( i\right)} \nabla \!h\!\left(\! \vec{w}^{\left( i\right)} \!\right) \!+\! \zeta^{(i)} \vec{w}^{\left( i\right)} \!\right] \! \right)\!,
\end{align}
\endgroup
where the definition of proximal operator is given in Definition~\ref{definition_chD:prox_operator} of Appendix~\ref{appendix:use_lemmas_def}, \linebreak$\nu^{(i)} \!\coloneqq\! {1}/{\left( \sum_{m \in [M] } \rho_m \!+\! \zeta^{(i)} \right)}$, and $\vec{u}_m^{\left( i+1\right)} \!\coloneqq\! \rho_m \vec{w}_m^{\left( i+1\right)} \!+\! \bm{\lambda}_m^{\left( i+1\right)}$. In the {\FedTOPADMM} algorithm, instead of directly using $\left\{\vec{u}_m^{\left( i\right)}\right\}$  in~\eqref{eqn_chD:fed_top_admm__global_weight_update}, the server utilizes $\vec{v}_m^{\left( i\right)}$, which is updated as described from Step-7 to Step-12 of Algorithm~\ref{alg:fed_top_admm}. When communication events $i\!\in\!\pazocal{P}$ occur, the server receives parameter vector $\vec{u}_m^{\left( i\right)}$ from selected user $m \!\in\! \pazocal{U}^{\left(i\right)}$ and consequently updates $\vec{v}_m^{\left( i\right)}$. Note the difference between the update in~\eqref{eqn_chD:fed_top_admm__global_weight_update} with the Step-13 of Algorithm~\ref{alg:fed_top_admm}, which aggregates all the weights of users and server appropriately. Consequently, the global updated weight vector $\vec{w}^{\left( i+1\right)}$ is generated at~\mbox{Step-13}.

The second subproblem of \ac{TOP-ADMM}~\eqref{eqn_chD:step2__FedTOPADMM_algorithm}, corresponding to the primal update of weight vector $\vec{w}_m^{\left(i+1\right)}$ for each user $m$, is equivalent to the subproblem in the classical \ac{ADMM} or {\FedADMM}~\cite{zhou_and_li_fedadmm2022, zhou_iceadmm_draft_2022}. We use the inexact solution to this subproblem, \ie, a linear approximation of the function $f_m$, proposed in~\cite{zhou_and_li_fedadmm2022, zhou_iceadmm_draft_2022}, in which the recipe is given in Step-19 of Algorithm~\ref{alg:fed_top_admm}. We refer the interested readers to~\cite{zhou_and_li_fedadmm2022, zhou_iceadmm_draft_2022} for the detailed analysis of the inexact solution. \mbox{Step-19} is repeated $J$-times before the server receives the updated parameter $\vec{u}_m^{\left(i+1\right)}$ from the selected user $m \!\in\! \pazocal{U}^{\left(i\right)}$. In Step-23, if the user $m$ was not selected for the communication with the server, \ie, $m \!\notin\! \pazocal{U}^{\left(i\right)}$, the server essentially utilizes the previously received parameter from the nonselected user $m$.

\skblack{
\subsection{Comparison Among {ADMM}, {FedADMM}, {TOP-ADMM}, and {FedTOP-ADMM} }\label{sec_chD:cmp_admm_fedadmm_topadmm_fedtopadmm}
To this end, we would like to accentuate the prowess of three-operator algorithms, such as our proposed \ac{TOP-ADMM}/{\FedTOPADMM}, compared to existing two-operator \ac{ADMM}~\eqref{eqn_chD:classical_consensus_admm_algorithm} or one-operator {\FedADMM}~\cite{zhou_and_li_fedadmm2022}. There are many problems of interest where one-operator such as existing {\FedADMM}~\cite{zhou_and_li_fedadmm2022} is not capable to tackle the problem having non-differentiable or non-smooth regularizer---see, \eg, sparse logistic regression problem~\eqref{eqn_chD:obj__distributed_l1_logistic_regression} with non-differentiable $\ell_1$ norm. More concretely, see Table~\ref{table_chD:distributed_opt_formulations} that compares which algorithm can solve what form of distributed optimization problems. 
Recall, {\FedTOPADMM} is built on the extended version of (three-operator) \ac{TOP-ADMM}~\eqref{eqn_chD:top_admm_algorithm_for_FL}, where the extended \ac{TOP-ADMM}~\eqref{eqn_chD:top_admm_algorithm_for_FL} include variable ``proximity" convex regularizers to the primal $\vec{w}$-update and Glowinski's relaxation parameter to the dual $\bm{\lambda}$-update in contrast to the vanilla \ac{TOP-ADMM}~\eqref{eqn_chD:generalized_top_admm_algorithm_iterates__for_convergence}. Moreover, {\FedTOPADMM} generalizes (one-operator) {\FedADMM} and essentially two-operator \ac{ADMM}. However, this existing one-operator {\FedADMM} utilizes (two-operator) \ac{ADMM}. Nonetheless, clearly, the {\FedTOPADMM} in Algorithm~\ref{alg:fed_top_admm} becomes {\FedADMM} when $\tau^{\left(i\right)}\!=\!0$ and $\zeta^{\left(i\right)}\!=\!0$ for all $i\!=\!0,1,\ldots,I-1$ iterations. 
}


\subsection{Connections with Existing Works on FL}

The other benchmarking algorithms besides {\FedADMM}~\cite{zhou_and_li_fedadmm2022} are {\FedProx}~\cite{Tian_Li_etal__FedProx__2020} and \linebreak{\FedAvg}~\cite{McMahan__FedAvg__2017}, which can easily coexist with the Algorithm~\ref{alg:fed_top_admm} framework as highlighted in~\cite{zhou_and_li_fedadmm2022}. Specifically, in Step-19 of Algorithm~\ref{alg:fed_top_admm}, the {\FedTOPADMM} becomes {\FedProx} or {\FedAvg} by setting the following \ac{TOP-ADMM} parameters to zero, \ie, $\tau^{\left(i \right)}\!=\!0$, $\zeta^{\left(i \right)}\!=\!0$, and $\gamma\!=\!0$  and replacing the weight update for each user with  
\begin{align*}
\vec{u}_m^{\left(i+1\right)} \!\coloneqq\! \vec{w}_m^{\left(i+1\right)} \!\approx\! \vec{w}_m^{\left(i\right)} \!-\! \eta \! \left[\! \nabla f_m\!\left( \vec{w}_m^{\left(i\right)} \right) \!+\! \mu \!\left( \vec{w}_m^{\left(i\right)}  \!-\! \vec{v}^{\left(i+1\right)}  \right)\!   \right],     
\end{align*}
where $\eta$ is a step size and $\mu$ is a scaling parameter for the proximal term, in which $ \mu \!=\!0$ for {\FedAvg}. Moreover, note that {\FedADMM\texttt{-VC}} represents {\FedADMM} with a virtual client collocated at the edge server. 

\section{Numerical Results} \label{sec_chD:simulation_results}
\skblack{In this section, we conduct the experiments using the distributed (sparse) logistic regression to benchmark the performance of these existing algorithms against our proposed {\FedTOPADMM} algorithms. 
\begin{subequations}\label{eqn_chD:distributed_l1_logistic_regression}
\begin{align}
\label{eqn_chD:obj__distributed_l1_logistic_regression}
\underset{ \left\{{\vec{w}}_m \in \Rm^n \right\},  {\vec{w}} \in \Rm^n}{\text{minimize}}&   \ \sum_{m=1}^M \frac{1}{d_m} \sum_{j=1}^{d_m} \Biggl[ \! \log\!\left(\!1 \!+\!\exp\!\left\{ \left(\mat{A}_m\!\left[:,j\right]\right)^\trans \vec{w}_m \right\} \!\right) \nonumber \\
    &\hspace{7mm}\!-\! \vec{t}_m\left[j\right] \left\{ \left(\mat{A}_m\left[:,j\right]\right)^\trans \vec{w}_m \right\} \!+\!\frac{\kappa}{2}\left\|\vec{w}_m\right\|_2^2 \Biggr] \\ 
    \label{eqn_chD:l1_regularizer__distributed_l1_logistic_regression}
    &\hspace{7mm}\!+\! \upsilon \left\| \vec{w} \right\|_{1} \\
    \label{eqn_chD:consensus_constraint__distributed_l1_logistic_regression}
\text{subject to}\hspace{+3mm}& 
  \ \  {{\vec{w}}_m} - {\vec{w}} = {\vec{0}}, \quad \forall m=1,\ldots,M,
\end{align}
\end{subequations}
where the training dataset on the user $m$ is $\pazocal{D}_m\!\coloneqq\!\left\{ \vec{t}_m, \mat{A}_m \right\}_{m=1}^M$, \ie, the binary output is $\vec{t}_m \!\in\! \left\{ 0\!,+1\right\}^{d_m}\!$, the input feature matrix is $\mat{A}_m \!\in\! \Rm^{n \times d_m}$, and the regression weight vector is $\vec{w} \!\equiv\!\vec{w}_m \!\in\!\Rm^{n}$. 
The scaling factor to the regularizer is $\kappa \!=\!0.001$ in the experiments, unless otherwise mentioned. In case of non-sparse logistic regression problem, one can ignore the $\ell_1$-norm regularizer~\eqref{eqn_chD:l1_regularizer__distributed_l1_logistic_regression} by setting zero to $\upsilon \!\in\! \Rm_{\leq0}$. Notice that the non-sparse problem is also used in~\cite{zhou_and_li_fedadmm2022}.
}

\skblack{
For all the considered methods including our proposed method, the loss or objective function at each user $m$ or at the server for {\FedTOPADMM} in~\eqref{eqn_chD:distributed_l1_logistic_regression} reads 
\begin{align}\label{eqn_chD:objective_log_reg}
    h\left( \vec{w} \right) 
    \!\equiv\!& f_m\!\left( \vec{w}_m \right) \coloneqq\! \frac{1}{d_m} \sum_{j=1}^{d_m} \Biggl[ \! \log\!\left(\!1 \!+\!\exp\left\{ \left(\mat{A}_m\!\left[:,j\right]\right)^\trans \!\vec{w}_m \!\right\} \!\right) \nonumber \\
    &\hspace{7mm}\!-\! \vec{t}_m\left[j\right] \left\{ \left(\mat{A}_m\left[:,j\right]\right)^\trans \vec{w}_m \right\} \!+\!\frac{\kappa}{2}\left\|\vec{w}_m\right\|_2^2 \Biggr].
\end{align}
}

\subsection{Experimental Settings}

In this section, we present numerical results to illustrate the performance of our proposed {\FedTOPADMMI} and {\FedTOPADMMII} algorithms compared to the state-of-the-art algorithms {\FedADMM}~\cite{zhou_and_li_fedadmm2022}, {\FedProx}~\cite{Tian_Li_etal__FedProx__2020}, and {\FedAvg}~\cite{McMahan__FedAvg__2017}. 

\skblack{
Before we proceed further with more realistic dataset and comparison of our proposed {\FedTOPADMM} with the abovementioned benchmarking methods, we want to highlight the strength of our proposed {\FedTOPADMM} in contrast to these existing algorithms. In other words, we show the potency of three-operator structure of {\FedTOPADMM} in contrast to these existing one-operator, \ie, corresponding to sum of separable $f_m$ functions, {\FedAvg}, {\FedProx}, {\FedADMM}.  More concretely, we show how easily our proposed three-operator {\FedTOPADMM} can exploit the non-differentiable regularizer~\eqref{eqn_chD:l1_regularizer__distributed_l1_logistic_regression} that can easily be handled by setting, say, function $g$ as a scaled $\ell_1$-norm, whose proximal operator is well known in the literature, see, \eg, \cite{Beck2017, Ryu_Yin_ls_book_2022_draft}. Although it is unclear how to employ $\ell_1$-norm regularizer with existing one-operator methods {\FedAvg} and {\FedProx} because of non-differentiablity of $\ell_1$-norm, we can only modify {\FedADMM} to incorporate the second operator corresponding to function $g$ since {\FedADMM} is built on the classical consensus two-operator \ac{ADMM}~\eqref{eqn_chD:two_operator_general_consensus_admm}. Thus, in the {\it modified} {\FedADMM}, after the global aggregation one can include the second operator, \ie, proximal operator corresponding to the function $g$, without the third operator corresponding to the gradient of $h$. More precisely, the Step-13 of Algorithm~\ref{alg:fed_top_admm} for the modified {\FedADMM} without the third operator (corresponding to the gradient of function $h$) can be expressed as $\vec{w}^{\left( i+1\right)} \! = \! {\prox}_{\nu^{\left( i\right)} g}\!\left( \!\frac{1}{\nu^{\left( i\right)}} \! \left[\sum_{m=1}^M \!\vec{v}_m^{\left( i\right)} \right] \! \right)$ with $\gamma\!=\!1$. The numerical convergence behaviour of {\FedTOPADMM} and modified {\FedADMM}, in terms of objective, primal residual, and dual residual against iterations are depicted in Figure~\ref{fig_chD:fig1__convergence_behaviour_of_top_admm__distributed_logistic_regression}. In this simple but illustrative example, the random test setup is similar to \cite[Section~11.2]{Boyd2011}. Additionally, in this toy example, we have generated the synthetic sparse training dataset with a total of 20000 examples having a feature vector length $n\!=\!100$. We have employed $M\!=\!100$ users that are distributed, and each user has $d_M\!=\!200$ training examples, where all the users are active and $J=1$. Clearly, {\FedTOPADMM} shows faster convergence than {\FedADMM}, while delivering similar/better training error ($0.82\%$) than {that of} {\FedADMM} ($0.85\%$). 
}

\begin{figure*}[htp!]
  \centering
  \begin{minipage}{.525\linewidth}
    \centering
    \subcaptionbox{ Objective Error-vs.-Iterations. \label{fig_chD:fig1a__logistic_reg__objective_vs_iter} }
    {\includegraphics[width=\textwidth,trim=32mm 92mm 31mm 92mm,clip]{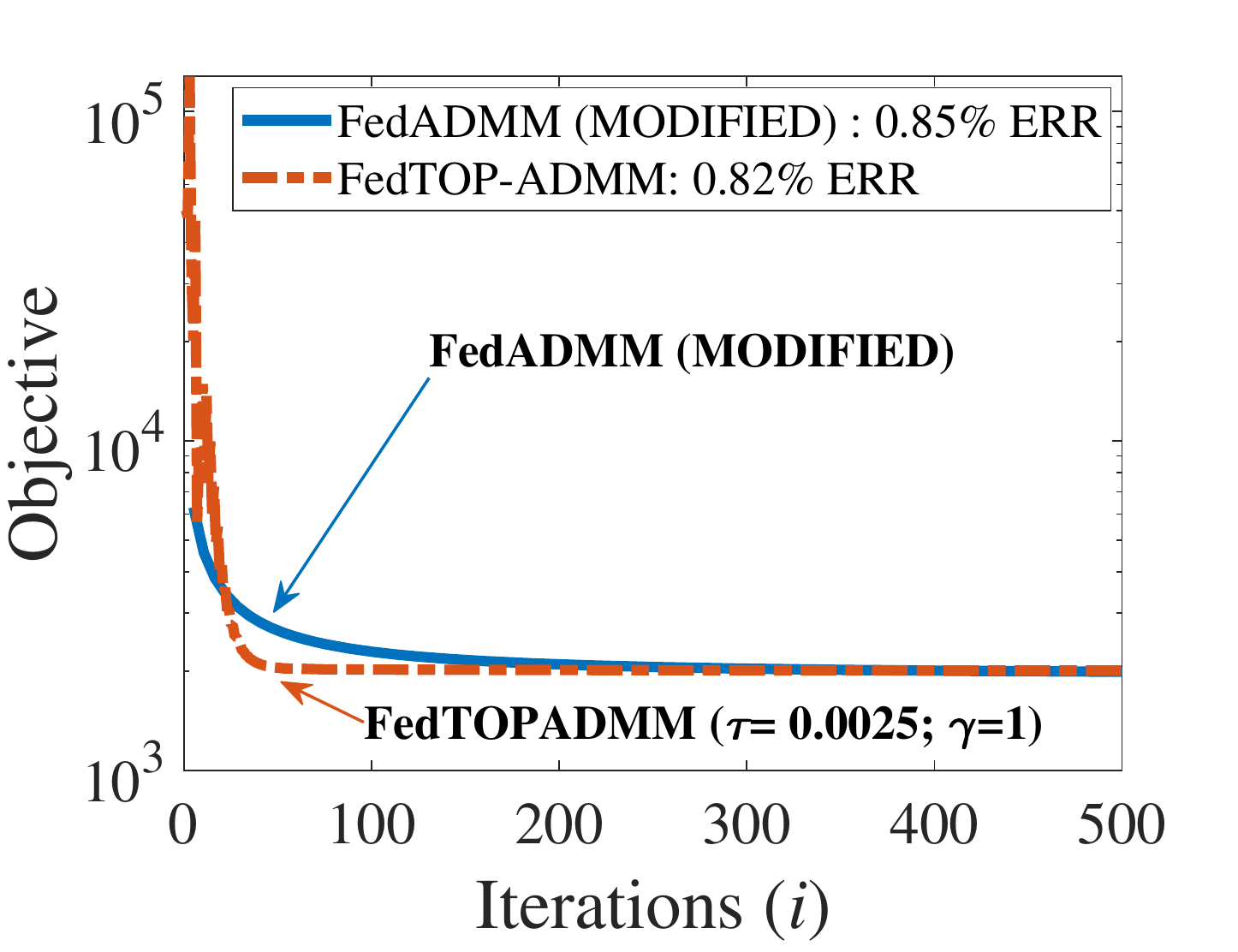}}
  \end{minipage}\quad
  \begin{minipage}{.525\linewidth}
    \centering
    \subcaptionbox{ Primal Residual-vs.-Iterations. \label{fig_chD:fig1c__logistic_reg__primal_residual_vs_iter}}
    {\includegraphics[width=\textwidth,trim=32mm 92mm 31mm 92mm,clip]{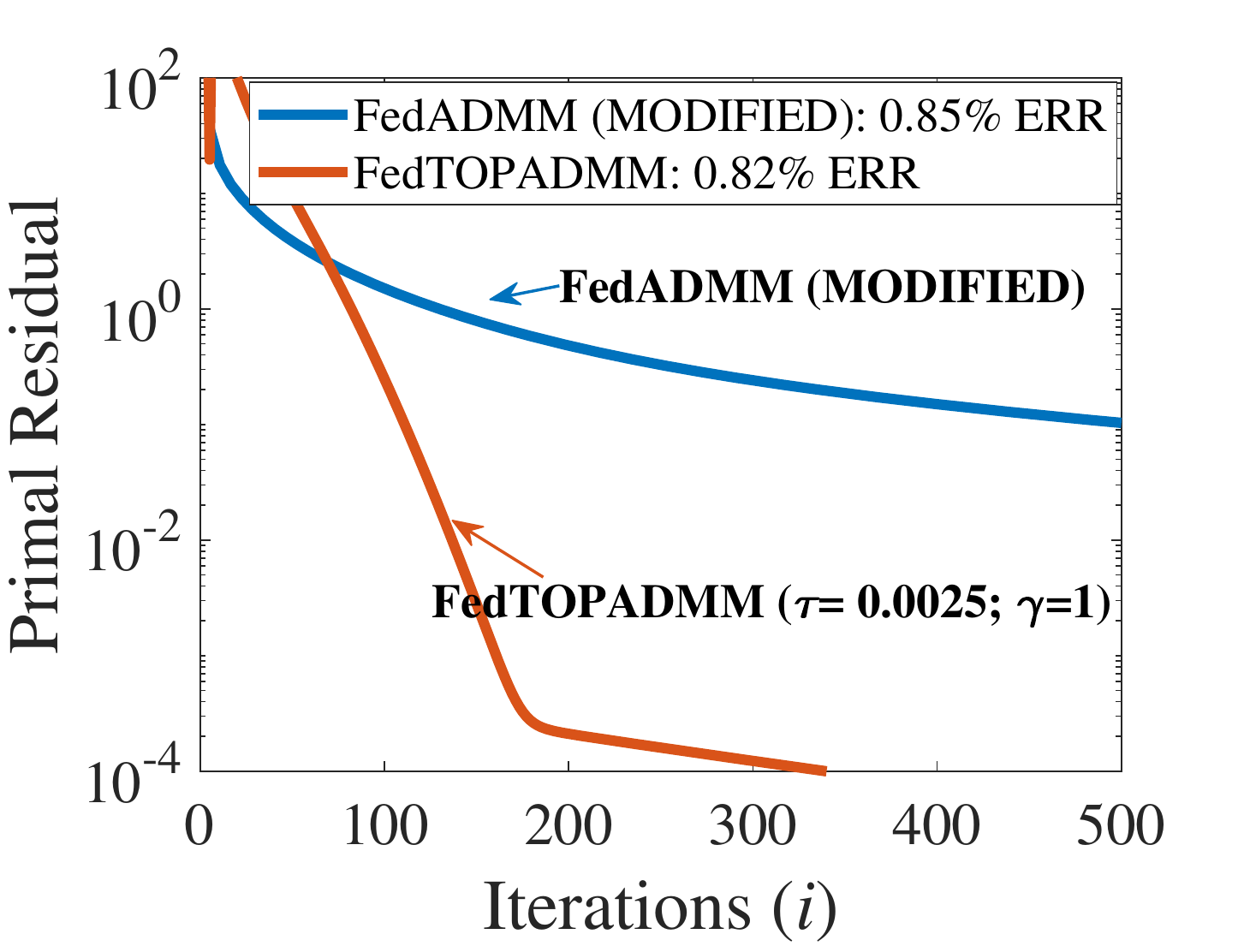}}
  \end{minipage}\quad
  \begin{minipage}{.525\linewidth}
    \centering
    \subcaptionbox{ Dual Residual-vs.-Iterations. \label{fig_chD:fig1b__logistic_reg__dual_residual_vs_iter}}
    {\includegraphics[width=\textwidth,trim=32mm 92mm 31mm 92mm,clip]{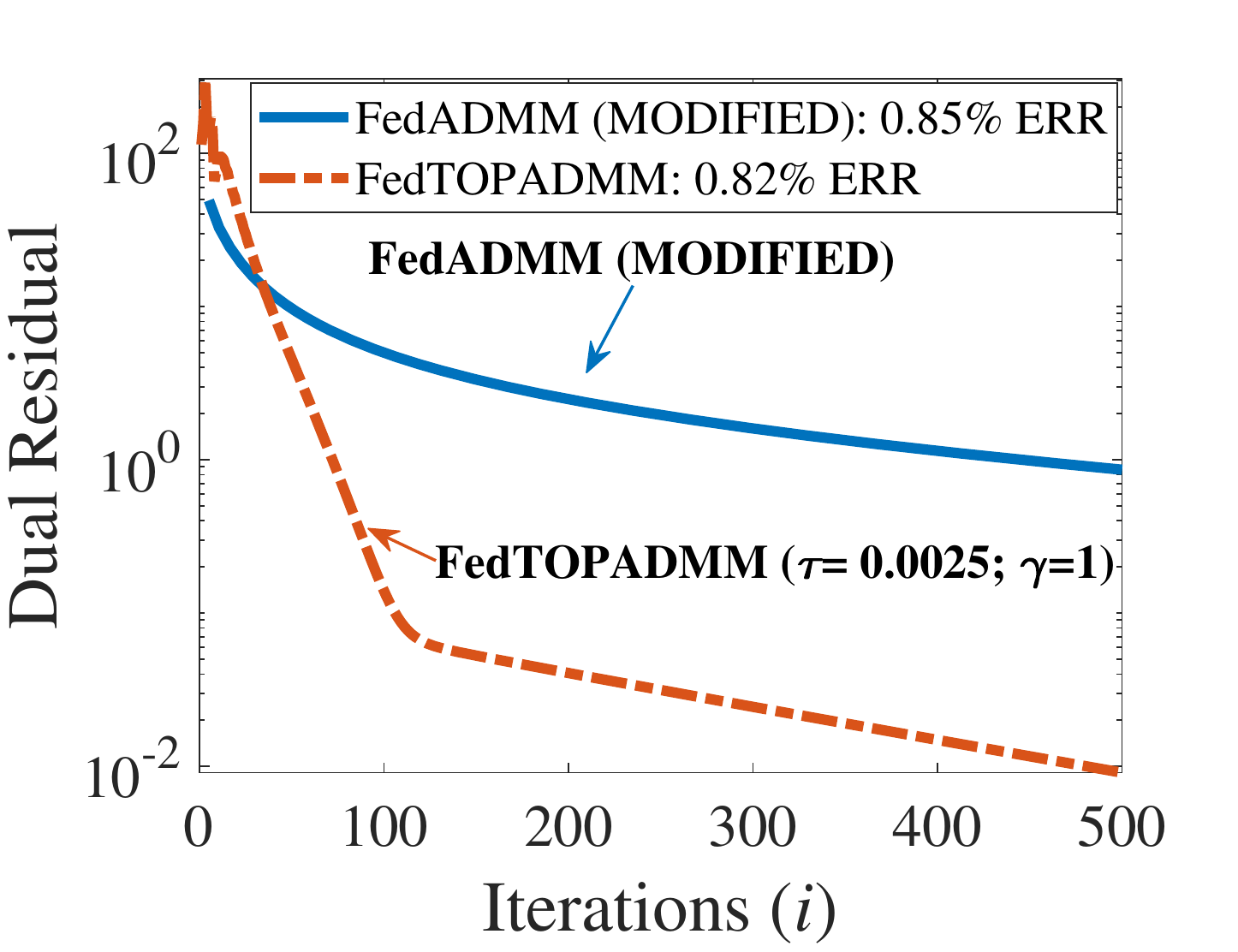}}
  \end{minipage}
  \caption{\skblack{Comparison of convergence behaviour between  {\FedADMM} and {\FedTOPADMM} for the distributed {\it sparse} logistic regression problem~\eqref{eqn_chD:distributed_l1_logistic_regression}.} }
  \label{fig_chD:fig1__convergence_behaviour_of_top_admm__distributed_logistic_regression}  
\end{figure*}

\skblack{In our next set of experiments, we have ignored the sparse parameter in the logistic regression such that we can compare {\FedTOPADMM} with not only (modified) {\FedADMM} but also {\FedAvg} and {\FedProx}. Additionally,} in all the \skblack{subsequent} considered simulations, the total number of users is fixed to 200, \ie, $M\!=\!200$. However, 10 users are selected uniformly at random during each communication event of the global iteration, \ie, $i \!\in\!\pazocal{P}$. 

We have conducted experiments using one of the most popular real-world datasets, namely MNIST~\cite{Lecun__MNIST_data__1998}. Specifically, we have scaled/normalized the MNIST input data. There are many ways to scale the input data matrix $\mat{A} \!\in\! \Rm^{n \times d}$, where $d \!=\! \sum_{m=1}^M d_m$ and $n$ represent the total number of data samples for training and testing, and feature vector length, respectively. Moreover, we have analyzed two scaling approaches: 
1) 
$\vec{a}\!\in\! \Rm^{n \times 1},\; \vec{a} \!\coloneqq  \texttt{mean}\left( \mat{A}, 2 \right) \!\oslash\! \texttt{std}\left( \mat{A}, [\,], 2 \right)$%
, and 2) 
${\vec{a} \!\coloneqq }\texttt{mean}\left( \mat{A}, 2 \right) \!\oslash\! \texttt{var}\left( \mat{A}, [\,], 2 \right)\!\in \! \Rm^{n \times 1}$, 
where $\oslash$ corresponds to elementwise division{\footnote{Because of numerical division, if any element is not a number, we replace the element by zero.}}, and $\texttt{mean}$, $\texttt{std}$, and $\texttt{var}$ represent the mean, standard deviation, and variance along the second dimension of matrix $\mat{A}$ as used in MATLAB expressions. Finally, the scaled version of the MNIST dataset is expressed as $\mat{A} \! \leftarrow \! \mat{A} \!-\! \bm{1}_{1 \times d} \!\otimes\! \vec{a}$, where $\bm{1}_{1 \times d}$ denotes a row vector of all ones with dimension $1 \times d$, and $\otimes$ represents a Kronecker product. \skblack{Additionally, we use other popular datasets, such as CIFAR-10 and CIFAR-100~\cite{Alex_Cifar10_100_dataset_2009}.}

\skblack{We evaluate the performance using two data partitioning: 1) \ac{i.i.d.}, where the data is randomly shuffled and the corresponding labels are shuffled accordingly, and 2) non-\ac{i.i.d.}, where the training labels are sorted in ascending order and the corresponding input data is ordered accordingly. Thus, in case of MNIST dataset, this non-\ac{i.i.d.} data split is one of the pathological cases because each user or base station would have at most two class labels.}

In Figure~\ref{fig_chD:examples__mnist_data}, we illustrate the examples of MNIST digits without, see Figure~\ref{fig_chD:MNIST_dataset__without_normalization} and with these two aforementioned scaling approaches, see Figs.\ref{fig_chD:MNIST_dataset__with_normalization__approach2__with_std}-\ref{fig_chD:MNIST_dataset__with_normalization__mistakenApproach__with_var}. Based on our exhaustive experiments, we have found the second approach in Figure~\ref{fig_chD:MNIST_dataset__with_normalization__mistakenApproach__with_var} more challenging to learn than the first scaling approach in Figure~\ref{fig_chD:MNIST_dataset__with_normalization__approach2__with_std} and the unscaled original version in Figure~\ref{fig_chD:MNIST_dataset__without_normalization}. Consequently, we have employed the second scaling approach for our further numerical analysis. The data distribution among the users and the server is \ac{i.i.d.} unless otherwise stated. Moreover, we have considered a binary classifier by simply employing digit~1 as the true label and other digits as false labels.
\begin{figure*}[!t]
\centering
  \begin{minipage}[!t]{.525\linewidth}
  \begin{subfigure}[!t]{\textwidth}
    \centering
    \includegraphics[width=\textwidth,trim=45mm 98mm 40mm 92mm,clip]{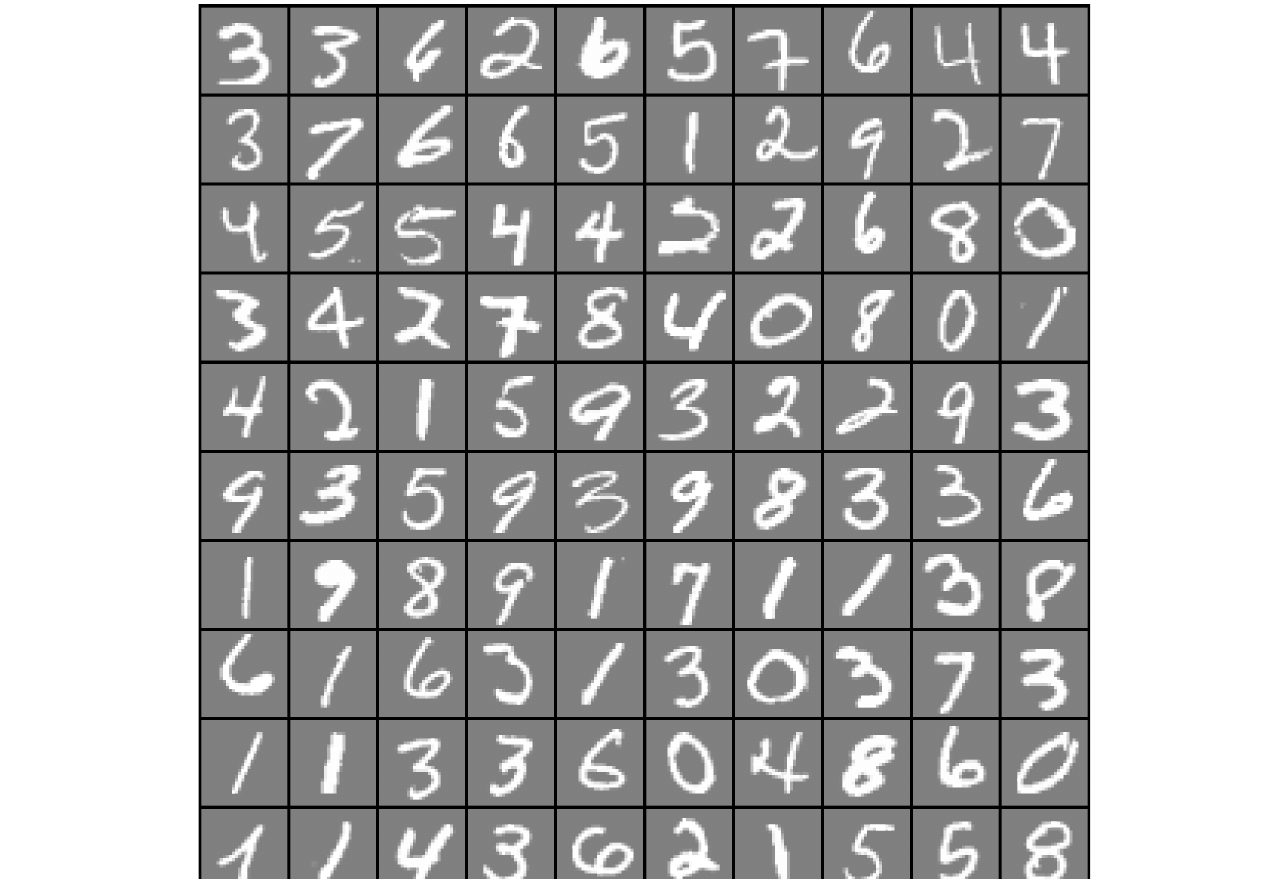}
    \caption{Without any scaling}\label{fig_chD:MNIST_dataset__without_normalization}
  \end{subfigure}
  \end{minipage}\hfill
  \begin{minipage}[!t]{.525\linewidth}
  \begin{subfigure}[!t]{\textwidth}
    \centering
    \includegraphics[width=\textwidth,trim=45mm 98mm 40mm 92mm,clip]{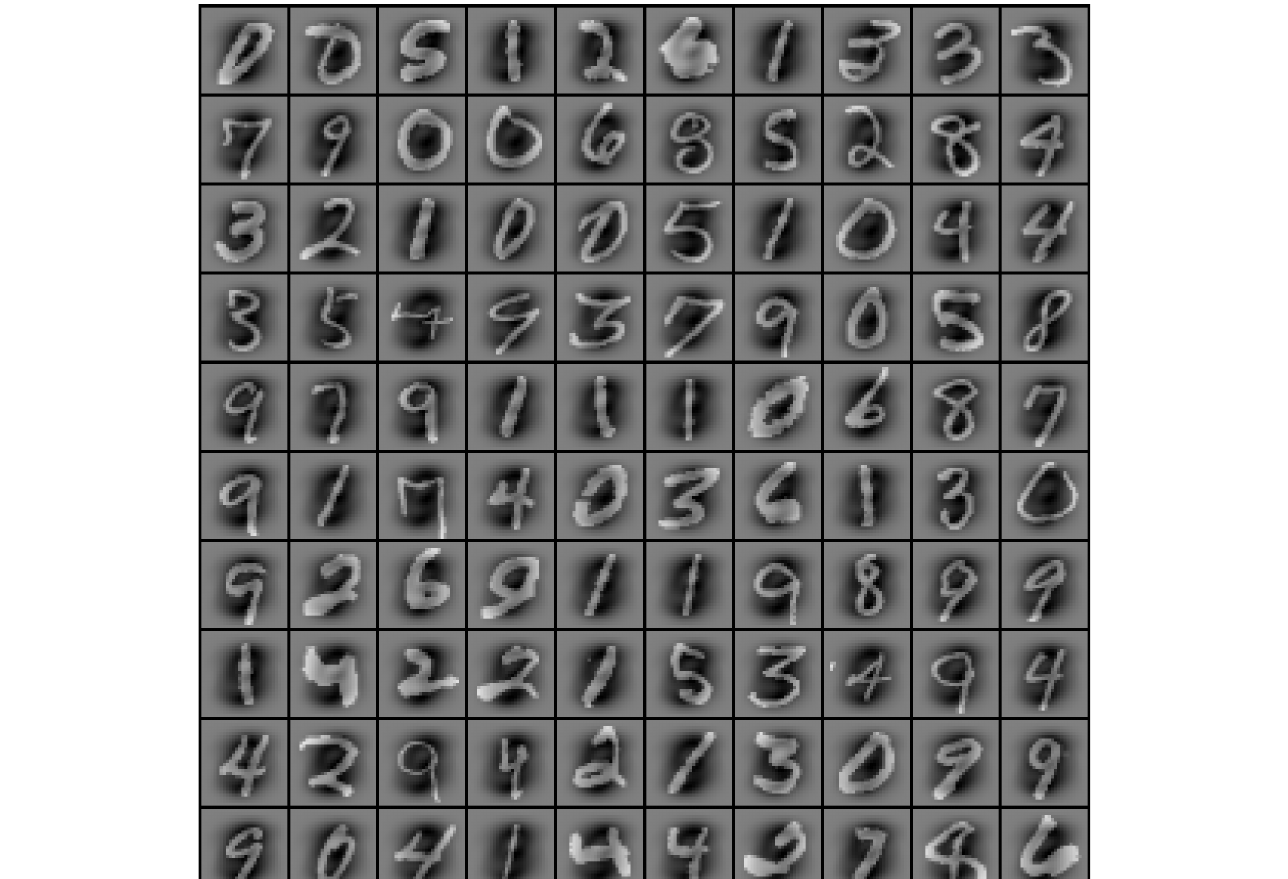}
    \caption{With a scaling approach~1}\label{fig_chD:MNIST_dataset__with_normalization__approach2__with_std}
  \end{subfigure}
  \end{minipage}\hfill
  \begin{minipage}[!t]{.525\linewidth}
  \begin{subfigure}[!t]{\textwidth}
  \centering
    \includegraphics[width=\textwidth,trim=45mm 98mm 40mm 92mm,clip]{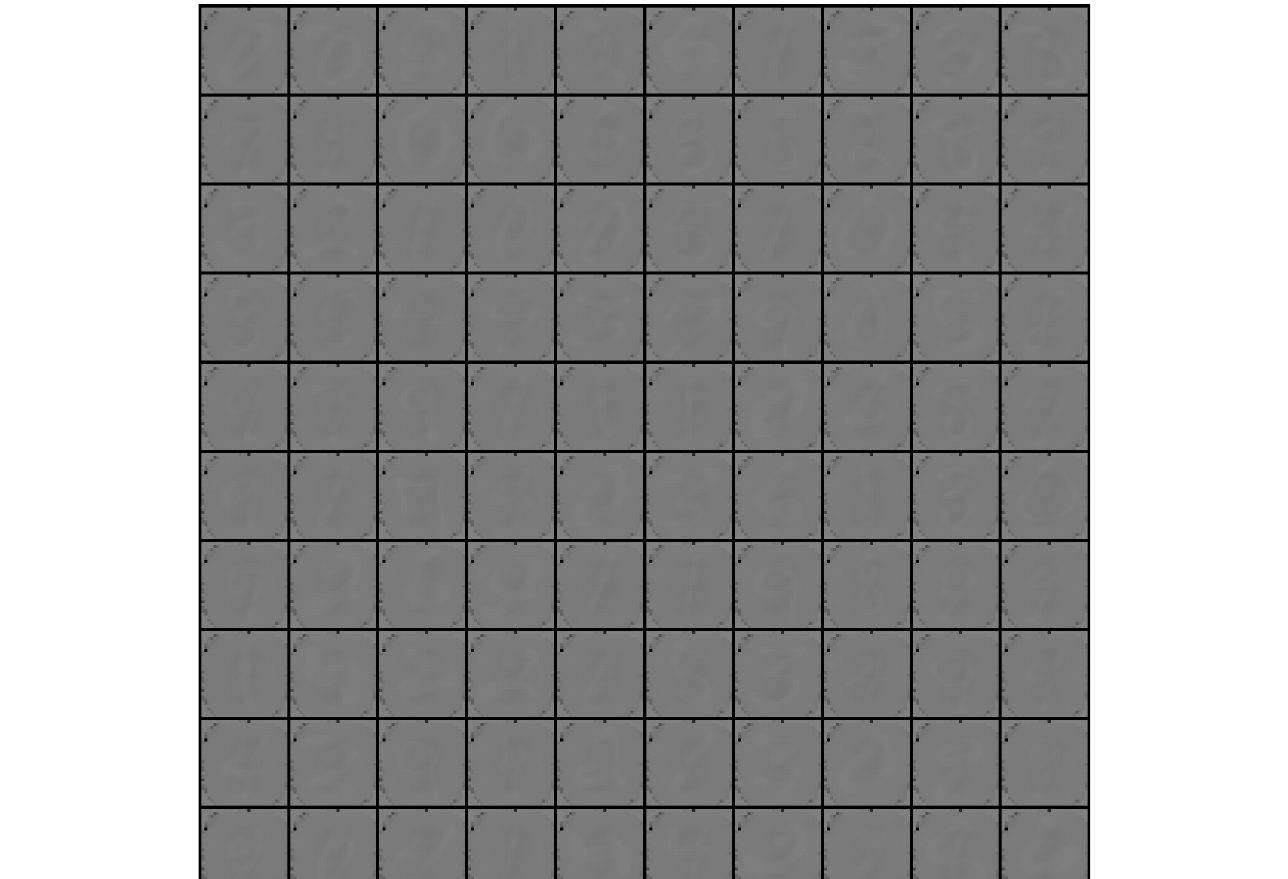}
    \caption{With a proposed scaling approach~2}\label{fig_chD:MNIST_dataset__with_normalization__mistakenApproach__with_var}
    \end{subfigure}
  \end{minipage}
  \caption{Examples of MNIST handwritten digits without scaling and with two different scaling approaches. }\label{fig_chD:examples__mnist_data}
\end{figure*}

We perform grid search to tune the hyperparameters of the benchmarking algorithms including our proposed algorithms. The step size of the gradient descent in {\FedAvg} is chosen from the candidate set, \ie, $\eta\!\in\!\left\{
10^{-1},10^{-2},10^{-3},10^{-4},10^{-5},5\times10^{-6},10^{-6}
\right\}$. For {\FedProx}, in addition to the step size $\eta$,
we need to tune the hyperparameter $\mu$ for the proximal term, which is appropriately selected from the set, \ie, $\mu\!\in\!\left\{0.01, 0.1, 0.5, 500 \right\}$. The hyperparameters of {\FedADMM} are chosen as suggested in~\cite{zhou_and_li_fedadmm2022, zhou_iceadmm_draft_2022}. More specifically, the inexact version of {\FedADMM} sets $\mat{Q}_m\!=\!r_m \I$ with $r_m \!=\! \operatorname{eig}_{\rm max} \!\left( \mat{A}_m^\trans \mat{A}_m \right)/\left(4\!+\!\kappa\right)$ that utilizes the maximum eigenvalue of $ \mat{A}_m^\trans \mat{A}_m$ Gram matrix of input data, such that the hyperparameter  $\rho_m\left(a\right) \!=\! {\left[a \log\left(M d_m \right) \alpha_m r_m\right]}/{\left( \log\left(2\!+\!J\right) \right)} $ with $\alpha_m \!=\! 1/\left(M d_m\right)$ and a user defined parameter $a$. We have considered the following mean values of the hyperparameters $\texttt{mean}\left(\left\{ \rho_m \right\}\right)\!\in\!\bigl\{\mbox{3.4035e-1}, \ \mbox{3.4035e-2}, \ \allowbreak \mbox{3.4035e-3},\ \mbox{3.4035e-4},\ \mbox{3.4035e-5},\  \mbox{3.4035e-6},\ \mbox{3.4035e-7} \bigr\}$ using appropriate values of $a$ and the maximum eigenvalues of the input data matrices. 
Subsequently, the tunable parameters in {\FedTOPADMM} are chosen in common with {\FedADMM}.
Furthermore, the extra tunable parameters of {\FedTOPADMM}, particularly, the parameters on the server side $\tau^{(i)}$ and $\zeta^{(i)}$, are selected to be monotonically decreasing with increasing iterations. Specifically, we have employed the following recipe for both $\tau^{(i)}$ and $\zeta^{(i)}$ to decrease monotonically over iterations: $\beta^{\left(i+1\right)} \!=\! \beta^{(0)}/\left(1 \!+\! \left[i \mu^\prime \beta^{\left(i\right)}\right]\right)$, where $\beta \!\coloneqq\! \tau$ and $\beta \!\coloneqq\! \zeta$, appropriately, and $\mu^\prime\!=\!10$. The candidate set of additional tunable parameters of {\FedTOPADMM} are $\tau^{\left(0\right)}\!\in\!\bigl\{
10^{-1},10^{-2},10^{-3},\ldots, 10^{-9}\bigr\}$, $\zeta^{\left(0\right)}\!\in\!\left\{5, 2.5, 1, 0.5, 0.025, 0.005, 0.00025\right\}$, and $\gamma\!\in\!\left\{0.1, 0.5, 1, 1.5, 1.999 \right\}$.

\subsection{Experimental Results}

\begin{figure*}[tp!]
  \begin{minipage}[t]{1\linewidth}
  \begin{subfigure}[t]{0.495\textwidth}
    \centering
    \includegraphics[width=\textwidth,trim=32.5mm 85mm 31mm 92.75mm,clip]{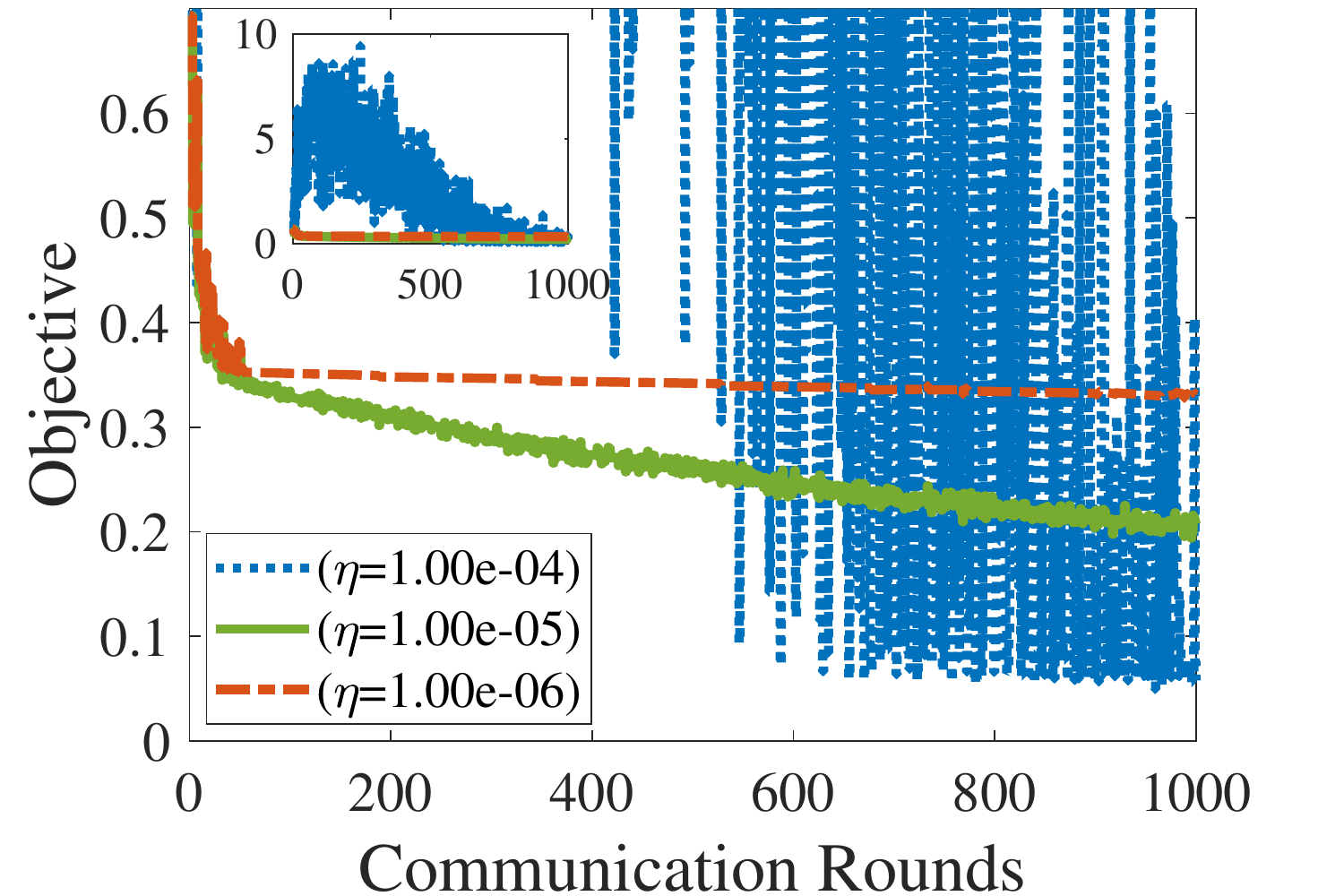}
    \caption{{\FedAvg}~\cite{McMahan__FedAvg__2017}}\label{fig_chD:FINAL__mnist_periodic__J10_IID0__FedAvg_for_various_set_of_pars__obj_vs_CR}
  \end{subfigure}
  \begin{subfigure}[t]{0.495\textwidth}
    \centering
    \includegraphics[width=\textwidth,trim=32.5mm 85mm 31mm 92.75mm,clip]{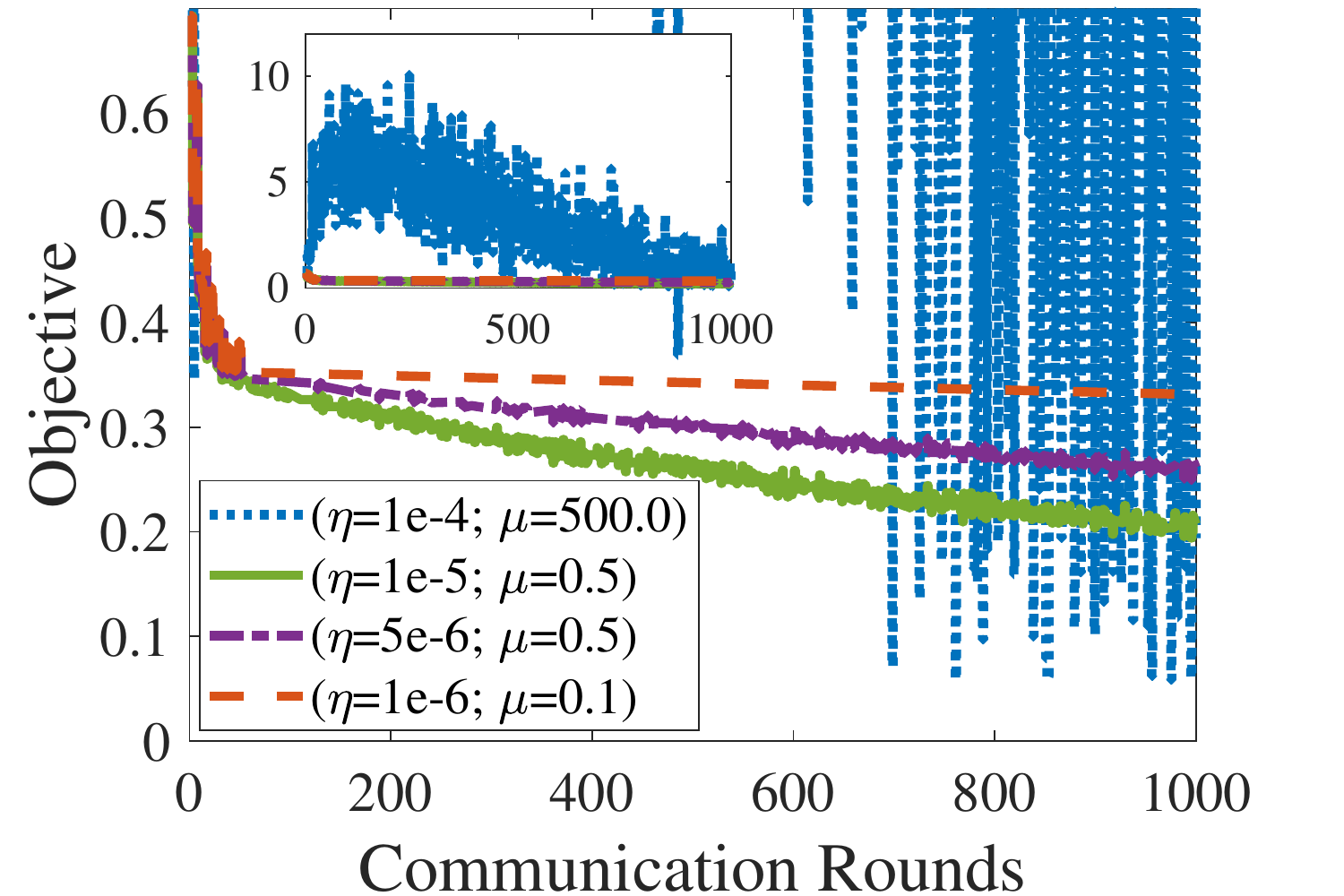}
    \caption{{\FedProx}~\cite{Tian_Li_etal__FedProx__2020}}\label{fig_chD:FINAL__mnist_periodic__J10_IID0__FedProx_for_various_set_of_pars__obj_vs_CR}
  \end{subfigure}
  
    \begin{subfigure}[t]{0.495\textwidth}
    \centering
    \includegraphics[width=\textwidth,trim=32.5mm 85mm 31mm 92.75mm,clip]{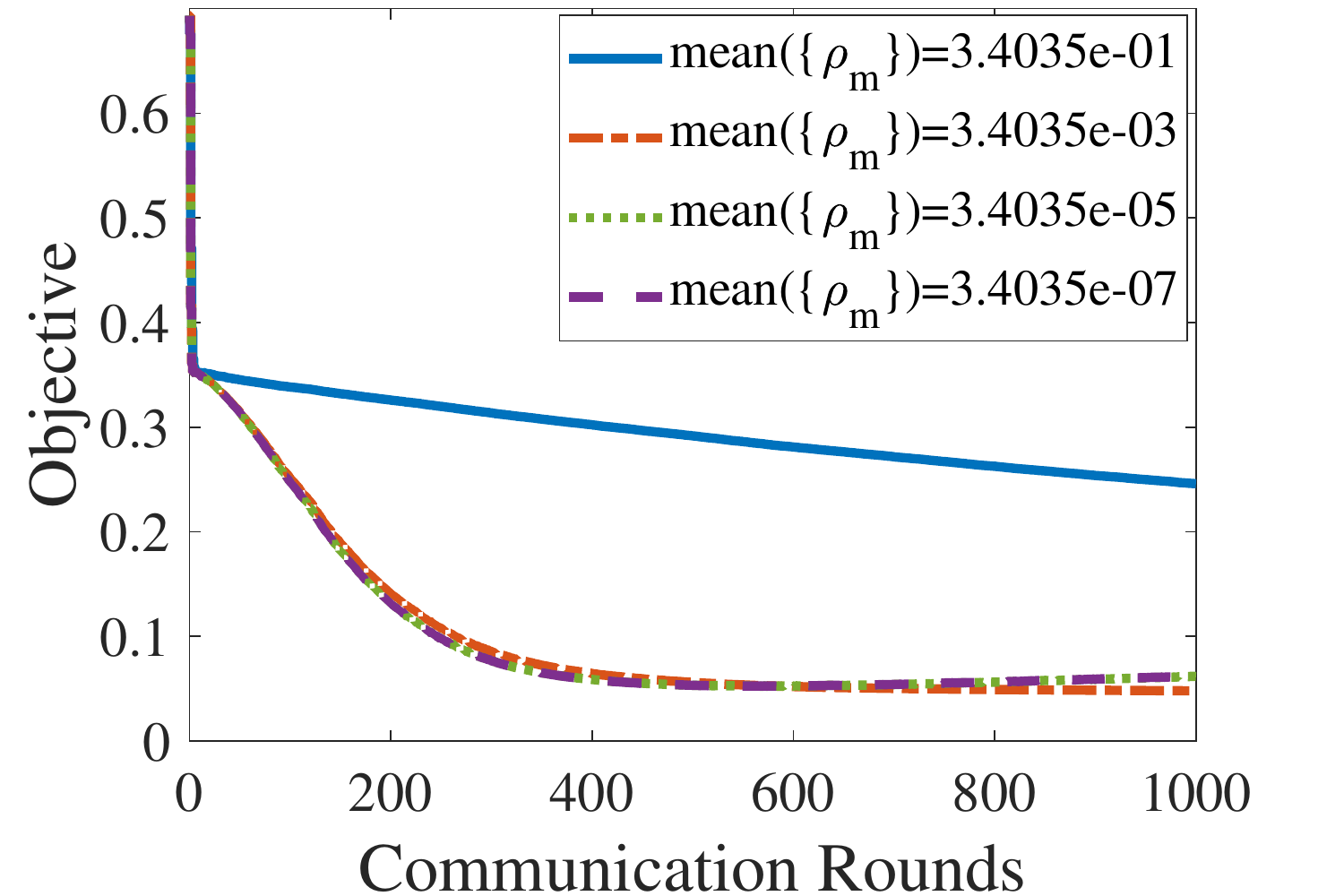}
    \caption{{\FedADMM}~\cite{zhou_and_li_fedadmm2022}}\label{fig_chD:FINAL__mnist_periodic__J10_IID0__FedADMM1_for_various_set_of_pars__obj_vs_CR}
  \end{subfigure}
  \begin{subfigure}[t]{0.495\textwidth}
    \centering
    \includegraphics[width=\textwidth,trim=32.5mm 85mm 31mm 92.75mm,clip]{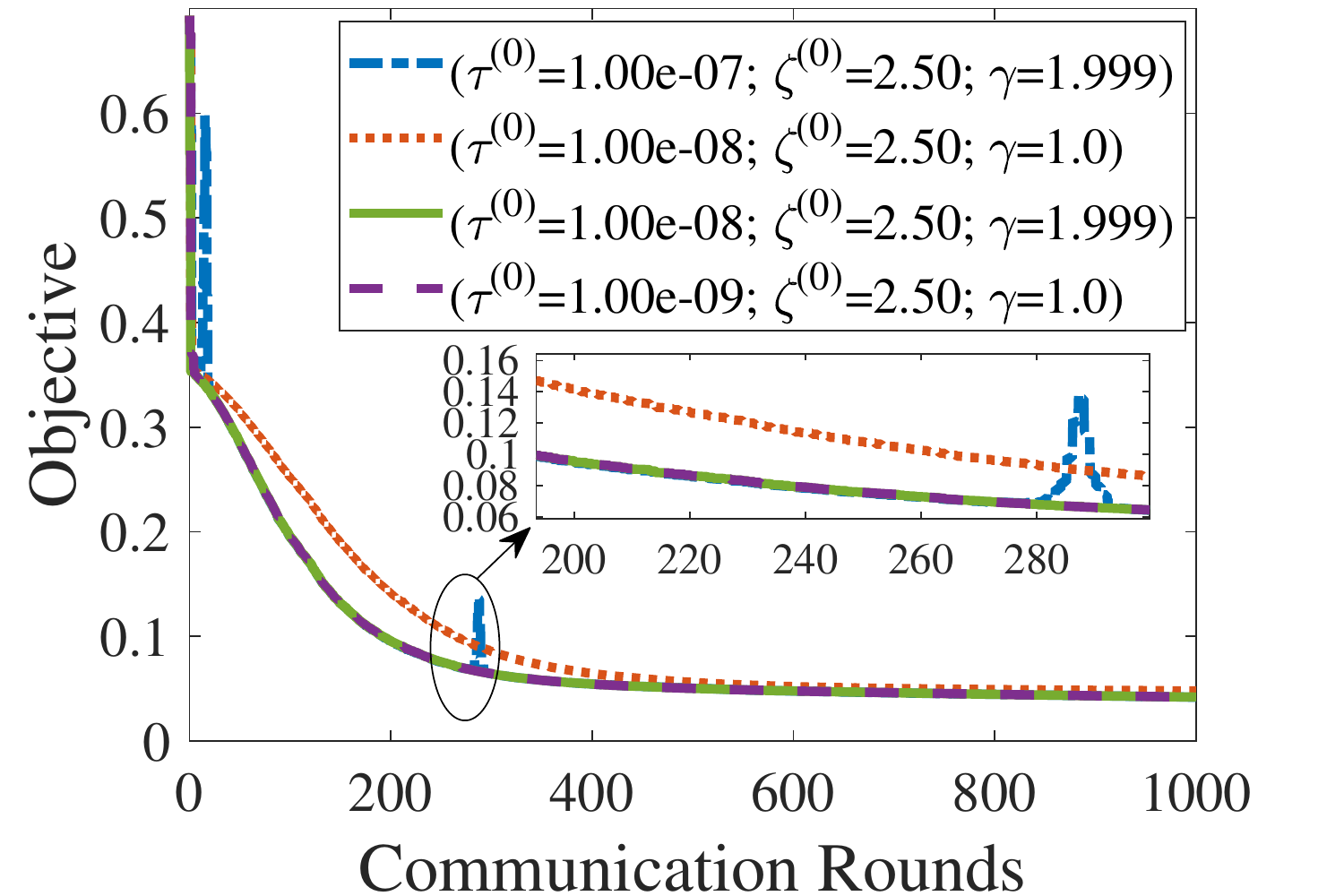}
    \caption{Proposed {\FedTOPADMM} }\label{fig_chD:FINAL__mnist_periodic__J10_IID0__FedTOPADMM1_for_various_set_of_pars__obj_vs_CR}
  \end{subfigure}
  \caption{Convergence analysis of existing, {\FedAvg}, {\FedProx}, and {\FedADMM}, and our proposed {\FedTOPADMM}\texttt{I/II} algorithms for various hyperparameters under $J\!=\!10$ local iterations.}\label{fig_chD:convergence_analysis_proposed_three_operator_admm_types_with_prior_arts}
  \end{minipage}
\end{figure*}

\begin{figure*}[!htp]
  \begin{minipage}[t]{1\linewidth}
  \begin{subfigure}[t]{0.495\textwidth}
    \centering
    \includegraphics[width=\textwidth,trim=32.5mm 85mm 31mm 88mm,clip]{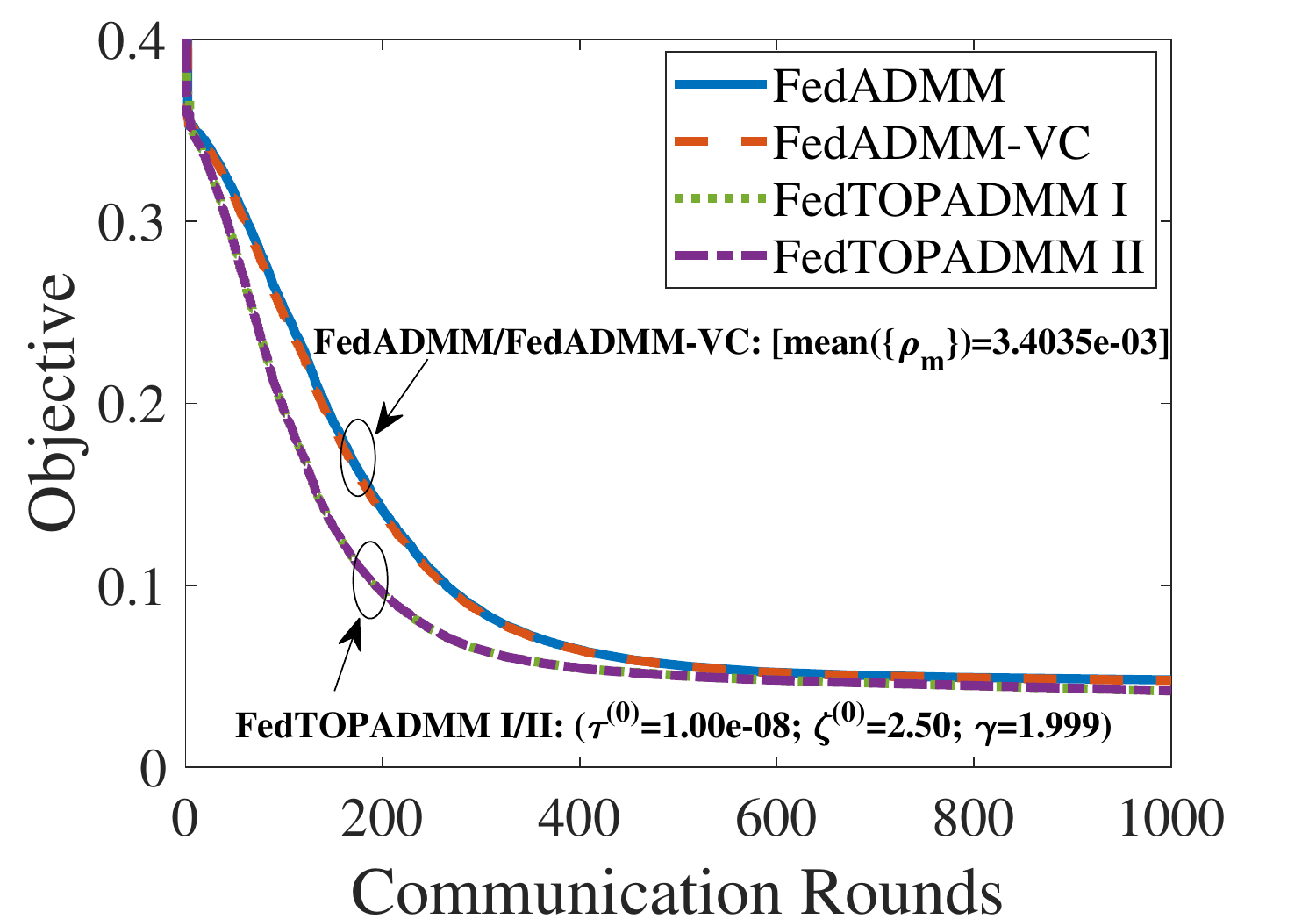}
    \caption{Objective vs. Communication rounds}\label{fig_chD:FINAL__mnist_periodic__J10_IID0__FedADMM1_vs_FedADMM1_VC__FedTOPADMM_I_vs_II__obj_vs_CR}
  \end{subfigure}
  \begin{subfigure}[t]{0.495\textwidth}
    \centering
    \includegraphics[width=\textwidth,trim=32.5mm 85mm 31mm 88mm,clip]{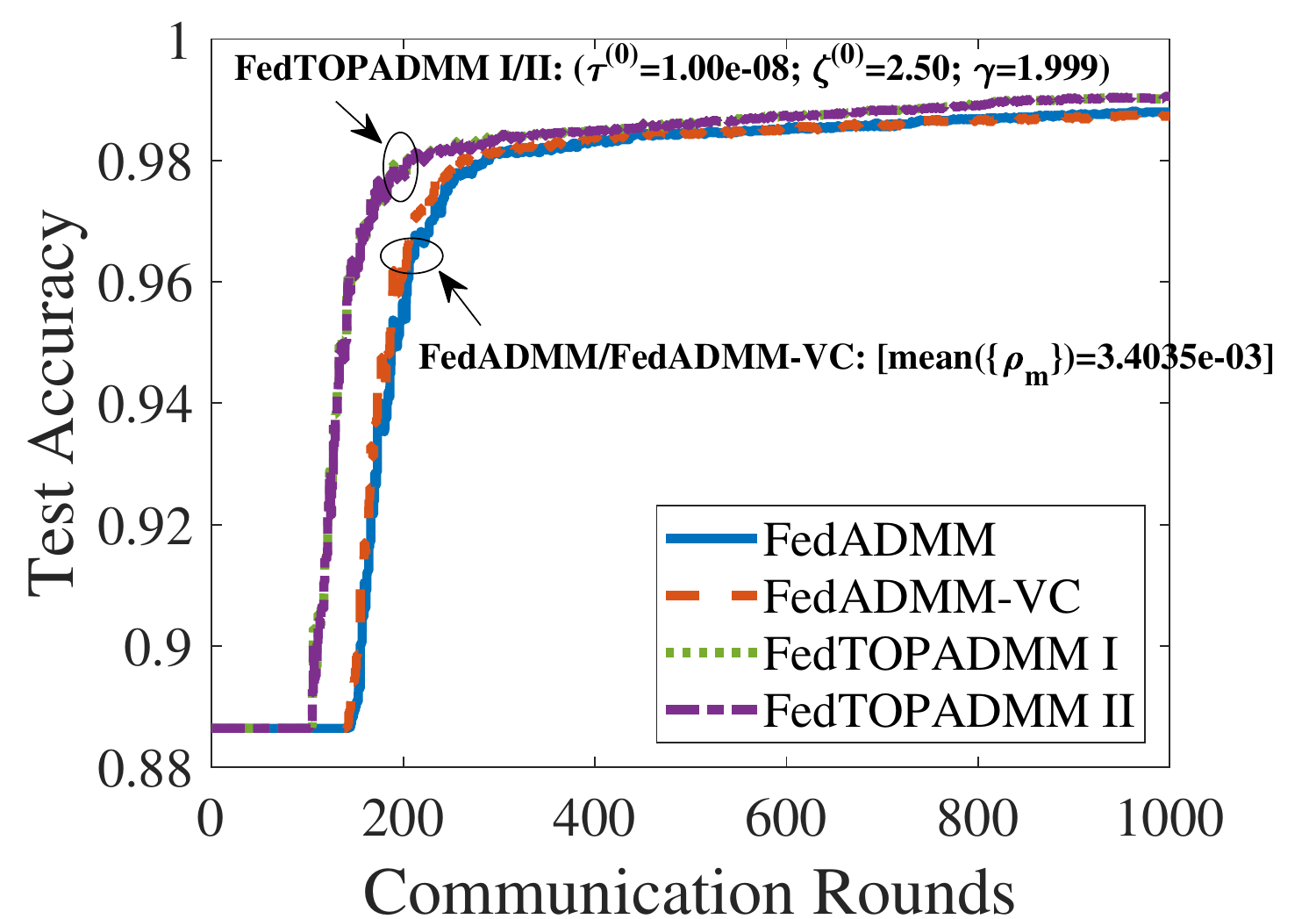}
    \caption{Test accuracy vs. Communication rounds}\label{fig_chD:FINAL__mnist_periodic__J10_IID0__FedADMM1_vs_FedADMM1_VC__FedTOPADMM_I_vs_II__Test_accuracy_vs_CR}
  \end{subfigure}
    \caption{Comparison of {\FedTOPADMM}~\texttt{I/II} with {\FedADMM} and {\FedADMM\texttt{-VC}} under $J\!=\!10$. }\label{fig_chD:cmp_perf_fedtopadmm1_and2__fedadmm_and_fedadmmvc}
  \end{minipage}
\end{figure*}

\begin{figure*}[!t]
 \centering
  \begin{minipage}[t]{.525\linewidth}
  \begin{subfigure}[t]{\textwidth}
    \centering
    \includegraphics[width=\textwidth,trim=32.5mm 85mm 31mm 92mm,clip]{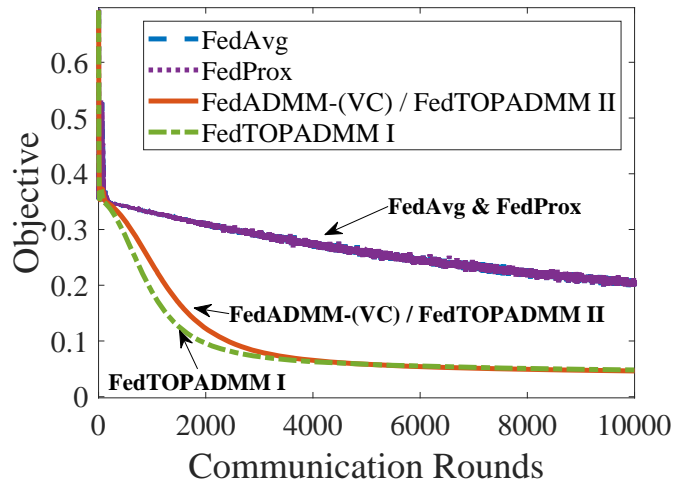}
    \caption{$J=1$}\label{fig_chD:FINAL__mnist_periodic__J1_IID0__cmp_among_all_FL_algos__obj_vs_CR}
  \end{subfigure}
  \end{minipage}\hfill
  \begin{minipage}[t]{.525\linewidth}
  \begin{subfigure}[t]{\textwidth}
    \centering
    \includegraphics[width=\textwidth,trim=32.5mm 85mm 31mm 92mm,clip]{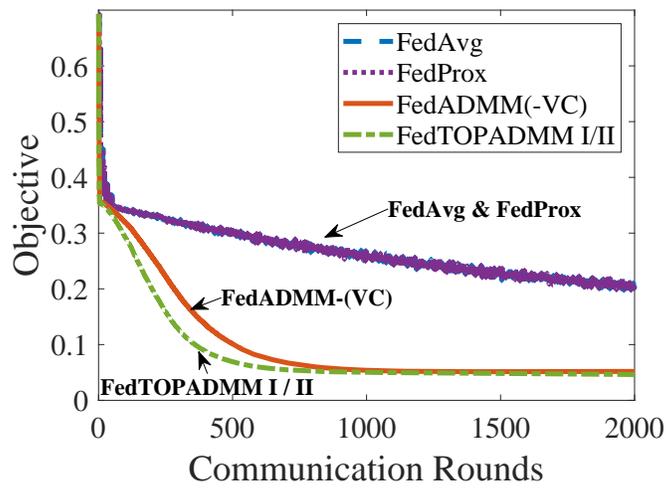}
    \caption{$J=5$}\label{fig_chD:FINAL__mnist_periodic__J5_IID0__cmp_among_all_FL_algos__obj_vs_CR}
  \end{subfigure}
  \end{minipage}\hfill
  \begin{minipage}[t]{.525\linewidth}
  \begin{subfigure}[t]{\textwidth}
  \centering
    \includegraphics[width=\textwidth,trim=32.5mm 85mm 31mm 92mm,clip]{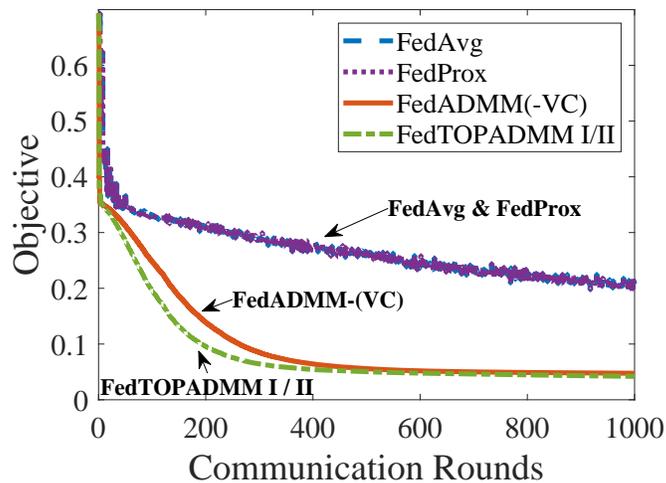}
    \caption{$J=10$}\label{fig_chD:FINAL__mnist_periodic__J10_IID0__cmp_among_all_FL_algos__obj_vs_CR}
    \end{subfigure}
  \end{minipage}
  \caption{Performance comparison, in terms of objective against communication rounds, of proposed {\FedTOPADMMI} /{\FedTOPADMMII} with  {\FedADMM} /{\FedADMM\texttt{-VC}}, {\FedProx}, and {\FedAvg}, considering three cases of local iterations on the users' side, \ie, $J\!\in\!\left\{1, 5, 10 \right\}$. }\label{fig_chD:cmp_among_all_FL_algos__for_three_Js__obj_vs_iter}
\end{figure*}

\begin{figure*}[!htp]
  \begin{minipage}[t]{1\linewidth}
  \begin{subfigure}[t]{0.495\textwidth}
    \centering
    \includegraphics[width=\textwidth,trim=32.5mm 85mm 31mm 88mm,clip]{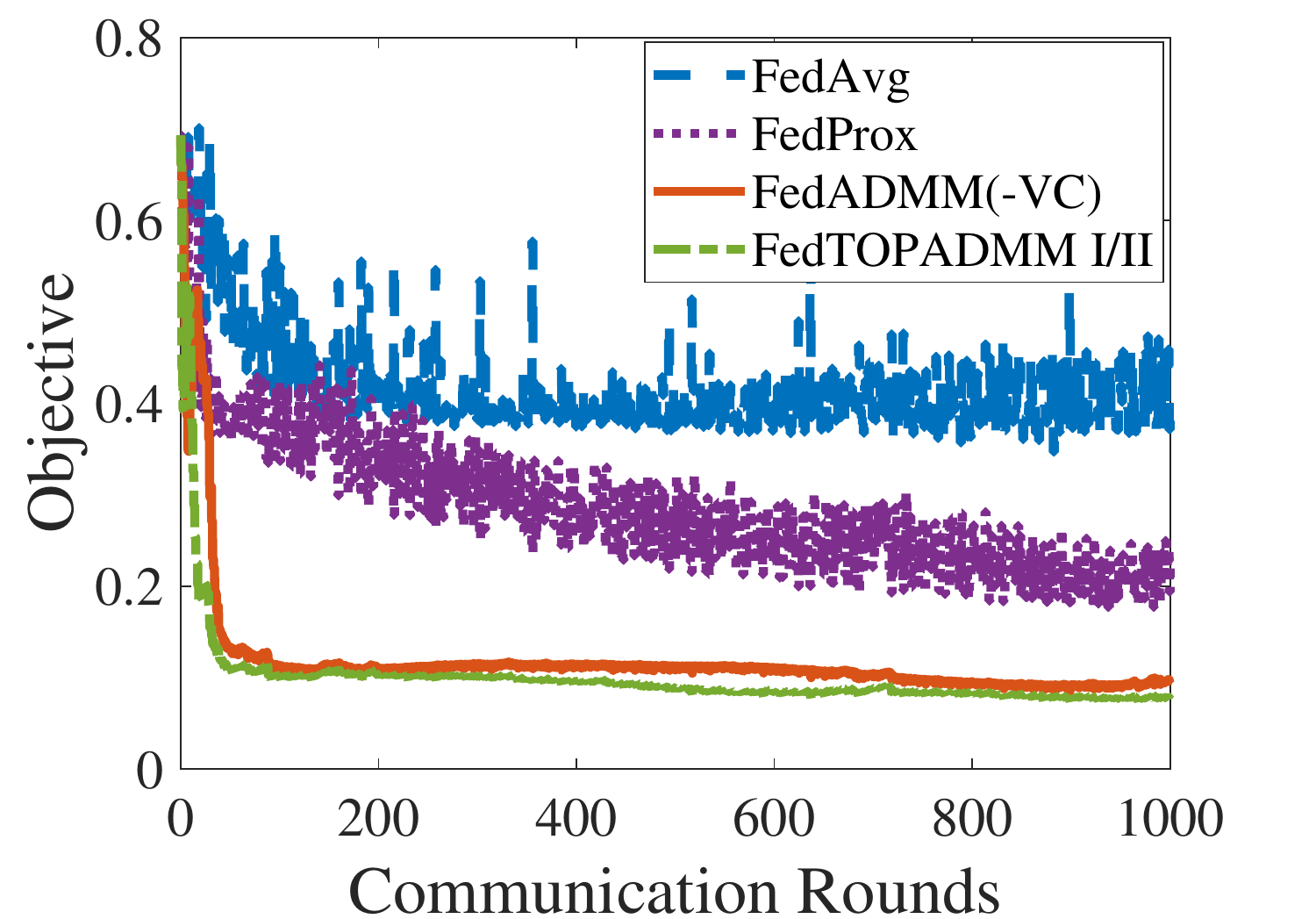}
    \caption{Objective vs. communication rounds}\label{fig_chD:mnist_nonIID_case0_J10_obj_vs_CR}
  \end{subfigure}
  \begin{subfigure}[t]{0.495\textwidth}
    \centering
    \includegraphics[width=\textwidth,trim=32.5mm 85mm 31mm 88mm,clip]{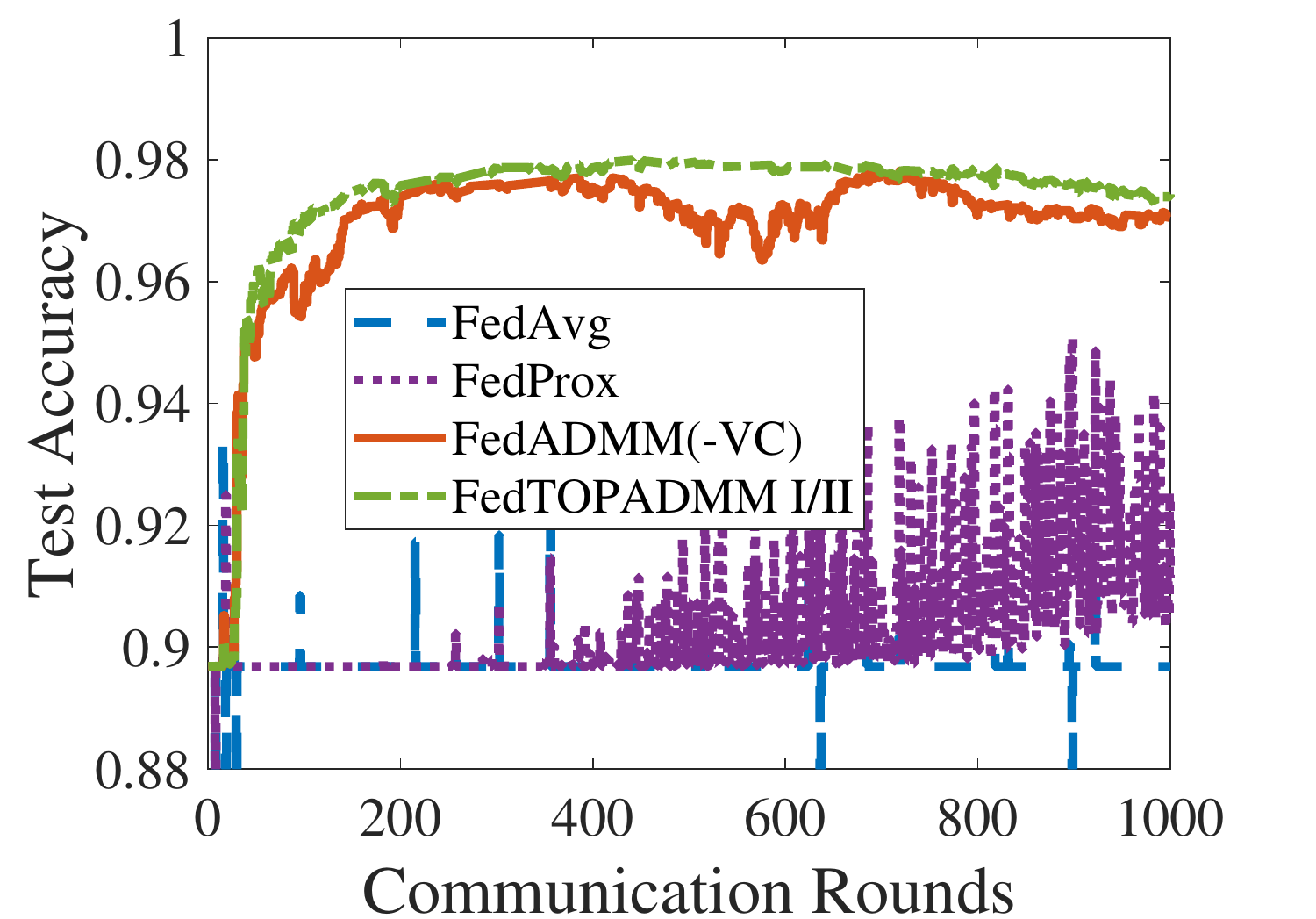}
    \caption{Test accuracy vs. communication rounds}\label{fig_chD:mnist_nonIID_case0_J10_Test_accuracy_vs_CR}
  \end{subfigure}
    \caption{\skblack{Performance comparison of proposed {\FedTOPADMM\texttt{I/II}} with existing {\FedADMM\texttt{(-VC)}}, {\FedProx}, and {\FedAvg} under non-\ac{i.i.d.} data without scaling MNIST dataset and $J=10$.} }\label{fig_chD:mnist_nonIID_case0_J10}
  \end{minipage}
\end{figure*}

\begin{figure}[!htp]
    \centering
    \begin{minipage}[t]{0.625\linewidth}
    \centering
    \includegraphics[width=\textwidth,trim=32.5mm 85mm 31mm 88mm,clip]{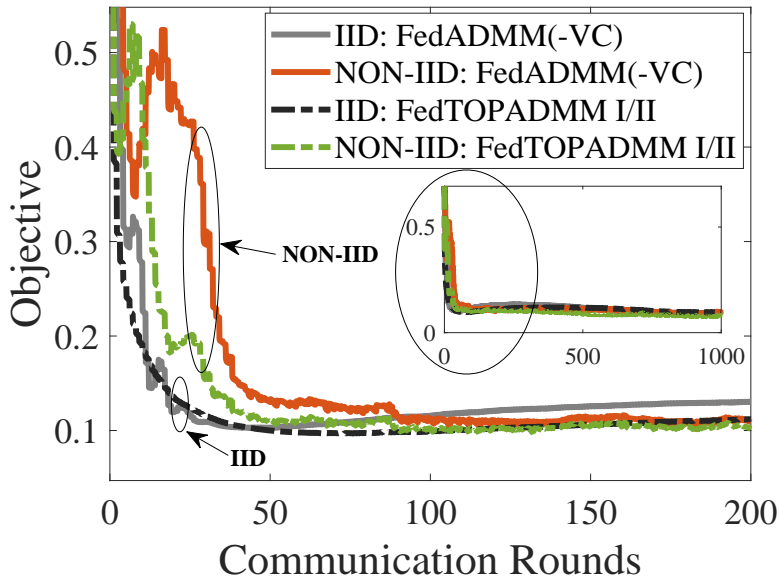}
    \caption{{\skblack{Performance comparison among \ac{i.i.d.} and non-\ac{i.i.d.} (without scaling of) MNIST dataset.}}}\label{fig_chD:mnist_nonIID__vs_IID__case0_J10_obj_vs_CR__ADMM_and_TOPADMM}
    \end{minipage}
\end{figure}

\begin{figure*}[!htp]
  \begin{minipage}[t]{1\linewidth}
  \begin{subfigure}[t]{0.495\textwidth}
    \centering
    \includegraphics[width=\textwidth,trim=32.5mm 85mm 31mm 88mm,clip]{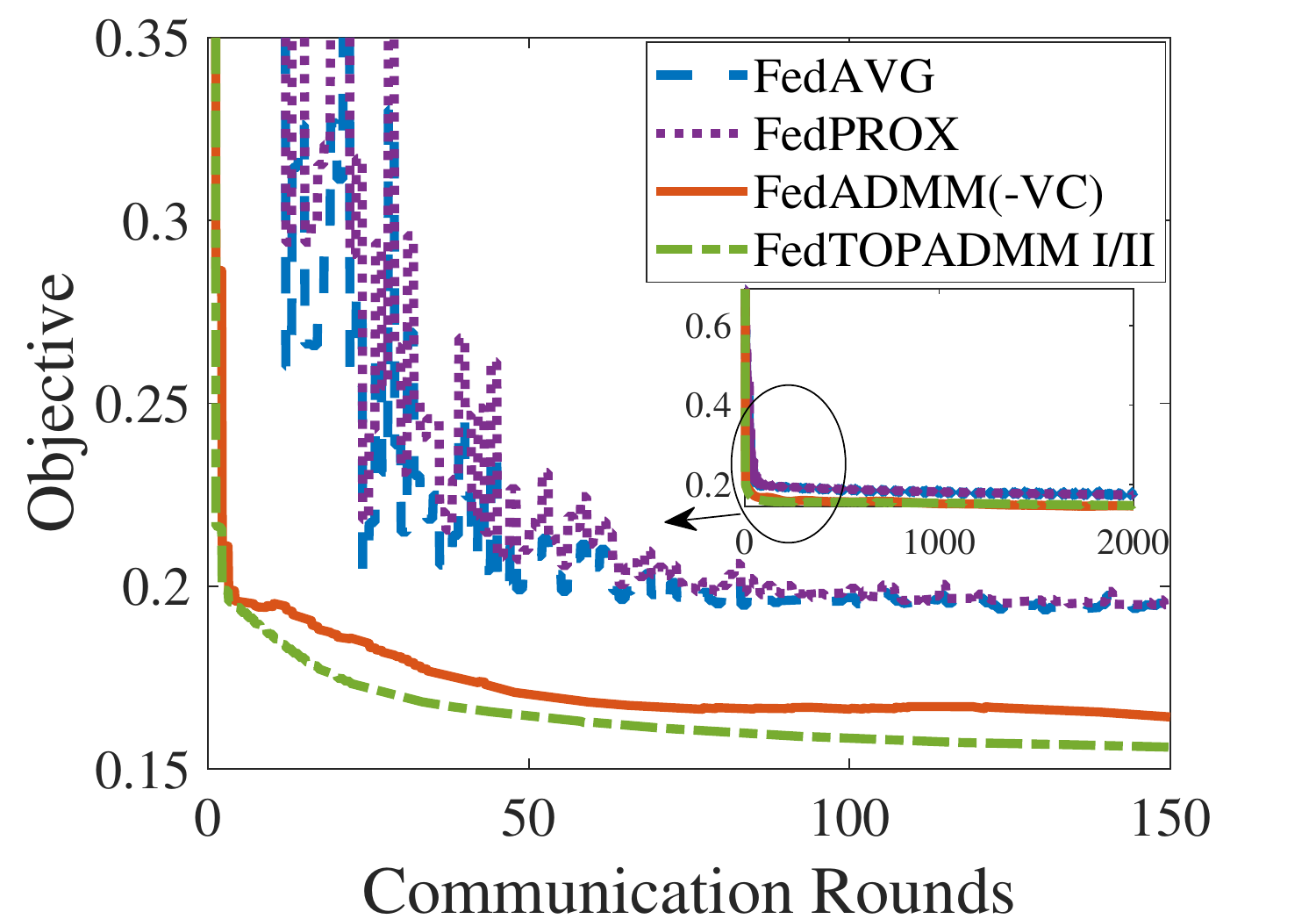}
    \caption{Objective vs. communication rounds}\label{fig_chD:cifar100_IID_J10_obj_vs_CR}
  \end{subfigure}
  \begin{subfigure}[t]{0.495\textwidth}
    \centering
    \includegraphics[width=\textwidth,trim=32.5mm 85mm 31mm 88mm,clip]{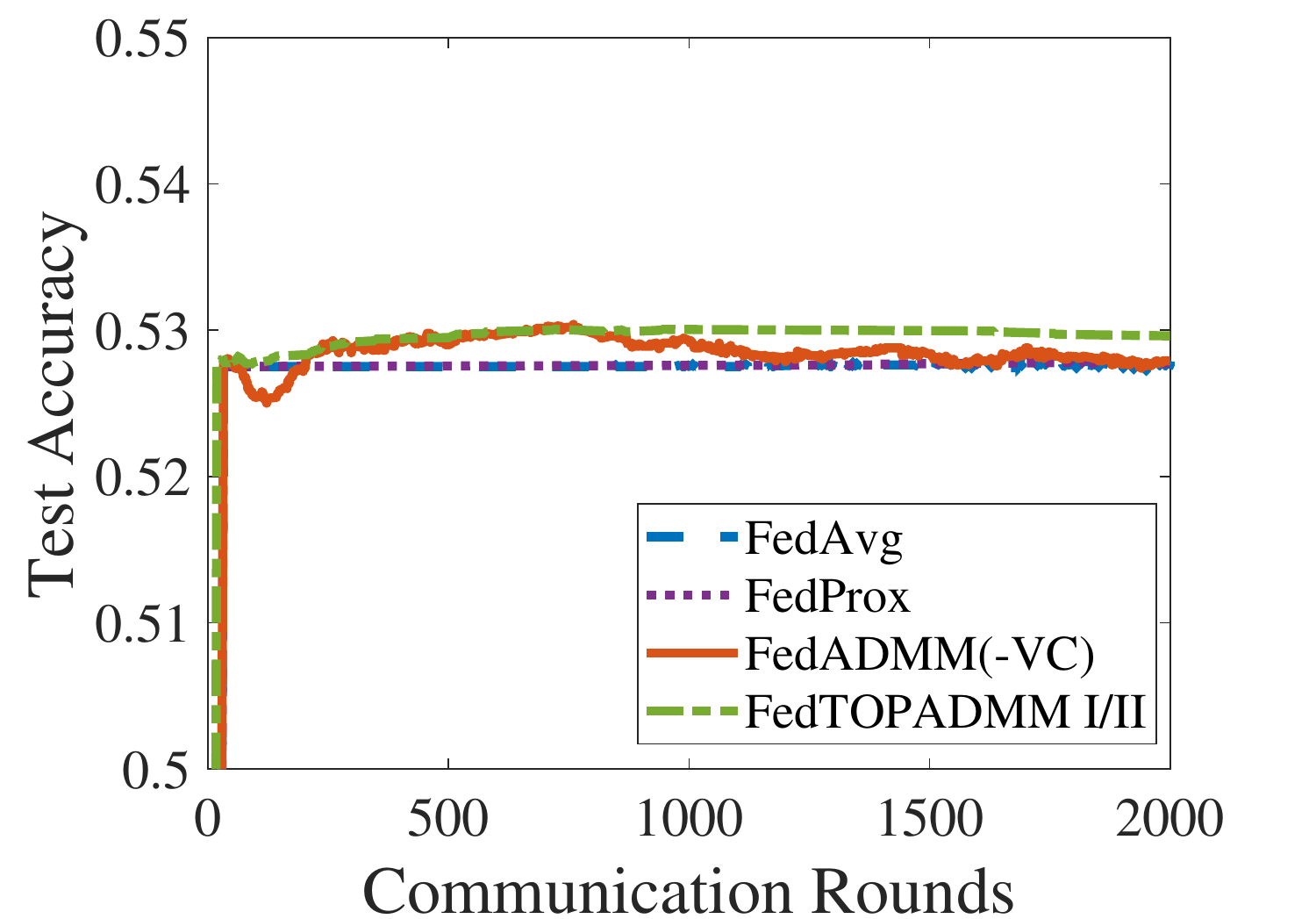}
    \caption{Test accuracy vs. communication rounds}\label{fig_chD:cifar100_IID_J10_Test_accuracy_vs_CR} 
  \end{subfigure}
    \caption{\skblack{Performance comparison of proposed {\FedTOPADMM\texttt{I/II}} with considered existing methods for CIFAR-100 \ac{i.i.d.} dataset (coarse labels) and $J=10$.} }\label{fig_chD:cifar100_IID__J10}
  \end{minipage}
\end{figure*}

\begin{figure}[!htp]
    \centering
    \begin{minipage}[t]{0.625\linewidth}
    \centering
    \includegraphics[width=\textwidth,trim=32.5mm 85mm 31mm 88mm,clip]{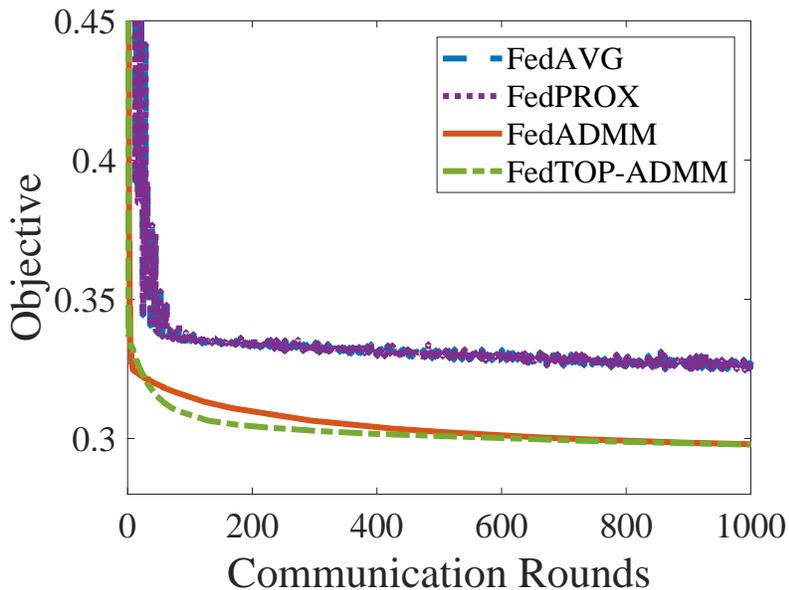}
    \caption{{\skblack{Objective vs. iterations for $J=10$ and CIFAR-10 \ac{i.i.d.} dataset.}}}\label{fig_chD:cifar10_obj_vs_CR__J10}
    \end{minipage}
\end{figure}



Figure~\ref{fig_chD:convergence_analysis_proposed_three_operator_admm_types_with_prior_arts} illustrates the convergence behaviour of {\FedAvg}~\cite{McMahan__FedAvg__2017}, {\FedProx}~\cite{Tian_Li_etal__FedProx__2020}, {\FedADMM} \cite{zhou_and_li_fedadmm2022}, and our proposed {\FedTOPADMM} for some chosen set of parameters from the candidate set. Based on these numerical results, we have selected the following parameters of the respective methods: 1) {\FedAvg} :- $\eta\!=\!10^{-5}$, 2) {\FedProx}:- $\eta\!=\!10^{-5}; \mu\!=\!0.5$, 3) {\FedADMM} /{\FedADMM\texttt{-VC}} :- $\texttt{mean}\left(\left\{ \rho_m \right\}\right)\!=\!3.4035e\!-\!3$, and 4) {\FedTOPADMMI} /{\FedTOPADMMII} :- $\texttt{mean}\left(\left\{ \rho_m \right\}\right)\!=\!3.4035e\!-\!3$, $\tau^{\left(0\right)}\!=\!1e\!-\!8$, $\zeta^{\left(0\right)}\!=\!2.5$, and $\gamma\!=\!1.999$ unless otherwise mentioned. 

Figure~\ref{fig_chD:cmp_perf_fedtopadmm1_and2__fedadmm_and_fedadmmvc} compares the performance in terms of both loss function or objective~\eqref{eqn_chD:objective_log_reg} 
and test accuracy, among {\FedTOPADMMI}, {\FedTOPADMMII}, {\FedADMM}, and \allowbreak{\FedADMM\texttt{-VC}} with the aforementioned chosen parameters. Additionally, Figure~\ref{fig_chD:cmp_among_all_FL_algos__for_three_Js__obj_vs_iter} compares the performance for $J\!\in\!\left\{1\!,5\!,10\right\}$. Noticeably, these results substantiate our argument that exploiting the data knowledge on the edge server using our proposed {\FedTOPADMM} schemes outperform  {\FedADMM\texttt{-VC}}, \ie, with a virtual client. Moreover, these results indicate that {\FedTOPADMMII} has non-noticeable performance loss compared to  {\FedTOPADMMI} when $J \!>\!1$. For instance, with $J\!=\!10$, {\FedTOPADMM} has a gain of up to 33\% in the communication efficiency with respect to {\FedADMM} to reach a test accuracy of 98\%. Furthermore, as mentioned before,  {\FedTOPADMMII} boils down to {\FedADMM} when $J\!=\!1$.

\skblack{
In Figure~\ref{fig_chD:mnist_nonIID_case0_J10}, we compare the performance of {\FedTOPADMM} with {\FedADMM} and {\FedProx} under non-\ac{i.i.d.} distribution of the MNIST dataset for $J\!=\!10$. Recall that this non-\ac{i.i.d.} data split is one of the pathological cases because each user or base station would have at most two class labels. Nevertheless, {\FedProx} performs slightly better than {\FedAvg}. However, both {\FedTOPADMM} and {\FedADMM} outperform {\FedProx} and {\FedAvg}. Further, {\FedTOPADMM} has a gain of up to 27\% in the communication efficiency with respect to {\FedADMM} to reach a test accuracy of 97\% under non-\ac{i.i.d.} distribution of MNIST dataset. Observe that we have changed the tunable parameters of all the methods compared to the previous results of MNIST. In particular, we have chosen the following parameters of respective methods: a) {\FedTOPADMM\texttt{I/II}} ($\texttt{mean}\left(\left\{ \rho_m \right\}\right)\!=\!6.5731e\!-\!6$, $\tau^{\left(0\right)}\!=\!1e\!-\!7$, $\zeta^{\left(0\right)}\!=\!1.5$), b) {\FedADMM\texttt{(-VC)}} ($\texttt{mean}\left(\left\{ \rho_m \right\}\right)\!=\!6.5731e\!-\!6$), c) {\FedProx} ($\eta\!=\!1e\!-\!3$; $\mu\!=\!0.5$), and d) {\FedAvg} ($\eta\!=\!0.5e\!-\!3$). For completeness, in Figure~\ref{fig_chD:mnist_nonIID__vs_IID__case0_J10_obj_vs_CR__ADMM_and_TOPADMM}, we compare the performance of non-\ac{i.i.d.} with \ac{i.i.d.} data considering the same tunable parameters used for Figure~\ref{fig_chD:mnist_nonIID_case0_J10}, where the performance under \ac{i.i.d.} data is unsurprisingly slightly better than the performance under non-\ac{i.i.d.} data.   
}

\skblack{
Lastly, in Figure~\ref{fig_chD:cifar100_IID__J10} and Figure~\ref{fig_chD:cifar10_obj_vs_CR__J10}, we present the performance for CIFAR-100 and CIFAR-10 dataset, respectively. We can construe the similar performance trend as observed in MNIST. 
}

\section{Conclusions} \label{sec_chD:conclusion_future_work}


In this paper, we proposed a novel {\FedTOPADMM} algorithmic framework for \linebreak communication-efficient \ac{FL} utilizing our recently proposed consensus \ac{TOP-ADMM} algorithm, which can tackle the sum of three composite functions in a distributed manner. Specifically, we developed two variants of {\FedTOPADMM}, namely {\FedTOPADMMI} and {\FedTOPADMMII} that learn a considered global machine learning model using data on both the edge server and the users. 
Our experiments showed that FedTOP-ADMM has a significant gain of up to 33\% in the communication efficiency with respect to {\FedADMM} to reach a desired test accuracy of 98\% using the proposed scaling of the MNIST dataset.
{For future works, we intend to establish the convergence analysis of FedTOP-ADMM for $J \! > \! 1$ and enhanced TOP-ADMM. Moreover, we intend to investigate the scheduling of edge devices to participate in the \ac{FL} 
using FedTOP-ADMM, as well as the power allocation of the selected devices.}


\begin{appendices}


\section{Some Useful Lemmas and Definitions}\label{appendix:use_lemmas_def}
We present herein some useful definitions, propositions and lemmas that are important to \ac{ADMM} methods.

{
\begin{definition}[$L$-smooth function~{\cite{Bauschke:2011, Beck2017}}]  \label{def_ChD:definition_of_lipschitz_continous_gradient}
A differentiable function $f\!:\!\Cm^n\! \rightarrow\! \Rm$ {is $L$-smooth, \ie,} has $L$-Lipschitz continuous gradient (for $L \!>\! 0$) if
$\left\| \nabla \!f\!\left(\vec{z}_1\right)  \!-\! \nabla\! f\!\left(\vec{z}_2\right) \right\| \!\leq\! L \!\left\|\! \vec{z}_1 \!-\! \vec{z}_2 \!\right\| \:, \forall \vec{z}_1,\! \vec{z}_2 \!\in\! \Cm^n$.
\end{definition}
}

\begin{definition}[Subgradient{\cite[Definition~16.1]{Bauschke:2011}}] \label{def_ChD:definition_of_subgradient}
Given a proper function $f\!:\!\Cm^n\! \rightarrow\! \Rm$, a vector $\vec{g} \! \in \! \Cm^n$ is denoted as a subgradient of $f\!\left(\vec{z}\right)$ at some point~$\vec{x}$~if 
\begin{align}
     f\left(\vec{z}\right) 
     &\geq f\left(\vec{x}\right) + 2 \Re \left\{ \vec{g}^\herm \left(\vec{z} - \vec{x} \right)  \right\}, \quad \forall \vec{z} \in \Cm^n.
\end{align}
\end{definition}

{
\begin{definition}[Proximal mapping~\cite{Parikh2013}\cite{Beck2017}] \label{definition_chD:prox_operator}
Let us consider a proper closed convex function $f: \dom_f \mapsto \left(-\infty\right., \left. +\infty \right]$, where $\dom_f$ corresponds to the domain of a function $f$. Then, the proximal mapping of $f$ is the operator given by: 
\begin{align*}
    \prox_{\lambda f}\left( \vec{x}\right) 
    &\coloneqq\! \left( \I + \lambda \partial f \right)^{-1} \left( \vec{x} \right), \\ 
    &= \!\arg\min_{\vec{z} \in  \dom_f} \left\{ f\left(\vec{z}\right) \! + \! \frac{1}{\beta \lambda} \left\|\vec{x} \! - \! \vec{z} \right\|_2^2\right\}, 
\end{align*}
for any $\vec{x} \! \in \!  \dom_f$, where $\partial f$ is a subdifferential of $f$~\cite{Boyd2004ConvexOptimization, Bauschke:2011}, and $\lambda \! > \! 0$. If $\vec{z}$ is complex-valued or real-valued, $\beta = 1$ or $\beta=2$, respectively. Note that the proximal operator to an indicator function becomes an orthogonal projection, \ie, $\prox_{\lambda \Ind_{\pazocal{C}}}\left( \vec{z}\right) \!=\! \proj_{\pazocal{C}}\left( \vec{z}\right)$.
\end{definition}
}

\section{Convergence Analysis of TOP-ADMM}\label{appendix:convergence_analysis_topadmm}

{
To establish the convergence of \ac{TOP-ADMM} algorithm~\ref{eqn_chD:generalized_top_admm_algorithm_iterates__for_convergence}, we first present two {standard assumptions from} \ac{ADMM} proofs in the literature followed by five lemmas in the sequel. Subsequently, these assumptions and five lemmas are required to prove Proposition~\ref{prop_chD:asymptotic_residual_error_and_objective}, which guarantees that the primal and dual residual errors vanish asymptotically. Finally, we establish the global convergence of \ac{TOP-ADMM}, \ie, proof of Theorem~\ref{thm_ChD:definition_of_topadmm_algorith__feasible_problem}. Although the proof structure is inspired by the 
convergence results from the classic \ac{ADMM}~\cite{Boyd2011}, our convergence analysis results are new from the sum of three functions with consensus constraints perspective, \ie, for \ac{TOP-ADMM}.
}



Towards the convergence analysis goal, we define the augmented Lagrangian to problem \eqref{eqn_chD:general_consensus_top_admm__generic_form} as
\begin{align} \label{eqn_chD:augmented_Lagrangian_top_admm__generic_form}
    \pazocal{L}_{\rho}&\left(\left\{ {\vec{x}}_m \right\}_{m=1}^M\!,  {\vec{z}}, \left\{{\vec{y}}_m \right\}_{m=1}^M \right) \nonumber \\ 
    \!\coloneqq& \! \sum_{m=1}^M  f_m \left( {{\vec{x}}_m} \right)   \!+\! g\left( {{\vec{z}}} \right) \!+\! \beta h\!\left( {{\vec{z}}} \right)  \nonumber \\ 
    &+ \! \sum_{m=1}^M 2\!\Re\!\left\{ {\vec{y}}_m^\herm \left( {\vec{x}}_m \! - \! {\vec{z}} \right) \right\} \!+ \! \sum_{m=1}^M  \rho \left\| {\vec{x}}_m \! - \! {\vec{z}} \right\|_2^2,
\end{align}
and, for brevity, we define {the} objective value at iteration~$i$~as
\begin{align}\label{eqn_chD:obj_value_iter_i}
{p}^{\left( i \right)} \! \coloneqq \! \sum_{m=1}^M \! f_m\! \left(\! {{\vec{x}}_m^{\left(i \right)}} \! \right)    \! + \! g\!\left( \! {{\vec{z}^{\left(i \right)}}} \! \right) \! + \!   \beta h\!\left( \! {{\vec{z}^{\left(i \right)}}}  \right).
\end{align}
Then, let us consider two assumptions that are standard in \ac{ADMM} literature~\cite{Boyd2011, Glowinski2016}.

\begin{assumption} \label{assumption:L0_lagrangian_saddle_point}
Let $\left( \left\{ {\vec{x}}_m^\star \right\}_{m=1}^M,  {\vec{z}^\star}, \left\{{\vec{y}}_m^\star \right\}_{m=1}^M\right)$ be a saddle point for the (unaugmented) Lagrangian $\pazocal{L}_{0}$ in \eqref{eqn_chD:augmented_Lagrangian_top_admm__generic_form}. Specifically, the following holds for all $ \left\{ {\vec{x}}_m \right\}_{m=1}^M,  {\vec{z}}, \left\{{\vec{y}}_m \right\}_{m=1}^M$:
\begin{align*}
    \pazocal{L}_{0}\!\left(\!\left\{ {\vec{x}}_m^\star \right\}_{m=1}^M\!,  {\vec{z}^\star}\!, \left\{{\vec{y}}_m \right\}_{m=1}^M \!\right) \!
    \!&\leq\! \pazocal{L}_{0}\!\left(\!\left\{ {\vec{x}}_m^\star \right\}_{m=1}^M\!,  {\vec{z}^\star}\!, \left\{{\vec{y}}_m^\star \right\}_{m=1}^M \!\right),\\ 
    \!&\leq \!\pazocal{L}_{0}\!\left(\!\left\{ {\vec{x}}_m \right\}_{m=1}^M\!,  {\vec{z}}, \left\{{\vec{y}}_m^\star \right\}_{m=1}^M \!\right).  
\end{align*}
\end{assumption}

\begin{assumption} \label{assumption:solution_each_subproblem}
Consider subproblems 
\eqref{eqn_chD:update_xm__step1_parallel__general_top_admm__prox__for_convergence} and \eqref{eqn_chD:update_z__step2__general_top_admm__prox__for_convergence}.  We assume that each subproblem has at least {one} solution. 
\end{assumption}

Note that {Assumption~\ref{assumption:solution_each_subproblem}} does not require the uniqueness of the solution. 

\begin{lemma} \label{lemma:diff_between_optimal_objective_and_at_iteration_kplus1}
Consider the optimal objective, $p^\star$, and the objective at iteration $i\!+\!1$, $p^{\left( i+1 \right)}$ defined in~\eqref{eqn_chD:obj_value_iter_i}. Then, the difference between the optimal objective and the objective at iteration $i\!+\!1$ is as follows: 
$$
    p^\star - p^{\left( i+1 \right)} \leq \sum_{m=1}^M 2\Re\left\{ \left(\vec{y}_m^\star\right)^\herm \Delta\vec{r}_m^{\left( i + 1 \right)} \right\}.
$$
\begin{proof}
Using the primal feasibility ${\vec{x}}_m^\star \! - \! {\vec{z}}^\star \! = \! \vec{0} \ \forall m\!=\!1,\ldots,M$, and Assumption~\ref{assumption:L0_lagrangian_saddle_point} (or duality theory \cite{Boyd2004ConvexOptimization}), let us write
\begin{equation*}
    p^{\star}\!\coloneqq\!\pazocal{L}_{0}\left(\left\{ {\vec{x}}_m^\star \right\}_{m=1}^M,  {\vec{z}^\star}, \left\{{\vec{y}}_m^\star \right\}_{m=1}^M \right) 
\end{equation*}
such that  
\begin{equation*}
    p^{\star} \leq \! \pazocal{L}_{0}\left(\left\{ {\vec{x}}_m^{\left(i+1 \right)} \right\}_{m=1}^M,  {\vec{z}^{\left(i+1 \right)}}, \left\{{\vec{y}}_m^\star \right\}_{m=1}^M \right).
\end{equation*}
Then, we can write 
\begin{align*}
p^{\star}
    \! \leq & \sum_{m=1}^M \!\Biggl[ \! f_m \!\left(\! {{\vec{x}}_m^{\left(i+1 \right)}} \!\right)  \!+\! g\!\left(\! {{\vec{z}^{\left(i+1 \right)}}} \!\right) \! + \!   \beta h\!\left( {{\vec{z}^{\left(i+1 \right)}}} \! \right) \!\Biggr] \nonumber \\
    & +\! \sum_{m=1}^M \!2\Re\!\left\{\!  \left(\vec{y}_m^\star\right)^\herm \left( \!{\vec{x}}_m^{\left(i+1 \right)} \! - \! {\vec{z}^{\left(i+1 \right)}} \right) \!\right\},    
\end{align*}
which can be rewritten as $p^{\star}\! \leq\! p^{\left( i+1 \right)} \! + \! \sum_{m=1}^M 2\Re\!\left\{\! \left(\vec{y}_m^\star\right)^\herm \Delta\!\vec{r}_m^{\left( i + 1 \right)} \right\}$,
where $\Delta\vec{r}_m^{\left( i + 1 \right)} \!=\!\! \left( {\vec{x}}_m^{\left( i+1 \right)} \! - \! {\vec{z}}^{\left( i+1 \right)} \right)$, as in Proposition~\ref{prop_chD:asymptotic_residual_error_and_objective}.
\end{proof}
\end{lemma}

The following lemma will be useful in {Lemma~\ref{lemma:diff_between_objective_at_iteration_kplus1_and_optimal_objective}}.

\begin{lemma}[{Three-point inequality}] \label{lemma:imp_relation_among_Lipschitz_gradient}
Let {the} convex and differentiable\footnote{{By differentiability, we mean Wirtinger complex gradient exist for functions $f:\Cm^n \rightarrow \Rm$, see, \eg, \cite{hjorungnes:2011}}.} function $h:\Cm^n \rightarrow \Rm$ {have} $L$-Lipschitz continuous gradient (for $L \! \geq \! 0$), where $\vec{z} \! \in \! \Cm^n$, then $\forall \vec{z}^{\left( i + 1 \right)}, \vec{z}^{\left( i \right)}$, the following inequality holds:
\begin{align} \label{eqn_chD:imp_relation_among_Lipschitz_gradient}
    2 \Re\!\left\{ \!\nabla h\left(\vec{z}^{\left( i  \right)}\!\right)^{\!\herm} \left(\! \vec{z}^{\left( i + 1 \right)} \! - \! \vec{z} \!\right) \!\right\} 
    \!\geq& h\!\left(\vec{z}^{\left( i + 1 \right)}\right) \!- \! h\!\left(\vec{z}\right)  \nonumber \\ 
    &- \!  L\! \left\| \vec{z}^{\left( i + 1 \right)} \! - \! \vec{z}^{\left( i \right)} \right\|^2 \!.
\end{align}
\begin{proof} It follows by applying the descent lemma in~\cite{Bauschke:2011}, and the convexity of $h$.
\end{proof}
\end{lemma}

In the subsequent lemma, we will use the dual residual error definition: $
    \Delta\vec{x}_m^{\left( i+1 \right)} \!\coloneqq\! \left( {\vec{x}}_m^{\left( i+1 \right)} \! - \! \vec{x}_m^{\left( i \right)} \right).
$
Observe that we will use Lemma~\ref{lemma:diff_between_optimal_objective_and_at_iteration_kplus1} and the following Lemma~\ref{lemma:diff_between_objective_at_iteration_kplus1_and_optimal_objective} in Lemma~\ref{lemma:diff_between_lyapunov_function_at_iter_kplus1_and_iter_k}.

\begin{lemma} \label{lemma:diff_between_objective_at_iteration_kplus1_and_optimal_objective}
The difference between the achieved objective at iteration $i+1$, \ie, $p^{\left( i+1 \right)}$, and the optimal objective, $p^\star$, is 
\begin{align}
    p^{\left( i+1 \right)} \! - \! p^{\star} 
    \!\leq &- \!\sum_{m=1}^M \! 2\! \Re\!\Biggl\{\! \left({\vec{y}}_m^{\left(i+1\right)} \! \right)^{\!\herm}  {\Delta\vec{r}}_m^{\left( i + 1\right)}  \nonumber \\ 
    &\hspace{10.5mm}+ \! \rho \!  \left( \!  {\vec{z}}^{\left( i + 1\right)} \! - \! {\vec{z}}^{\left( i \right)}  \right)^\herm
    \left( {\vec{x}}_m^{\left( i + 1\right)} - \vec{x}_m^\star \right) \Biggr\} \nonumber \\ 
    &+\! \rho M \tau L \! \left\| {\Delta\vec{z}}^{\left( i+1 \right)} \right\|_2^2.
\end{align}
\end{lemma}
\begin{proof}

{
We know {that} $\vec{x}_m^{\left( i + 1 \right)}$ minimizes the sequence update
 defined in \eqref{eqn_chD:update_xm__step1_parallel__general_top_admm__prox__for_convergence} such that 
\begin{align} \label{eqn_chD:x_m_kplus1_minimizes__top_admm_update__2}
    \vec{0} \!  
    &\in \!  \partial f_m\!\left( \!  {\vec{x}}_m^{\left( i + 1\right)} \!\right) \! +  \!  \rho \!  \left( \!  {\vec{x}}_m^{\left( i + 1\right)} \! - \! {\vec{z}}^{\left( i \right)} \! + \! \frac{{\vec{y}}_m^{\left( i \right)}}{\rho} \!\right), \nonumber \\
\Longleftrightarrow&
    \partial f_m\!\left( \!  {\vec{x}}_m^{\left( i + 1\right)} \right) \!
    \ni  - \! {\vec{y}}_m^{\left(i+1\right)}  \!- \! \rho \!  \left( \!  {\vec{z}}^{\left( i + 1\right)} \! - \! {\vec{z}}^{\left( i \right)}  \right),
\end{align}
where we use dual update~\eqref{eqn_chD:update_dual_ym__step3_parallel__general_top_admm}, the notation $\partial$ denotes subdifferential~\cite{Boyd2004ConvexOptimization, Bauschke:2011}, 
and $\Longleftrightarrow$ means if and only if and also overloaded as an equivalent operator.
}

Using the subgradient Definition~\ref{def_ChD:definition_of_subgradient}, we then have 
\begin{align*}
f_m\!\left( \!  {\vec{x}}_m \right) 
    \geq&
    f_m\left( \!  {\vec{x}}_m^{\left( i + 1\right)} \right) 
    \nonumber \\ 
    & + \! 2  \Re\!\Biggl\{ \! \Biggl( \!- \!  {\vec{y}}_m^{\left(i+1\right)} 
    \! - \! \rho \!  \left( \!  {\vec{z}}^{\left( i + 1\right)} \! - \! {\vec{z}}^{\left( i \right)}  \right)\!\Biggr)^\herm \!\left( {\vec{x}}_m  \!-\! {\vec{x}}_m^{\left( i + 1\right)} \right) \Biggr\},    
\end{align*}
such that
\begin{align}
\label{eqn_chD:x_m_kplus1_minimizes__top_admm_update__without_alpha}
f_m\!&\left( \!  {\vec{x}}_m  \!\right)  \!+ \! 2 \Re\!\left\{\! \left(\! {\vec{y}}_m^{\left(i+1\right)}  \! + \! \rho \!  \left( \!  {\vec{z}}^{\left( i + 1\right)} \! - \! {\vec{z}}^{\left( i \right)} \! \right)\!\right)^\herm \!{\vec{x}}_m \!\right\} \nonumber \\ 
    &\geq \! f_m\!\left( \!  {\vec{x}}_m^{\left( i + 1\right)} \!\right) \!+  \! 2 \Re\!\left\{\! \left( \!{\vec{y}}_m^{\left(i+1\right)}  \! + \! \rho \!  \left( \!  {\vec{z}}^{\left( i + 1\right)} \! - \! {\vec{z}}^{\left( i \right)}  \!\right)\!\right)^\herm \! {\vec{x}}_m^{\left( i + 1\right)} \!\right\}. 
\end{align}


Similarly, we also know {that} $\vec{z}^{\left( i + 1 \right)}$
minimizes \eqref{eqn_chD:update_z__step2__general_top_admm__prox__for_convergence}. Then, let us consider $\beta \!\coloneqq\! \rho M \tau$ such that
\begin{align} \label{eqn_chD:z_kplus1_minimizes__top_admm_update__1}
    \vec{0} 
    \!\in& \partial g\!\left( {\vec{z}}^{\left( i+1 \right)} \right) \! + \! \underbrace{\rho M \tau}_{\coloneqq \beta} \nabla h \! \left( \! {\vec{z}}^{\left( i \right)} \right) \nonumber \\ 
    &- \! \sum_{m=1}^M \!\left(\! \underbrace{\rho \left( \! {\vec{x}}_m^{\left( i+1 \right)} \! - \! {\vec{z}}^{\left( i+1 \right)} \right) \! + \! {\vec{y}}_m^{\left( i \right)} }_{= {\vec{y}}_m^{\left( i+1 \right)}} \!\right), \nonumber \\
   \Longleftrightarrow &
    \partial \! g\!\left( {\vec{z}}^{\left( i+1 \right)} \right) 
    \! \ni - \! \beta \nabla h \! \left( \! {\vec{z}}^{\left( i \right)} \right) \!+ \! \sum_{m=1}^M {\vec{y}}_m^{\left( i+1 \right)}.
\end{align}

Hence, using \eqref{eqn_chD:z_kplus1_minimizes__top_admm_update__1} and the subgradient Definition~\ref{def_ChD:definition_of_subgradient}, note that $$ g\left( {\vec{z}}^{\star} \right) 
    \! \geq \! g\left( {\vec{z}}^{\left( i\!+\!1 \right)} \right) 
    \!+ \! 2\!\Re\!\left\{\! \left( -\beta \! \nabla \!h \! \left( \! {\vec{z}}^{\left( i \right)} \! \right) \!+ \! \sum_{m=1}^M \! {\vec{y}}_m^{\left( i\!+\!1 \right)}  \right)^{\!\herm} \! \left(\! {\vec{z}}^{\star} \!-\! {\vec{z}}^{\left( i\!+\!1 \right)}\!\right) \!\right\}$$.
Thus, the following inequalities hold:
\begin{align}
    g\!\left( {\vec{z}}^{\star} \right) &\! - \!\sum_{m=1}^M  2\Re\!\left\{\! \left(\! {\vec{y}}_m^{\left( i+1 \right)} \! \right)^\herm \!{\vec{z}}^{\star} \!\right\} \nonumber \\ \!
    \geq& \! g\!\left(\! {\vec{z}}^{\left( i+1 \right)} \right) \!
    - \! \sum_{m=1}^M \!  2\Re\!\left\{\! \left( \! {\vec{y}}_m^{\left( i+1 \right)}  \right)^\herm \!{\vec{z}}^{\left( i+1 \right)} \!\right\} \nonumber \\ 
    &+ \! \beta 2\Re\!\left\{\! \left(  \nabla h \! \left( \! {\vec{z}}^{\left( i \right)} \right) \right)^\herm \left( -\!{\vec{z}}^{\star} \!+\! {\vec{z}}^{\left( i+1 \right)} \right)\! \right\}, \nonumber \\
    \overset{\text{(b)}}{\geq}& g\left( {\vec{z}}^{\left( i+1 \right)} \right) \! - \! \sum_{m=1}^M  2\Re\left\{ \left( {\vec{y}}_m^{\left( i+1 \right)}  \right)^\herm {\vec{z}}^{\left( i+1 \right)} \right\} \nonumber \\
    \label{eqn_chD:z_kplus1_minimizes__top_admm_update__2}
    &+ \! \beta h \! \left( \! {\vec{z}}^{\left( i+1 \right)} \right) \!-\! \beta h \! \left( \! {\vec{z}}^{\star} \right) \! - \! \rho M \tau L \! \left\| \underbrace{{\vec{z}}^{\left( i+1 \right)} \!-\! {\vec{z}}^{\left( i \right)}}_{={\Delta\vec{z}}^{\left( i+1 \right)} } \right\|_2^2,
\end{align} 
where in $\text{(b)}$ we used inequality \eqref{eqn_chD:imp_relation_among_Lipschitz_gradient}. Then, {inequality} \eqref{eqn_chD:z_kplus1_minimizes__top_admm_update__2} can be rearranged as
\begin{align}
    \label{eqn_chD:z_kplus1_minimizes__top_admm_update__3}
    &g\!\left( \!{\vec{z}}^{\left( i+1 \right)} \right) 
    \!+\! \beta h\!\left(\! {\vec{z}}^{\left( i+1 \right)} \right) \!
    - \! \sum_{m=1}^M \! 2 \Re \! \left\{\! \left({\vec{y}}_m^{\left(i+1\right)}\right)^\herm {\vec{z}}^{\left( i+1 \right)} \!\right\} \nonumber \\
    &\leq \! g\!\left(\! \vec{z}^\star \!\right) \! +\! \beta h\!\left(\! \vec{z}^\star \!\right) \!
    - \! \sum_{m=1}^M \!2 \Re \!\left\{\! \left(\!{\vec{y}}_m^{\left(i+1\right)}\!\right)^{\!\herm} \vec{z}^\star \!\right\}
    \! + \! \rho M \!\tau \!L \!\left\| \! {\Delta\!\vec{z}}^{\left( i+1 \right)} \!\right\|_2^2\!.
\end{align}

Now, we add \eqref{eqn_chD:x_m_kplus1_minimizes__top_admm_update__without_alpha} for all $m\!=\!1,\ldots,M$ into \eqref{eqn_chD:z_kplus1_minimizes__top_admm_update__3} such that
\begin{align}
    \sum_{m=1}^M \!f_m \!&\left( \!  {\vec{x}}_m^{\left( i + 1\right)} \right) \!+ \! 2 \!\Re\!\left\{\! \left( \!{\vec{y}}_m^{\left(i+1\right)}  \! + \! \rho \!  \left( \!  {\vec{z}}^{\left( i + 1\right)} \! - \! {\vec{z}}^{\left( i \right)}  \!\right)\!\right)^\herm \!{\vec{x}}_m^{\left( i + 1\right)} \!\right\} \nonumber \\
    &\!+\! g\!\left( {\vec{z}}^{\left( i+1 \right)} \right) \! +\! \beta h\!\left(\! {\vec{z}}^{\left( i+1 \right)} \right) \!- \! \sum_{m=1}^M \!2 \!\Re \! \left\{\! \left(\!{\vec{y}}_m^{\left(i+1\right)}\right)^\herm \!{\vec{z}}^{\left( i+1 \right)} \!\right\},\nonumber \\
    \leq& \sum_{m=1}^M \!f_m\!\left( \!  {\vec{x}}_m^{\star} \right) \!+ \! 2 \Re\!\left\{\! \left(\!{\vec{y}}_m^{\left(i+1\right)}  \! + \! \rho \!  \left( \!  {\vec{z}}^{\left( i + 1\right)} \! - \! {\vec{z}}^{\left( i \right)}  \right)\!\right)^\herm {\vec{x}}_m^{\star} \!\right\} \nonumber \\ 
    &\!+ \! g\!\left( \vec{z}^\star \right) \! +\! \beta h\!\left( \vec{z}^\star \right) \!
    - \! \sum_{m=1}^M 2 \Re \left\{ \left({\vec{y}}_m^{\left(i+1\right)}\right)^\herm \vec{z}^\star \right\}  \nonumber \\
    &{+\! \rho M \tau L \left\| {\Delta\vec{z}}^{\left( i+1 \right)} \right\|_2^2}, \nonumber \\
    \Longleftrightarrow& 
    p^{\left( i + 1 \right)} \! + \! \sum_{m=1}^M 2 \Re\Biggl\{ \left({\vec{y}}_m^{\left(i+1\right)}  \right)^\herm \left( {\vec{x}}_m^{\left( i + 1\right)} - {\vec{z}}^{\left( i+1 \right)} \right) \nonumber \\
    &\hspace{15.5mm} + \! \rho \!  \left( \!  {\vec{z}}^{\left( i + 1\right)} \! - \! {\vec{z}}^{\left( i \right)}  \right)^\herm
    {\vec{x}}_m^{\left( i + 1\right)}  \Biggr\}, \nonumber \\
    \leq& p^{\star} \! + \! \sum_{m=1}^M 2 \Re\Biggl\{ \left({\vec{y}}_m^{\left(i+1\right)}  \right)^\herm \left( {\vec{x}}_m^{\star} -  {\vec{z}}^{\star} \right) \nonumber \\ 
    &\hspace{7.5mm}\! + \! \rho \!  \left( \!  {\vec{z}}^{\left( i + 1\right)} \! - \! {\vec{z}}^{\left( i \right)}  \right)^\herm  {\vec{x}}_m^{\star}  \Biggr\} \!+\! \rho M \tau L \left\| {\Delta\vec{z}}^{\left( i+1 \right)} \right\|_2^2.
\end{align}
Considering the primal feasibility result, $ {\vec{x}}_m^{\star} \! - \! {\vec{z}}^{\star} \!= \! \vec{0}$, and the primal residual error definition, ${\Delta\vec{r}}_m^{\left( i + 1\right)} \!=\! {\vec{x}}_m^{\left( i + 1\right)} \! - \! {\vec{z}}^{\left( i+1 \right)}$, in the above inequality results {in}
\begin{align*}
    p^{\left( i + 1 \right)} \! - \! p^{\star} \! \leq & - \! \sum_{m=1}^M \! 2 \Re\Biggl\{ \! \left({\vec{y}}_m^{\left(i+1\right)} \! \right)^{\!\herm} \! {\Delta\vec{r}}_m^{\left( i + 1\right)}  \nonumber \\ 
    &\hspace{12.5mm}\! + \! \rho \!  \left( \!  {\vec{z}}^{\left( i + 1\right)} \! - \! {\vec{z}}^{\left( i \right)} \!  \right)^{\!\herm}
    \left( \! {\vec{x}}_m^{\left( i + 1\right)} \! - \! \vec{x}_m^\star \! \right) \!\Biggr\}  \nonumber \\ 
    &\!+\! \rho M \tau L \left\| {\Delta\vec{z}}^{\left( i+1 \right)} \right\|_2^2 \nonumber.
\end{align*}

\end{proof}

We define {the} following 
function for the subsequent lemma.

\begin{definition} \label{definition:definition_of_lyapunov_function}
Let a Lyapunov candidate function for the \ac{TOP-ADMM} algorithm at given iteration $i$ be defined as
\begin{align}  \label{eqn_chD:definition_of_lyapunov_function}
   {V}^{\left( i \right)} \coloneqq \sum_{m=1}^M \frac{1}{\rho} \left\| \vec{y}_m^{\left( i \right)} -  \vec{y}_m^\star \right\|_2^2 +  \rho \left\| \vec{z}^{\left( i \right)} -  \vec{z}^\star \right\|_2^2.
\end{align}
\end{definition}

\begin{lemma} \label{lemma:diff_between_lyapunov_function_at_iter_kplus1_and_iter_k}
The difference between the Lyapunov function \eqref{eqn_chD:definition_of_lyapunov_function} at every iteration $i+1$  and the previous iteration $i$ fulfils the following inequality:
\begin{align} \label{eqn_chD:diff_of_Lyapunov_function}
    {V}^{\left( i+1 \right)} &\!-\! {V}^{\left( i \right)} \nonumber 
    \\ 
    \leq&
    \sum_{m=1}^M  -\rho \Biggl[ \left\| \Delta\vec{r}_m^{\left(i+1\right)} \right\|_2^2  \!+ \! { \left\| {\Delta\vec{z}}^{\left( i+1 \right)} \right\|_2^2 \!+\! \tau L \left\| {\Delta\vec{z}}^{\left( i \right)} \right\|_2^2 } \Biggr].
\end{align}
\begin{proof}
We add the inequalities of Lemma~\ref{lemma:diff_between_optimal_objective_and_at_iteration_kplus1} and Lemma~\ref{lemma:diff_between_objective_at_iteration_kplus1_and_optimal_objective}, and then rearrange the terms such that 
\begin{align}  \label{eqn_chD:decomposed_sum_of_lemma1_and_lemma2__1}
    &\hspace{-10.5mm}\underbrace{\Bigl( p^{\left( i + 1 \right)}  \!-\! p^{\star} \Bigr) \!+\! \Bigl(  p^\star \!-\! p^{\left( i+1 \right)} \Bigr) }_{=0}\! \nonumber \\ 
    \leq& - \sum_{m=1}^M 2 \Re\!\Biggl\{ \left({\vec{y}}_m^{\left(i+1\right)}  \right)^\herm  {\Delta\vec{r}}_m^{\left( i + 1\right)} \Biggr\}  \nonumber \\ &\!-\! \sum_{m=1}^M 2 \Re\Biggl\{ \rho \!  \left( \!  {\vec{z}}^{\left( i + 1\right)} \! - \! {\vec{z}}^{\left( i \right)}  \right)^\herm
    \underbrace{\left( {\vec{x}}_m^{\left( i + 1\right)} \! - \! \vec{x}_m^\star \right)}_{ = {\Delta\vec{r}}_m^{\left( i + 1\right)} \!+\! {\vec{z}}^{\left( i + 1\right)} \!-\! \vec{z}^\star} \! \Biggr\} \nonumber \\
    &+ \! \sum_{m=1}^M 2\Re\left\{ \left(\vec{y}_m^\star\right)^\herm \Delta\vec{r}_m^{\left( i + 1 \right)} \right\} \!+\! \rho M \tau L \left\| {\Delta\vec{z}}^{\left( i+1 \right)} \right\|_2^2, \nonumber \\
    \Longleftrightarrow
    0 \geq& \sum_{m=1}^M \Biggl[ \underbrace{ \Re\left\{ 2\left({\vec{y}}_m^{\left(i+1\right)} - \vec{y}_m^\star \right)^\herm  {\Delta\vec{r}}_m^{\left( i + 1\right)} \right\} }_{\text{(a)}} \nonumber \\
    &+ 2\Re\left\{ \rho \left( \!  {\vec{z}}^{\left( i + 1\right)} \! - \! {\vec{z}}^{\left( i \right)}  \right)^\herm {\Delta\vec{r}}_m^{\left( i + 1\right)} \right\}  \nonumber \\
    & + \left. 2 \Re\left\{  \rho \left( \!  {\vec{z}}^{\left( i + 1\right)} \! - \! {\vec{z}}^{\left( i \right)}  \right)^\herm \left( {\vec{z}}^{\left( i + 1\right)} - \vec{z}^\star \right) \right\} \right] \nonumber \\ 
    &- \rho M \tau L \left\| {\Delta\vec{z}}^{\left( i+1 \right)} \right\|_2^2.
\end{align}
We rewrite part $\text{(a)}$ of \eqref{eqn_chD:decomposed_sum_of_lemma1_and_lemma2__1} using the dual update \eqref{eqn_chD:update_dual_ym__step3_parallel__general_top_admm}, ${\vec{y}}_m^{\left(i+1\right)}\!=\!{\vec{y}}_m^{\left(i\right)} + \rho \Delta\vec{r}_m^{\left(i+1\right)}$, as follows:
\begin{align}
    \Re&\left\{ 2\left({\vec{y}}_m^{\left(i+1\right)} - \vec{y}_m^\star \right)^\herm  {\Delta\vec{r}}_m^{\left( i + 1\right)} \right\} \nonumber \\
    =& \Re\Biggl\{ 2\left({\vec{y}}_m^{\left(i\right)} \! - \! \vec{y}_m^\star \right)^\herm   \underbrace{{\Delta\vec{r}}_m^{\left( i + 1\right)}}_{= \nicefrac{1}{\rho} \left( {\vec{y}}_m^{\left(i+1\right)}\!-\!{\vec{y}}_m^{\left(i\right)} \right)}  \! + \!  \underbrace{\rho \left\| \Delta\vec{r}_m^{\left(i+1\right)} \right\|_2^2}_{= \nicefrac{1}{\rho}\left\| {\vec{y}}_m^{\left(i+1\right)} \! - \! {\vec{y}}_m^{\left(i\right)} \right\|_2^2 } \Biggr\} \nonumber \\ 
    &\!+\! \Re\Biggl\{ \rho \left\| \Delta\vec{r}_m^{\left(i+1\right)} \right\|_2^2  \Biggr\}, \nonumber \\
    =& \frac{1}{\rho} \! \left( \! {2\Re\!\left\{\!\left(\!{\vec{y}}_m^{\left(i\right)} \! - \! \vec{y}_m^\star \!\right)^\herm \! \left(\!{\vec{y}}_m^{\left(i+1\right)} \! - \! {\vec{y}}_m^{\left(i\right)} \!\right) \!\right\}  \! + \!  \left\| {\vec{y}}_m^{\left(i+1\right)} \! - \! {\vec{y}}_m^{\left(i\right)} \right\|_2^2} \! \right)  \nonumber \\ 
    &\!+\! \rho \left\| \Delta\vec{r}_m^{\left(i+1\right)} \right\|_2^2,  \nonumber \\
    \label{eqn_chD:decomposed_sum_of_lemma1_and_lemma2__1__part1__1}
    =& \frac{1}{\rho} \! \left( \!\left\| \!{\vec{y}}_m^{\left(i+1\right)} \! - \! {\vec{y}}_m^{\star} \! \right\|_2^2 \! - \! \left\|\! {\vec{y}}_m^{\left(i\right)} - {\vec{y}}_m^{\star} \! \right\|_2^2 \right)  \! + \! \rho \left\| \Delta\vec{r}_m^{\left(i+1\right)} \right\|_2^2\!.
\end{align}

Now, {replace} part $\text{(a)}$ of \eqref{eqn_chD:decomposed_sum_of_lemma1_and_lemma2__1} with \eqref{eqn_chD:decomposed_sum_of_lemma1_and_lemma2__1__part1__1} such that 
\eqref{eqn_chD:decomposed_sum_of_lemma1_and_lemma2__1}  {becomes} 
\begin{align} \label{eqn_chD:decomposed_sum_of_lemma1_and_lemma2__2}
    0 \geq& \sum_{m=1}^M  \bigggl[\underbrace{\frac{1}{\rho} \! \left( \left\| {\vec{y}}_m^{\left(i+1\right)} \! - \! {\vec{y}}_m^{\star} \right\|_2^2 \! -  \left\| {\vec{y}}_m^{\left(i\right)} \! - \! {\vec{y}}_m^{\star} \right\|_2^2 \right)}_{\text{(a)}} \nonumber \\ 
    &\hspace{5.5mm}+ \underbrace{ \rho \left\| \Delta\vec{r}_m^{\left(i+1\right)} \right\|_2^2 }_{\text{(b)}} \!+ \! \underbrace{2   \Re\left\{ \rho \left( \!  {\vec{z}}^{\left( i + 1\right)} \! - \! {\vec{z}}^{\left( i \right)}  \right)^\herm {\Delta\vec{r}}_m^{\left( i + 1\right)} \right\} }_{\text{(c)}}  \nonumber \\
    &\hspace{5.5mm}+ \underbrace{2 \Re\left\{  \rho \left( \!  {\vec{z}}^{\left( i + 1\right)} \! - \! {\vec{z}}^{\left( i \right)}  \right)^\herm \left( {\vec{z}}^{\left( i + 1\right)} - \vec{z}^\star \right) \right\} }_{\text{(d)}} \bigggr]  \nonumber \\ 
    &\!-\! \rho M \tau L \left\| {\Delta\vec{z}}^{\left( i+1 \right)} \right\|_2^2. 
\end{align}

We rewrite the sum of parts $\text{(b)}$--$\text{(d)}$ of \eqref{eqn_chD:decomposed_sum_of_lemma1_and_lemma2__2}~such~that
\begin{align} \label{eqn_chD:decomposed_sum_of_lemma1_and_lemma2__2__part2__1}
     \rho& \left\|\!  \Delta\! \vec{r}_m^{\left(i+1\right)} \right\|_2^2 \!+\! \Re\! \left\{\! 2\rho \left( \!  {\vec{z}}^{\left( i + 1\right)} \! - \! {\vec{z}}^{\left( i \right)}  \! \right)^\herm {\Delta\! \vec{r}}_m^{\left( i + 1\right)} \!\right\} \nonumber \\
     &\!+\! \Re\!\left\{ \! 2 \rho \!\left( \!  {\vec{z}}^{\left( i + 1\right)} \! - \! {\vec{z}}^{\left( i \right)}  \right)^\herm \!\left( \underbrace{ {\vec{z}}^{\left( i + 1\right)} \! -\!  \vec{z}^\star}_{= \left({\vec{z}}^{\left( i + 1\right)} - {\vec{z}}^{\left( i \right)} \right) \! +\!  \left({\vec{z}}^{\left( i\right)}  - \vec{z}^\star \right)} \right) \!\right\}, \nonumber \\
     =& \underbrace{\rho \left\| \Delta\vec{r}_m^{\left(i+1\right)} \right\|_2^2 \!+\! \Re\left\{ 2\rho \left( \!  {\vec{z}}^{\left( i + 1\right)} \! - \! {\vec{z}}^{\left( i \right)}  \right)^\herm {\Delta\vec{r}}_m^{\left( i + 1\right)} \right\}}_{\text{(a)}} \nonumber \\
     &\!+\! \underbrace{ \rho \left\|  {\vec{z}}^{\left( i + 1\right)} \! - \! {\vec{z}}^{\left( i \right)}  \right\|_2^2 }_{\text{(b)}} \!+\! \rho \left\|  {\vec{z}}^{\left( i + 1\right)} \! - \! {\vec{z}}^{\left( i \right)}  \right\|_2^2 \nonumber  \\
     &+\! \Re\left\{ 2 \rho\left( {\vec{z}}^{\left( i + 1\right)} \! - \! {\vec{z}}^{\left( i \right)} \right)^\herm \left({\vec{z}}^{\left( i\right)}  - \vec{z}^\star \right) \right\}. 
\end{align}

Note that the sum of part $\text{(a)}$ and $\text{(b)}$ of \eqref{eqn_chD:decomposed_sum_of_lemma1_and_lemma2__2__part2__1} equals  $\rho \left\| \Delta\vec{r}_m^{\left(i+1\right)}  \!+\! \left( {\vec{z}}^{\left( i + 1\right)} \! - \! {\vec{z}}^{\left( i \right)}  \right) \right\|_2^2$. Thus, \eqref{eqn_chD:decomposed_sum_of_lemma1_and_lemma2__2__part2__1} can be expressed as follows:
\begin{align}
     \rho& \! \left\| \! \Delta\vec{r}_m^{\left(i+1\right)}  \!+\! \left( \! {\vec{z}}^{\left( i + 1\right)} \! - \! {\vec{z}}^{\left( i \right)}  \right) \! \right\|_2^2 \!+\! \rho \!\left\| \underbrace{{\vec{z}}^{\left( i + 1\right)} \! - \! {\vec{z}}^{\left( i \right)}}_{ = \left({\vec{z}}^{\left( i + 1\right)} \!-\! {\vec{z}}^\star \right) \!-\! \left({\vec{z}}^{\left( i \right)} \!-\! {\vec{z}}^\star \right) }  \right\|_2^2  \nonumber \\
     &\!+\! \Re \left\{2 \rho\left( \underbrace{{\vec{z}}^{\left( i + 1\right)} \! - \! {\vec{z}}^{\left( i \right)}}_{= \left({\vec{z}}^{\left( i + 1\right)} - {\vec{z}}^\star \right) - \left({\vec{z}}^{\left( i \right)} - {\vec{z}}^\star \right)} \right)^\herm \left({\vec{z}}^{\left( i\right)}  - \vec{z}^\star \right) \right\}, \nonumber \\
     \label{eqn_chD:decomposed_sum_of_lemma1_and_lemma2__2__part2__5}
     =& \rho \left\| \Delta\vec{r}_m^{\left(i+1\right)}  + \left( {\vec{z}}^{\left( i + 1\right)} \! - \! {\vec{z}}^{\left( i \right)}  \right) \right\|_2^2 \!+ \! \rho  \left\|  {\vec{z}}^{\left( i + 1\right)} - {\vec{z}}^\star \right\|_2^2 \nonumber \\
     &\! - \! \rho \left\|  {\vec{z}}^{\left( i\right)} - {\vec{z}}^\star \right\|_2^2.
\end{align}

Now, we substitute the sum of parts $\text{(b)}$--$\text{(d)}$ of \eqref{eqn_chD:decomposed_sum_of_lemma1_and_lemma2__2} with \eqref{eqn_chD:decomposed_sum_of_lemma1_and_lemma2__2__part2__5} such that \eqref{eqn_chD:decomposed_sum_of_lemma1_and_lemma2__2} becomes
\begin{align} \label{eqn_chD:decomposed_sum_of_lemma1_and_lemma2__3}
    0 
    \geq& 
    \! \sum_{m=1}^M  \bigggl[ \left( \frac{1}{\rho}  \left( \left\| {\vec{y}}_m^{\left(i+1\right)} - {\vec{y}}_m^{\star} \right\|_2^2 \right) \! + \! \rho  \left\|  {\vec{z}}^{\left( i + 1\right)} - {\vec{z}}^\star \right\|_2^2 \right)  \nonumber \\ 
    &\hspace{5.5mm}-\! \left( \frac{1}{\rho}  \left( \left\| {\vec{y}}_m^{\left(i\right)} - {\vec{y}}_m^{\star} \right\|_2^2 \right) \! + \! \rho  \left\|  {\vec{z}}^{\left( i \right)} - {\vec{z}}^\star \right\|_2^2 \right) \bigggr] \nonumber \\ 
    &+ \sum_{m=1}^M \rho \left\| \Delta\vec{r}_m^{\left(i+1\right)}  + \left( {\vec{z}}^{\left( i + 1\right)} \! - \! {\vec{z}}^{\left( i \right)}  \right) \right\|_2^2  \nonumber \\
    &-\! \rho M \tau L \left\| {\Delta\vec{z}}^{\left( i+1 \right)} \right\|_2^2.
\end{align}

Using the Lyapunov definition---see~\eqref{eqn_chD:definition_of_lyapunov_function}---in \eqref{eqn_chD:decomposed_sum_of_lemma1_and_lemma2__3}, the inequality can be rearranged as
\begin{align} \label{eqn_chD:decomposed_sum_of_lemma1_and_lemma2__4}
    &\hspace{-3.5mm}{V}^{\left( i+1 \right)} \!-\! {V}^{\left( i \right)} \nonumber \\
    \leq& \sum_{m=1}^M  -\rho \left\| \Delta\vec{r}_m^{\left(i+1\right)}  + \left( {\vec{z}}^{\left( i + 1\right)} \! - \! {\vec{z}}^{\left( i \right)}  \right) \right\|_2^2 \!+\! \rho M \!\tau \! L \!\left\| {\Delta\vec{z}}^{\left( i+1 \right)} \right\|_2^2, \nonumber \\
    =& \sum_{m=1}^M \! -\rho \Biggl( \! \left\|\! \Delta\vec{r}_m^{\left(i+1\right)} \right\|_2^2 \!+\! \left\|\! {\Delta\vec{z}}^{\left( i+1 \right)} \right\|_2^2  \!- \tau L\left\|\! {\Delta\vec{z}}^{\left( i+1 \right)} \right\|_2^2 \!\Biggr) \nonumber \\ 
    &-\! \sum_{m=1}^M \!\rho 2\Re\!\left\{\! \left(\Delta\vec{r}_m^{\left(i+1\right)}\right)^\herm \left( {\vec{z}}^{\left( i + 1\right)} \! - \! {\vec{z}}^{\left( i \right)}  \right) \!\right\}. 
\end{align}

We will now bound the last component in \eqref{eqn_chD:decomposed_sum_of_lemma1_and_lemma2__4}. Utilizing the result in \eqref{eqn_chD:z_kplus1_minimizes__top_admm_update__3} by recalling that $\vec{z}^{\left( i + 1 \right)}$  
minimizes \eqref{eqn_chD:update_z__step2__general_top_admm__prox__for_convergence}, subsequently applying the result of three-point inequality Lemma~\ref{lemma:imp_relation_among_Lipschitz_gradient}, and replacing $\vec{z}^\star$ with $\vec{z}^{\left( i \right)}$, 
the following inequality {holds at iteration $i\!+\!1$} with ${\vec{y}}_m^{\left(i+1\right)}$ at hand:
{
\begin{align}
\label{eqn_chD:inequality1__due_to_z_kplus1_minimizes__top_admm_update}
    g\!\left(\! {\vec{z}}^{\left( i + 1\right)} \!\right) &\! +\! \beta h\!\left(\! {\vec{z}}^{\left( i + 1\right)} \!\right) \!
    \!- \! \sum_{m=1}^M \! 2 \!\Re\! \left\{\! \left(\!{\vec{y}}_m^{\left(i+1\right)}\!\right){\!^\herm} \!\vec{z}^{\left( i + 1\right)} \!\right\} \nonumber \\
    \leq& g\!\left(\! {\vec{z}}^{\left( i \right)} \!\right) \! +\! \beta h\!\left(\! {\vec{z}}^{\left( i \right)} \!\right) \!- \! \sum_{m=1}^M \!2 \Re \!\left\{\! \left(\!{\vec{y}}_m^{\left(i+1\right)}\!\right){\!^\herm} \!\vec{z}^{\left( i \right)}\! \right\} \nonumber \\ 
    &+ \!\rho M \!\tau L \!\left\|\! {\Delta\!\vec{z}}^{\left( i+1 \right)} \!\right\|_2^2\!. 
\end{align}   
}
Similarly, recalling that $\vec{z}^{\left( i  \right)}$
minimizes \eqref{eqn_chD:update_z__step2__general_top_admm__prox__for_convergence}, subsequently applying the result of the three-point inequality in Lemma~\ref{lemma:imp_relation_among_Lipschitz_gradient}, and replacing $\vec{z}^\star$ with $\vec{z}^{\left( i \right)}$ at the $i$-th iteration with ${\vec{y}}_m^{\left(i\right)}$ at hand, the following inequality {holds}:
{
\begin{align}
    \label{eqn_chD:inequality2__due_to_z_kplus1_minimizes__top_admm_update}
    g\left(\! {\vec{z}}^{\left( i \right)} \right) &\! +\! \beta h\!\left(\! {\vec{z}}^{\left( i\right)} \right) \!- \! \sum_{m=1}^M \!2 \Re \!\left\{\! \left(\!{\vec{y}}_m^{\left(i\right)}\right)^{\!\herm} \vec{z}^{\left( i \right)} \!\right\}  \nonumber \\
    \!\leq& \! g\!\left( {\vec{z}}^{\left( i+1 \right)} \right) \! +\! \beta h\!\left( {\vec{z}}^{\left( i+1 \right)} \right) \!- \! \sum_{m=1}^M \!2 \Re \!\left\{\! \left({\vec{y}}_m^{\left(i\right)}\right){\!^\herm} \vec{z}^{\left( i + 1\right)} \!\right\} \nonumber \\ 
    &+ \! \rho M \tau L \! \left\| \!{\Delta\vec{z}}^{\left( i \right)} \right\|_2^2\!.
\end{align}
}
Now, we add 
\eqref{eqn_chD:inequality1__due_to_z_kplus1_minimizes__top_admm_update} and \eqref{eqn_chD:inequality2__due_to_z_kplus1_minimizes__top_admm_update}, and rearrange such that 
\begin{align}
    -\sum_{m=1}^M &2 \Re \left\{ \left( \underbrace{{\vec{y}}_m^{\left(i+1\right)} \! - \! {\vec{y}}_m^{\left(i\right)}}_{= \rho \Delta\vec{r}_m^{\left(i+1\right)} \ \text{using \eqref{eqn_chD:update_dual_ym__step3_parallel__general_top_admm}} } \right)^\herm \left(\vec{z}^{\left( i + 1\right)} \!-\! \vec{z}^{\left( i\right)} \right) \right\},  \nonumber \\  
    \leq&{\rho M \tau L \left(\left\| {\Delta\vec{z}}^{\left( i+1 \right)} \right\|_2^2 \!+\! \left\| {\Delta\vec{z}}^{\left( i \right)} \right\|_2^2\right) } \nonumber \\
    \label{eqn_chD:decomposed_sum_of_lemma1_and_lemma2__4__part2__bound}
    \Longleftrightarrow  
    -\sum_{m=1}^M &\rho 2 \Re \left\{ \left(  \Delta\vec{r}_m^{\left(i+1\right)} \right)^\herm \left( {\vec{z}}^{\left( i + 1\right)} \! - \! {\vec{z}}^{\left( i \right)}  \right) \right\}  \nonumber \\   
    \leq& {\rho M \tau L \left(\left\| {\Delta\vec{z}}^{\left( i+1 \right)} \right\|_2^2 \!+\! \left\| {\Delta\vec{z}}^{\left( i \right)} \right\|_2^2\right) } .
\end{align}
{We use 
\eqref{eqn_chD:decomposed_sum_of_lemma1_and_lemma2__4__part2__bound} in \eqref{eqn_chD:decomposed_sum_of_lemma1_and_lemma2__4}{,} which finally {yields} \eqref{eqn_chD:diff_of_Lyapunov_function}.}
\end{proof}
\end{lemma}

{In the subsequent lemma, we show that the Lyapunov function is {non-increasing} for $\tau\!\geq\!0$. 
}

\begin{lemma} \label{lemma:proof_of_nonnegativity_Lyapunov}
{ Consider the Lyapunov function ${V}^{\left( i \right)}$ from Definition~\ref{definition:definition_of_lyapunov_function}. Then, ${V}^{\left( i \right)}$ is non-increasing {over the iterations} when  $\tau\!\geq\!0$.}
\begin{proof}
{
The proof follows when $\tau\!\geq\!0$ is considered in the right hand side of \eqref{eqn_chD:diff_of_Lyapunov_function} (with $\rho \!>\!0$)}, \ie, 
$ 
    {-\rho \left[  \left\|\!\Delta\vec{r}_m^{\left( i+1 \right)}\!\right\|^2 \!+\! \left\|\!{\Delta\vec{z}}^{\left( i+1 \right)} \!\right\|^2 \!+\! \tau \!L \! \left\|\! {\Delta\vec{z}}^{\left( i \right)} \right\|_2^2 \right] \!  \leq \! 0 \ \forall i}.
$
\end{proof}
\end{lemma}

We are now prepared to present and prove Proposition~\ref{prop_chD:asymptotic_residual_error_and_objective}, which establishes the convergence to zero of the residual error, the objective residual error, and the primal residual error.

\begin{prop} \label{prop_chD:asymptotic_residual_error_and_objective} 
The \ac{TOP-ADMM} iterative scheme in \eqref{eqn_chD:generalized_top_admm_algorithm_iterates__for_convergence} ensures that the residual error, the objective residual error, and the primal residual error converge to zero, asymptotically:
\begin{align} \label{eqn_chD:def_of_delta_z}
    \lim_{i \rightarrow +\infty} {\Delta \vec{z}}^{\left( i+1 \right)} \! =& 0,\\
    \lim_{i \rightarrow +\infty} \left( {p}^{\left( i +1 \right)} \! - \! {p}^\star \right) \!=& 0,\\
    \lim_{i \rightarrow +\infty} \Delta\vec{r}_m^{\left( i+1 \right)} \! =& 0, \quad \forall m=1,\ldots,M.
\end{align}

\begin{proof}

{Using the result of Lemma~\ref{lemma:diff_between_lyapunov_function_at_iter_kplus1_and_iter_k} and Lemma~\ref{lemma:proof_of_nonnegativity_Lyapunov}, \ie, \eqref{eqn_chD:diff_of_Lyapunov_function}, and ${V}^{\left( i \right)} \! \geq \! 0$ for every iteration $i$, we conclude}
$
    {V}^0 
    \!\geq \! \sum_{i=1}^\infty \left( {V}^{\left( i \right)} - {V}^{\left( i+1 \right)} \right) 
    \geq \sum_{i=1}^\infty \sum_{m=1}^M  \rho \Biggl[ \left\| \Delta\vec{r}_m^{\left(i+1\right)} \right\|_2^2 \! +\! \left\| {\Delta\vec{z}}^{\left( i+1 \right)} \right\|_2^2 \! +\! {  \tau L \left\| {\Delta\vec{z}}^{\left( i \right)} \right\|_2^2  }  \Biggr]
$
where the left hand series is absolutely convergent as ${V}^0 \!<\! \infty$. This absolute convergence implies that 
$
    \lim_{i \rightarrow +\infty} {\Delta \vec{z}}^{\left( i+1 \right)} \!\coloneqq\! \lim_{i \rightarrow +\infty} \left( {\vec{z}}^{\left( i+1 \right)} \! - \! {\vec{z}}^{\left( i \right)} \right) \! = \! 0 
$
and {for all $m\!=\!1,\ldots,M$,}
$
    \lim_{i \rightarrow +\infty} \Delta\vec{r}_m^{\left( i+1 \right)} \coloneqq \lim_{i \rightarrow +\infty} \left( {\vec{x}}_m^{\left( i+1 \right)} \! - \! {\vec{z}}^{\left( i+1 \right)} \right) \! = \! 0.
$ 
Furthermore, Lemma~\ref{lemma:diff_between_optimal_objective_and_at_iteration_kplus1} results in
$
    \lim_{i \rightarrow +\infty} {p}^{\left( i+1 \right)} \! - \! {p}^\star \! = \! 0,
$
since $\lim_{i \rightarrow +\infty} \Delta\vec{r}_m^{\left( i+1\right)} = 0$ for all $m\!=\!1,\ldots,M$. Therefore, the proof is completed.
\end{proof}
\end{prop}

\begin{corollary} \label{corollary:x_m_vanishes_asymptotically}
{For every $m$} 
$
    \lim_{i \rightarrow +\infty} \Delta\vec{x}_m^{\left( i+1 \right)} \coloneqq \lim_{i \rightarrow +\infty} \left( {\vec{x}}_m^{\left( i+1 \right)} \! - \! \vec{x}_m^{\left( i \right)} \right) \! = \! 0.
$
\begin{proof}
{This is a direct consequence of Proposition~\ref{prop_chD:asymptotic_residual_error_and_objective}, specifically on the results on the asymptotic dual and primal residual error.}
\end{proof}
\end{corollary}

In addition to Proposition~\ref{prop_chD:asymptotic_residual_error_and_objective} and Corollary~\ref{corollary:x_m_vanishes_asymptotically}, the following lemma is used for the convergence proof of Theorem~\ref{thm_ChD:definition_of_topadmm_algorith__feasible_problem} {of TOP-ADMM}. 

\begin{lemma} \label{lemma:delta_gradient__h_z__and_h_x_m_vanishes_asymptotically}
{Given} the results from  Proposition~\ref{prop_chD:asymptotic_residual_error_and_objective} and Corollary~\ref{corollary:x_m_vanishes_asymptotically}, and applying Definition \ref{def_ChD:definition_of_lipschitz_continous_gradient}, we have
\begin{subequations} \label{eqn_chD:delta_gradient__h_x_m_vanishes_asymptotically}
\begin{align} 
    \lim_{i \rightarrow +\infty}& \left(\nabla  h \! \left( \! {\vec{z}}^{\left( i+1 \right)} \right) \! - \! \nabla h \! \left( \! {\vec{z}}^{\left( i \right)} \right) \right) \! = \! 0;  \\
    \lim_{i \rightarrow +\infty}&  \left(\nabla  h \! \left( \! {\vec{x}}_m^{\left( i+1 \right)} \right) \! - \! \nabla h \! \left( \! {\vec{x}}_m^{\left( i \right)} \right) \right) \! = \! 0, \ \forall m.
\end{align} 
\end{subequations}

\end{lemma}

{To this end, we are now ready to establish the convergence of \ac{TOP-ADMM}.}
\begin{proof}[Global convergence proof of Theorem~\ref{thm_ChD:definition_of_topadmm_algorith__feasible_problem}]
According to the \ac{KKT} optimality conditions---see, \eg, \cite{Boyd2004ConvexOptimization}, {the necessary and sufficient optimality conditions for the considered general problem \eqref{eqn_chD:general_consensus_top_admm__generic_form} are the dual feasibility, \ie, } 
\begin{align}
    \label{eqn_chD:stationarity_condition__u_var_with_dual_var_in_top_admm_ver1}
    \vec{0} 
    &\in \frac{\partial}{\partial {\left(\vec{x}_m^\star\right)}^*} \pazocal{L}_{0}\left(\left\{ {\vec{x}}_m^\star \right\}_{m=1}^M,  {\vec{z}^\star}, \left\{{\vec{y}}_m^\star \right\}_{m=1}^M \right), \nonumber \\
    \Longleftrightarrow 
    \vec{0} & 
    \in   \partial f_m\left( {{\vec{x}}_m^\star} \right)  + {\vec{y}}_m^\star \ , \forall m\!=\!1,\ldots,M, \\
\label{eqn_chD:stationarity_condition__v_var_with_dual_var_in_top_admm_ver1}
    \vec{0} 
    &\in \frac{\partial}{\partial {\left(\vec{z}^\star\right)}^*} \pazocal{L}_{0}\left(\left\{ {\vec{x}}_m^\star \right\}_{m=1}^M,  {\vec{z}^\star}, \left\{{\vec{y}}_m^\star \right\}_{m=1}^M \right), \nonumber \\ 
    \Longleftrightarrow
    \vec{0} &
    \in   \partial g\left( {{\vec{z}}^\star} \right) +  \beta \nabla h\left( {{\vec{z}}^\star} \right)  -  \sum_{m=1}^M {\vec{y}}_m^\star, 
\end{align}
{and the primal feasibility, \ie,}
\begin{equation} \label{eqn_chD:primal_feasibility_condition_in_top_admm}
    {\vec{x}}_m^\star - {\vec{z}}^\star = \vec{0} \quad \forall m = 1,\ldots,M.
\end{equation}

Our goal is to show that \eqref{eqn_chD:stationarity_condition__v_var_with_dual_var_in_top_admm_ver1}-\eqref{eqn_chD:primal_feasibility_condition_in_top_admm} are satisfied. To do so, we now analyze the iterative \ac{TOP-ADMM} updates~\eqref{eqn_chD:generalized_top_admm_algorithm_iterates__for_convergence}. 
In the first step  \eqref{eqn_chD:update_xm__step1_parallel__general_top_admm__prox__for_convergence},  each ${\vec{x}}_m^{\left( i+1 \right)}$ minimizes the update over $\vec{x}_m$, \ie,
\begin{align*} 
    \vec{0} \!  
    \in   \partial f_m\!\left( \!  {\vec{x}}_m^{\left( i + 1\right)} \!\right) \!  +  \!  \rho \!  \left( \!  {\vec{x}}_m^{\left( i + 1\right)} \! - \! {\vec{z}}^{\left( i + 1\right)} \! + \! \frac{{\vec{y}}_m^{\left( i \right)}}{\rho} \! \right). 
\end{align*}
{When $i \!\rightarrow\!\infty$,} we apply the results from Proposition~\ref{prop_chD:asymptotic_residual_error_and_objective} such that
\begin{align*} 
    \vec{0}   
    \in   \partial f_m\left( \!  {\vec{x}}_m^{\left( i + 1\right)} \right)  \!  +  {\vec{y}}_m^{\left( i + 1\right)} ,
\end{align*}
which clearly satisfies the stationarity condition \eqref{eqn_chD:stationarity_condition__u_var_with_dual_var_in_top_admm_ver1}. 

Similarly, in the second step~\eqref{eqn_chD:update_z__step2__general_top_admm__prox__for_convergence},   
${\vec{z}}^{\left( i+1 \right)}$ minimizes that update over $\vec{z}$, \ie, 
\begin{align} 
    \vec{0} 
    \!\in \partial g\!\left(\! {\vec{z}}^{\left( i+1 \right)} \right) 
    \!-\! \sum_{m=1}^M \rho &\Biggl( \! {\vec{x}}_m^{\left( i \right)} \! - \! {\vec{z}}^{\left( i+1 \right)} 
    \!-\! \tau \nabla h \! \left( \! {\vec{z}}^{\left( i \right)} \right) \! + \! \frac{{\vec{y}}_m^{\left( i \right)}}{\rho}  \Biggr),
\end{align}
which can be rewritten as
\begin{align*}
    0 \! \in& \partial g\!\left( \!{\vec{z}}^{\left( i+1 \right)} \!\right) \!+\! {\rho M\! \tau \! \left( \!\nabla  h \! \left( \! {\vec{z}}^{\left( i \right)} \!\right) \! + \! \nabla \!h \! \left( \! {\vec{z}}^{\left( i+1 \right)} \!\right) \! - \! \nabla \!h \! \left( \! {\vec{z}}^{\left( i+1 \right)} \!\right) \!\right) } \nonumber \\ 
    &\!- \! \sum_{m=1}^M \! \left( \! \rho \! \left( \! {\vec{x}}_m^{\left( i+1 \right)} \! - \! {\vec{z}}^{\left( i+1 \right)} \! \right) \! + \! {\vec{y}}_m^{\left( i \right)}  \!\right) \! + \! \sum_{m=1}^M  \! \rho \! \left( \! {\vec{x}}_m^{\left( i+1 \right)} \! - \! {\vec{x}}_m^{\left( i \right)} \!  \right)\!. 
\end{align*}
Plugging {in the} dual variable update \eqref{eqn_chD:update_dual_ym__step3_parallel__general_top_admm} into the above expression and then rearranging terms yield 
\begin{align*}
\vec{0}  \!\in& \! \partial g\!\left(\! {\vec{z}}^{\left( i+1 \right)} \!\right) \!+\! \rho M \!\tau \! \left[ \!\nabla  h \! \left( \! {\vec{z}}^{\left( i+1 \right)} \!\right) \!- \! \left(\! \nabla h \! \left( \! {\vec{z}}^{\left( i+1 \right)} \!\right) \! - \! \nabla h \! \left( \! {\vec{z}}^{\left( i \right)} \!\right) \!\right) \!\right]\nonumber\\
&\! - \! \sum_{m=1}^M  {\vec{y}}_m^{\left( i+1 \right)} \!+ \! \sum_{m=1}^M \rho \Delta {\vec{x}}_m^{\left( i +1 \right)}.
\end{align*} 
{Considering that \ac{TOP-ADMM} generates sequences as $i \!\rightarrow\!\infty$, together with} the results from Corollary~\ref{corollary:x_m_vanishes_asymptotically} 
and Lemma~\ref{lemma:delta_gradient__h_z__and_h_x_m_vanishes_asymptotically}, we have that $\vec{0}  \!\in \! \partial g\!\left(\! {\vec{z}}^{\left( i+1 \right)} \!\right) \!+\! \beta \nabla  h \! \left( \! {\vec{z}}^{\left( i+1 \right)} \!\right) \! - \! \sum_{m=1}^M  {\vec{y}}_m^{\left( i+1 \right)}$ satisfies the stationarity condition~\eqref{eqn_chD:stationarity_condition__v_var_with_dual_var_in_top_admm_ver1}.
    
Finally, primal feasibility \eqref{eqn_chD:primal_feasibility_condition_in_top_admm} is directly satisfied by Proposition~\ref{prop_chD:asymptotic_residual_error_and_objective} 
{when $i \!\rightarrow\!\infty$}. 
\end{proof}

Therefore, the above result ensures the convergence of TOP-ADMM to a \ac{KKT} stationary point for a suitable choice of step-size $\tau$. Clearly, given the convexity assumptions, this implies that such a stationary point is also the optimal solution to problem~\eqref{eqn_chD:general_consensus_top_admm__generic_form}. 

\section{\skblack{Proof of Theorem \ref{theorem:global_convergence_fedtopadmm}}} \label{sec_chD:proof_of_fedtopadmm_convergence}
\skblack{
We present the basic convergence result of proposed {\FedTOPADMM} algorithm by extending the vanishing residual error property results of Theorem~\ref{thm_ChD:definition_of_topadmm_algorith__feasible_problem}. 
}

\skblack{
Recall {\FedTOPADMM} algorithm solves exactly the same problem as \ac{TOP-ADMM}---see Table~\ref{table_chD:distributed_opt_formulations}. Therefore, the augmented Lagrangian is same as~\eqref{eqn_chD:augmented_Lagrangian_top_admm__generic_form}, \ie, 
\begin{align*}
    \pazocal{L}_{\left(\left\{\rho_m\right\}\right)}&\left(\left\{ {\vec{w}}_m \right\}_{m=1}^M\!,  {\vec{w}}, \left\{{\bm{\lambda}}_m \right\}_{m=1}^M \right) \nonumber \\ 
    \!\coloneqq& \! \sum_{m=1}^M  f_m \left( {{\vec{w}}_m} \right)   \!+\! g\left( {{\vec{w}}} \right) \!+\! \beta h\!\left( {{\vec{w}}} \right)  \nonumber \\ 
    &+ \! \sum_{m=1}^M \left\{ {\bm{\lambda}}_m^\trans \left( {\vec{w}}_m \! - \! {\vec{w}} \right) \right\} \!+ \! \sum_{m=1}^M  \frac{\rho_m}{2} \left\| {\vec{w}}_m \! - \! {\vec{w}} \right\|_2^2.
\end{align*}
}

\skblack{
Similarly, the necessary and sufficient optimality conditions for~\eqref{eqn_chD:three_operator__FL_optimization_problem} consist of 1) dual feasibility, \ie, 
\begin{align}
\label{eqn_chD:stationarity_condition__1_fedtopadmm}
    \vec{0}
    &\in  \partial{g}\!\left( {{\vec{w}}^\star} \right)  \!+\!  \beta \nabla h\!\left( {{\vec{w}}^\star} \right)  \!-\!  \sum_{m=1}^M {\bm{\lambda}}_m^\star, \\
    \label{eqn_chD:stationarity_condition__2_fedtopadmm}
    \vec{0} 
    &\in  \partial {f}_m\left( {{\vec{w}}_m^\star} \right)   \!+\! {\bm{\lambda}}_m^\star, \forall m = 1,\ldots,M,
\end{align}
and 2) the primal feasibility, \ie,
\begin{align} \label{eqn_chD:primal_feasibility_condition_in_fedtopadmm}
    {\vec{w}}_m^\star \!-\! {\vec{w}}^\star = \vec{0} \quad \forall m = 1,\ldots,M.
\end{align}
}

\skblack{
To show the convergence of the proposed {\FedTOPADMM}, we show that the server and client processing satisfy the abovementioned optimality conditions asymptotically, \ie, when  $i \! \rightarrow \! \infty$. 
}

\skblack{
During the processing at the server side of Algorithm~\ref{alg:fed_top_admm}, ${\vec{w}}^{\left( i+1 \right)}$ essentially minimizes the $\vec{w}$-update corresponding to~\eqref{eqn_chD:step1__FedTOPADMM_algorithm}, \ie,
\begin{align}
    \vec{0} 
    \in& \partial {g}\left( {\vec{w}}^{\left( i+1 \right)} \right) \nonumber  \\ 
    \label{eqn_chD:proof_fedtopadmm_convergence_w_var__1}
    & - \sum_{m=1}^M \rho_m \left( \! \vec{w}_m^{\left( i \right)} \! - \! {\vec{w}}^{\left( i+1 \right)} \!-\! \tau^{\left( i \right)} \nabla h \! \left( \! {\vec{w}}^{\left( i \right)} \right) \! + \! \frac{{\bm{\lambda}}_m^{\left( i \right)}}{\rho_m}  \right) \nonumber \\
    & + \zeta^{\left(i\right)} \left( {\vec{w}}^{\left( i+1 \right)} \!-\! \vec{w}^{\left(i\right)}   \right), \\
    =& \partial {g}\left( {\vec{w}}^{\left( i+1 \right)} \right) \nonumber  \\ 
    \label{eqn_chD:proof_fedtopadmm_convergence_w_var__2}
    & + \! \sum_{m=1}^M \!\rho_m \tau^{\left( i \right)} \!\left[ \nabla h\!\left( \vec{w}^{\left( i \right)}\right) \!+\! \nabla h\!\left(\vec{w}^{\left(i+1 \right)}\right) \!-\! \nabla h\!\left(\vec{w}^{\left(i+1 \right)}\right) \right] \nonumber \\  
    & - \sum_{m=1}^M \rho_m \!\left[ \vec{w}_m^{\left( i \right)} \!-\! \vec{w}^{\left( i+1 \right)} \!+\!  \frac{\bm{\lambda_m^{\left(i\right)}}}{\rho_m} \!+\! \vec{w}_m^{\left( i+1 \right)} \!-\! \vec{w}_m^{\left( i+1 \right)} \right] \nonumber \\
    & + \zeta^{\left(i\right)} \left( {\vec{w}}^{\left( i+1 \right)} \!-\! \vec{w}^{\left(i\right)}   \right), \\
    \overset{\text{(a)}}{=}& \partial g\left( {\vec{w}}^{\left( i+1 \right)} \right)  \nonumber  \\ 
    &  + \sum_{m=1}^M \rho_m \tau^{\left( i \right)} \nabla h\! \left( \! {\vec{w}}^{\left( i \right)} \right) \!-\! \sum_{m=1}^M \bm{\lambda_m}^{\left( i+1 \right)} \nonumber  \\ 
    \label{eqn_chD:proof_fedtopadmm_convergence_w_var__3}
    & - \sum_{m=1}^M \rho_m \tau^{\left( i \right)} \left[\nabla h\! \left( \! {\vec{w}}^{\left( i+1 \right)} \right) \!-\! \nabla h\! \left( \! {\vec{w}}^{\left( i \right)} \right) \right] \nonumber \\
    & + \sum_{m=1}^M \rho_m \left(  \vec{w}_m^{\left( i+1 \right)} \!-\! \vec{w}_m^{\left( i \right)} \right) \!+\! \zeta^{\left(i\right)} \left( {\vec{w}}^{\left( i+1 \right)} \!-\! \vec{w}^{\left(i\right)}   \right),
\end{align}
where in $\text{(a)}$ we have plugged in the dual variable update~\eqref{eqn_chD:step3__FedTOPADMM_algorithm} assuming $\gamma\!=\!1$. 
}

\skblack{
Now, extending the vanishing residual errors property of \ac{TOP-ADMM} to \linebreak{\FedTOPADMM}, we have dual residual $\left( \! {\vec{w}}^{\left( i +1 \right)} \! - \! {\vec{w}}^{\left( i \right)} \right) \! \rightarrow \! 0$ and primal residual \linebreak$\left( \! {\vec{w}}_m^{\left( i +1 \right)} \! - \! {\vec{w}}^{\left( i \right)} \right) \! \rightarrow \! 0$,  which implies $\left(\!{\vec{w}}_m^{\left( i +1 \right)} \! - \! {\vec{w}}_m^{\left( i \right)}\right) \! \rightarrow \! 0$ when $i \! \rightarrow \! \infty$. Because we assume that the gradient of ${h}$ is $L$-Lipschitz and $\left(\!{\vec{w}}^{\left( i +1 \right)} \! - \! {\vec{w}}^{\left( i \right)}\right) \! \rightarrow \! 0$ for sufficiently large iterations, then the residual error $\left(\nabla  {h} \! \left( \! {\vec{w}}^{\left( i+1 \right)} \right) \! - \! \nabla {h} \! \left( \! {\vec{w}}^{\left( i \right)} \right) \right) \! \rightarrow \! 0$ by using Definition \ref{def_ChD:definition_of_lipschitz_continous_gradient}---such that \eqref{eqn_chD:proof_fedtopadmm_convergence_w_var__3} satisfies the stationarity condition~\eqref{eqn_chD:stationarity_condition__1_fedtopadmm}.
}

\skblack{
The processing at the user side corresponding to~\eqref{eqn_chD:step2__FedTOPADMM_algorithm},  each ${\vec{w}}_m^{\left( i+1 \right)}$ minimizes the $\vec{w}_m$-update, \ie,
\begin{align} \label{eqn_chD:proof_fedtopadmm_convergence_wm_var__1}
    \vec{0} \!  
    &\in \!  \partial {f}_m\!\left( \!  {\vec{w}}_m^{\left( i + 1\right)} \right) \!  +  \!  \rho_m \!  \left( \!  {\vec{w}}_m^{\left( i + 1\right)} \! - \! {\vec{w}}^{\left( i + 1\right)} \! + \! \frac{{\bm{\lambda}}_m^{\left( i \right)}}{\rho_m} \right). 
\end{align}
Now, using dual variable update~\eqref{eqn_chD:step3__FedTOPADMM_algorithm},  \eqref{eqn_chD:proof_fedtopadmm_convergence_wm_var__1} is
\iftrue
\begin{align} \label{eqn_chD:proof_fedtopadmm_convergence_wm_var__2}
    \vec{0}   
    &\in   \partial {f}_m\left( \!  {\vec{w}}_m^{\left( i + 1\right)} \right) \!  +  {\bm{\lambda}}_m^{\left( i + 1\right)},
\end{align}
which always satisfies the stationarity condition \eqref{eqn_chD:stationarity_condition__2_fedtopadmm} for sufficiently large iteration number $i\!\rightarrow\!\infty$.
\else 
$\vec{0}  \! \in \!  \partial \pazocal{F}_m\left( \!  {\vec{y}}_m^{\left( i + 1\right)} \right) \!  +  {\vec{z}}_m^{\left( i + 1\right)} $
which always satisfies the stationarity condition \eqref{eqn_chD:stationarity_condition__2_fedtopadmm} for sufficiently large iteration number $i\!\rightarrow\!\infty$.
\fi
}

\skblack{
Finally, primal feasibility \eqref{eqn_chD:primal_feasibility_condition_in_fedtopadmm} is satisfied by extending the primal residual error result in Theorem~\ref{thm_ChD:definition_of_topadmm_algorith__feasible_problem} to {\FedTOPADMM}, \ie, $\lim_{i \rightarrow +\infty} \left( {\vec{w}}_m^{\left( i+1 \right)} \! - \! {\vec{w}}^{\left( i+1 \right)} \right) \! = \! 0$. 
}

\end{appendices}

\singlespacing
\bibliographystyle{IEEEtran}
\bibliography{IEEEabrv,ref}

 \newcommand{\noop}[1]{}
\begin{thebibliography}{10}
\providecommand{\url}[1]{#1}
\csname url@samestyle\endcsname
\providecommand{\newblock}{\relax}
\providecommand{\bibinfo}[2]{#2}
\providecommand{\BIBentrySTDinterwordspacing}{\spaceskip=0pt\relax}
\providecommand{\BIBentryALTinterwordstretchfactor}{4}
\providecommand{\BIBentryALTinterwordspacing}{\spaceskip=\fontdimen2\font plus
\BIBentryALTinterwordstretchfactor\fontdimen3\font minus \fontdimen4\font\relax}
\providecommand{\BIBforeignlanguage}[2]{{%
\expandafter\ifx\csname l@#1\endcsname\relax
\typeout{** WARNING: IEEEtran.bst: No hyphenation pattern has been}%
\typeout{** loaded for the language `#1'. Using the pattern for}%
\typeout{** the default language instead.}%
\else
\language=\csname l@#1\endcsname
\fi
#2}}
\providecommand{\BIBdecl}{\relax}
\BIBdecl

\bibitem{konevcny2016federated}
J.~Kone{\v{c}}n{\`y}, H.~B. McMahan, F.~X. Yu, P.~Richt{\'a}rik, A.~T. Suresh, and D.~Bacon, ``Federated learning: Strategies for improving communication efficiency,'' \emph{arXiv preprint arXiv:1610.05492}, 2016.

\bibitem{McMahan__FedAvg__2017}
B.~McMahan, E.~Moore, D.~Ramage, S.~Hampson, and B.~A. Arcas, ``{Communication-efficient learning of deep networks from decentralized data},'' in \emph{AISTATS}, 2017, pp. 1273--1282.

\bibitem{Tian_Li_etal__FedProx__2020}
T.~Li, A.~K. Sahu, M.~Zaheer, M.~Sanjabi, A.~Talwalkar, and V.~Smith, ``Federated optimization in heterogeneous networks,'' in \emph{Machine Learning and Systems}, 2020, pp. 429--450.

\bibitem{Saber_etal__FL_from_splitting_algo__2021}
S.~Malekmohammadi, K.~Shaloudegi, Z.~Hu, and Y.~Yu, ``Splitting algorithms for federated learning,'' in \emph{Machine Learning and Principles and Practice of Knowledge Discovery in Databases}, M.~K. et~al., Ed.\hskip 1em plus 0.5em minus 0.4em\relax Cham: Springer International Publishing, 2021, pp. 159--176.

\bibitem{Bauschke:2011}
H.~H. Bauschke and P.~L. Combettes, \emph{{Convex analysis and monotone operator theory in Hilbert spaces}}, ser. CMS Books in Mathematics.\hskip 1em plus 0.5em minus 0.4em\relax New York, NY: Springer New York, 2011.

\bibitem{Kant_journal_msp_top_admm:2020}
S.~Kant, M.~Bengtsson, B.~G{\"{o}}ransson, G.~Fodor, and C.~Fischione, ``{Efficient optimization for large-scale MIMO-OFDM spectral precoding},'' \emph{{IEEE} Trans. Wireless Commun.}, vol.~20, no.~9, pp. 5496--5513, Sep. 2021.

\bibitem{Hellstrom2022}
H.~Hellstr\"{o}m, J.~M.~B. da~Silva~Jr., M.~M. Amiri, M.~Chen, V.~Fodor, H.~V. Poor, and C.~Fischione, ``{Wireless for Machine Learning: A Survey},'' \emph{Foundations and Trends® in Signal Processing}, vol.~15, no.~4, pp. 290--399, 2022.

\bibitem{zhou_and_li_fedadmm2022}
S.~Zhou and G.~Y. Li, ``{Federated learning via inexact ADMM},'' \emph{arXiv e-prints}, p. arXiv:2204.10607, Apr. 2022.

\bibitem{jordan2018communication}
M.~I. Jordan, J.~D. Lee, and Y.~Yang, ``Communication-efficient distributed statistical inference,'' \emph{Journal of the American Statistical Association}, 2018.

\bibitem{tandon2017gradient}
R.~Tandon, Q.~Lei, A.~G. Dimakis, and N.~Karampatziakis, ``Gradient coding: Avoiding stragglers in distributed learning,'' in \emph{Proc. of the International Conference on Machine Learning}.\hskip 1em plus 0.5em minus 0.4em\relax PMLR, 2017, pp. 3368--3376.

\bibitem{chen2018lag}
T.~Chen, G.~Giannakis, T.~Sun, and W.~Yin, ``{LAG: Lazily aggregated gradient for communication-efficient distributed learning},'' \emph{Advances in Neural Information Processing Systems}, vol.~31, 2018.

\bibitem{li2020federated}
T.~Li, A.~K. Sahu, A.~Talwalkar, and V.~Smith, ``Federated learning: Challenges, methods, and future directions,'' \emph{{IEEE} Signal Process. Mag.}, vol.~37, no.~3, pp. 50--60, 2020.

\bibitem{bagdasaryan2019differential}
E.~Bagdasaryan, O.~Poursaeed, and V.~Shmatikov, ``Differential privacy has disparate impact on model accuracy,'' in \emph{Advances in Neural Information Processing Systems}, 2019, pp. 15\,453--15\,462.

\bibitem{stich2018sparsified}
S.~U. Stich, J.-B. Cordonnier, and M.~Jaggi, ``Sparsified {SGD} with memory,'' in \emph{Advances in Neural Information Processing Systems}, 2018, pp. 4447--4458.

\bibitem{di2018efficient}
S.~Di, D.~Tao \emph{et~al.}, ``Efficient lossy compression for scientific data based on pointwise relative error bound,'' \emph{IEEE Transactions on Parallel and Distributed Systems}, vol.~30, no.~2, pp. 331--345, 2018.

\bibitem{wangni2018gradient}
J.~Wangni, J.~Wang \emph{et~al.}, ``Gradient sparsification for communication-efficient distributed optimization,'' in \emph{Advances in Neural Information Processing Systems}, 2018, pp. 1299--1309.

\bibitem{kairouz2021advances}
P.~Kairouz \emph{et~al.}, ``Advances and open problems in federated learning,'' \emph{Foundations and Trends® in Machine Learning}, vol.~14, no. 1–2, pp. 1--210, 2021.

\bibitem{sun2019communication}
J.~Sun, T.~Chen \emph{et~al.}, ``Communication-efficient distributed learning via lazily aggregated quantized gradients,'' in \emph{Advances in Neural Information Processing Systems}, 2019, pp. 3370--3380.

\bibitem{yu2019parallel}
H.~Yu, S.~Yang, and S.~Zhu, ``Parallel restarted {SGD} with faster convergence and less communication: Demystifying why model averaging works for deep learning,'' in \emph{Proc. of the AAAI Conference on Artificial Intelligence}, vol.~33, 2019, p. 5693–5700.

\bibitem{Combettes2011}
P.~L. Combettes and J.-C. Pesquet, ``{Proximal splitting methods in signal processing},'' in \emph{Fixed-Point algorithms for inverse problems in science and engineering}.\hskip 1em plus 0.5em minus 0.4em\relax Springer, New York, NY, 2011, pp. 185--212.

\bibitem{Boyd2011}
S.~Boyd, N.~Parikh, E.~Chu, B.~Peleato, and J.~Eckstein, ``{Distributed optimization and statistical learning via the alternating direction method of multipliers},'' \emph{Foundations and Trends{\textregistered} in Machine Learning}, vol.~3, no.~1, pp. 1--122, 2011.

\bibitem{Parikh2013}
N.~Parikh, S.~Boyd, N.~Parikh, and S.~Boyd, ``{Proximal algorithms},'' \emph{Foundations and Trends{\textregistered} in Optimization}, vol.~1, no.~3, pp. 123--231, 2013.

\bibitem{Komodakis_and_Pesquet:2015}
N.~{Komodakis} and J.~{Pesquet}, ``Playing with duality: {An} overview of recent primal-dual approaches for solving large-scale optimization problems,'' \emph{{IEEE} Signal Process. Mag.}, vol.~32, no.~6, pp. 31--54, Oct. 2015.

\bibitem{Glowinski2016}
R.~Glowinski, S.~J. Osher, and W.~Yin, \emph{{Splitting methods in communication, imaging, science, and engineering}}, ser. Scientific Computation.\hskip 1em plus 0.5em minus 0.4em\relax Cham, Switzerland: Springer International Publishing, 2016.

\bibitem{Beck2017}
A.~Beck, \emph{{First-order methods in optimization}}.\hskip 1em plus 0.5em minus 0.4em\relax SIAM, 2017.

\bibitem{Davis2017}
D.~Davis and W.~Yin, ``{A three-operator splitting scheme and its optimization applications},'' \emph{Set-Valued and Variational Analysis}, vol.~25, no.~4, pp. 829--858, Dec. 2017.

\bibitem{Yan_PD3O:2018}
M.~Yan, ``A new primal-dual algorithm for minimizing the sum of three functions with a linear operator,'' \emph{J. Sci. Comput.}, vol.~76, pp. 1698--1717, 2018.

\bibitem{Ryu_Yin_ls_book_2022_draft}
E.~K. Ryu and W.~Yin, \emph{{Large-scale convex optimization via monotone operators}}.\hskip 1em plus 0.5em minus 0.4em\relax Cambridge University Press (to be published), 2022.

\bibitem{FedSplit__2020}
R.~Pathak and M.~J. Wainwright, ``{FedSplit: An algorithmic framework for fast federated optimization},'' in \emph{Advances in Neural Information Processing Systems}, vol.~33, 2020, pp. 7057--7066.

\bibitem{FedDR__2021}
Q.~Tran-Dinh, N.~H. Pham, D.~T. Phan, and L.~M. Nguyen, ``{FedDR -- Randomized Douglas-Rachford splitting algorithms for nonconvex federated composite optimization},'' in \emph{Advances in Neural Information Processing Systems}, vol.~34, 2021, pp. 4447--4458.

\bibitem{zhou_iceadmm_draft_2022}
S.~Zhou and G.~Y. Li, ``{Communication-efficient ADMM-based federated learning},'' \emph{arXiv e-prints}, p. arXiv:2110.15318, Jan. 2022.

\bibitem{Wang_Banerjee__Bregman_ADMM__2014}
{H. {Wang}} and {A. {Banerjee}}, ``Bregman alternating direction method of multipliers,'' in \emph{Proc. Adv. Neural Inf. Process. Syst.}, 2014, pp. 2816--2824.

\bibitem{Banert2020}
S.~Banert, R.~I. Bo{\c{t}}, and E.~R. Csetnek, ``{Fixing and extending some recent results on the ADMM algorithm},'' \emph{Numerical Algorithms}, pp. 1--23, May 2020.

\bibitem{Combettes_data_science__2021}
P.~L. Combettes and J.-C. Pesquet, ``Fixed point strategies in data science,'' \emph{{IEEE} Trans. Signal Process.}, pp. 1--1, 2021.

\bibitem{Condat__prox_splitting:2021}
L.~{Condat}, D.~{Kitahara}, A.~{Contreras}, and A.~{Hirabayashi}, ``{Proximal splitting algorithms: A tour of recent advances, with new twists!}'' \emph{arXiv:1912.00137}, Dec. 2021.

\bibitem{Hong_and_Luo__2017}
M.~Hong and Z.-Q. Luo, ``{On the linear convergence of the alternating direction method of multipliers},'' \emph{Math. Program.}, vol. 162, no. 1--2, pp. 165--199, 2017.

\bibitem{Kant_etal__CFR_TOPADMM_2022}
S.~Kant, M.~Bengtsson, G.~Fodor, B.~Göransson, and C.~Fischione, ``{EVM mitigation with PAPR and ACLR constraints in large-scale MIMO-OFDM using TOP-ADMM},'' \emph{accepted to IEEE Trans. Wireless Commun.}, 2022.

\bibitem{Pierra1984}
G.~Pierra, ``{Decomposition through formalization in a product space},'' \emph{Mathematical Programming}, vol.~28, no.~1, pp. 96--115, Jan. 1984.

\bibitem{Stathopoulos2016}
G.~Stathopoulos, H.~Shukla, A.~Sz\H{u}cs, Y.~Pu, and C.~N. Jones, ``{Operator splitting methods in control},'' \emph{Foundations and Trends{\textregistered} in Systems and Control}, vol.~3, no.~3, pp. 249--362, 2016.

\bibitem{Kant_etal__intro_to_TOPADMM_2022}
S.~Kant, C.~Fischione, M.~Bengtsson, G.~Fodor, and B.~Göransson, ``{An introduction to three-operator ADMM for wireless communications and machine learning},'' \emph{in preparation}, 2022.

\bibitem{Lecun__MNIST_data__1998}
Y.~LeCun, L.~Bottou, Y.~Bengio, and P.~Haffner, ``Gradient-based learning applied to document recognition,'' \emph{Proceedings of the IEEE}, vol.~86, no.~11, pp. 2278--2324, 1998.

\bibitem{Alex_Cifar10_100_dataset_2009}
A.~Krizhevsky, \emph{{Learning multiple layers of features from tiny images}}.\hskip 1em plus 0.5em minus 0.4em\relax Technical report, 2009.

\bibitem{Boyd2004ConvexOptimization}
S.~P. Boyd and L.~Vandenberghe, \emph{{Convex optimization}}.\hskip 1em plus 0.5em minus 0.4em\relax Cambridge University Press, 2004.

\bibitem{hjorungnes:2011}
A.~Hj\o{}rungnes, \emph{{Complex-valued matrix derivatives: With applications in signal processing and communications}}.\hskip 1em plus 0.5em minus 0.4em\relax Cambridge University Press, 2011.

\end{thebibliography}


\begin{thebibliography}{99}


\bibitem{Tom2013MaskShaping}
A. Tom, A. Sahin, and H. Arslan, "Mask compliant precoder for OFDM spectrum shaping," \textit{IEEE Comm. Lett.}, vol. 17, no. 3, pp. 447-450, Mar. 2013.

\bibitem{cvx_solver_grant_boyd}
M. Grant and S. Boyd, "CVX: MATLAB software for disciplined convex programming," version 2.0 beta, http://cvxr.com/cvx, Sep. 2013.

\bibitem{ploynomial_roots__numerical_methods_newton_hartley_etc}
Wikipedia contributors, "Root-finding algorithm," Wikipedia, the free encyclopedia, 11 Mar. 2018. Web. 12 Mar. 2018.

\bibitem{DeBeek2009SculptingPrecoder}
J. van de Beek, "Sculpting the multicarrier spectrum: A novel projection precoder," \textit{IEEE Comm. Lett.}, vol. 13, no. 12, pp. 881-883, Dec. 2009.

\bibitem{fischer_and_Windpassinger__rv_mimo__2003}
R.F.H. Fischer and C. Windpassinger, "Real versus complex-valued equalisation in V-BLAST systems,"  \textit{Electronics Letters}, vol. 39, no. 6, pp. 470--471, Mar. 2003.

\bibitem{yang_hanzos__mimo_detection_survey__2015}
S. Yang and L. Hanzos, "Fifty Years of MIMO Detection: The Road to Large-Scale MIMOs," \textit{IEEE Comm. Surveys \& Tutorials}, vol. 17, no. 4, pp. 1941--1988, Sep. 2015.

\bibitem{day_and_heroux__cmplx_to_real__2000}
D. Day and M. A. Heroux, "Solving Complex-valued Linear Systems via Equivalent Real Formulations," \textit{SIAM J. Sci. Comput.}, vol. 23, no. 2, pp. 480--498, 2000.

\bibitem{combettes_pesquet__proximal_splitting_method_in_signal_processing_2011}
P. L. Combettes and J. C. Pesquet, "Proximal splitting methods in signal processing," in Fixed-Point Algorithms for Inverse Problems in Science and Engineering, NY, USA, Springer-Verlag, pp. 185-212, 2011.

\bibitem{huang_and_sidiropulos__qcqp_general_problem_solutions_oct2016}
K. Huang and N. D. Sidiropoulos, "Consensus-ADMM for General Quadratically Constrained Quadratic Programming," \textit{IEEE Trans. on Signal Process.}, vol. 64, no. 20, pp. 5297-5310, Oct. 15, 2016.

\bibitem{miller__inverse_of_sum_of_matrices__1981}
K. S. Miller, "On the Inverse of the Sum of Matrices," \textit{Mathematics Magazine}, vol. 54, no. 2, pp. 67--72, Mar., 1981.


\end{thebibliography}

\end{document}